\documentclass[format=acmsmall, review=false]{acmart}
\usepackage{acm-ec25-arxiv}
\makeatletter
\def\@ACM@checkaffil{
    \if@ACM@instpresent\else
    \ClassWarningNoLine{\@classname}{No institution present for an affiliation}%
    \fi
    \if@ACM@citypresent\else
    \ClassWarningNoLine{\@classname}{No city present for an affiliation}%
    \fi
    \if@ACM@countrypresent\else
        \ClassWarningNoLine{\@classname}{No country present for an affiliation}%
    \fi
}
\makeatother
\usepackage{booktabs} 
\usepackage[ruled]{algorithm2e} 
\usepackage{00macros}

\SetAlFnt{\small}
\SetAlCapFnt{\small}
\SetAlCapNameFnt{\small}
\SetAlCapHSkip{0pt}
\IncMargin{-\parindent}

\setcitestyle{authoryear}


\acmYear{2025}\copyrightyear{2025}
\setcopyright{rightsretained}
\acmConference[EC '25]{The 26th ACM Conference on Economics and Computation}{July 7--10, 2025}{Stanford, CA, USA}
\acmBooktitle{The 26th ACM Conference on Economics and Computation (EC '25), July 7--10, 2025, Stanford, CA, USA}
\acmDOI{10.1145/3736252.3742614}
\acmISBN{979-8-4007-1943-1/25/07}

\title[Selecting Optimal Alternates for Citizens' Assemblies]{Alternates, Assemble! Selecting Optimal Alternates for Citizens' Assemblies}

\author{Angelos Assos}
\email{assos@mit.edu}
\orcid{0009-0006-2841-3511}
\affiliation{%
  \institution{MIT}
  \city{Cambridge}
  \state{MA}
  \country{USA}
}

\author{Carmel Baharav}
\email{cbaharav@ethz.ch}
\orcid{0009-0004-3634-6721}
\affiliation{%
  \institution{ETH Zurich}
  \city{Zurich}
  \country{Switzerland}
}

\author{Bailey Flanigan}
\email{baileyf@mit.edu}
\orcid{0000-0002-7805-1989}
\affiliation{%
  \institution{Harvard}
  \city{Cambridge}
  \state{MA}
  \country{USA}
}
\author{Ariel D. Procaccia}
\email{arielpro@seas.harvard.edu}
\orcid{0000-0002-8774-5827}
\affiliation{%
  \institution{Harvard}
  \city{Cambridge}
  \state{MA}
  \country{USA}
}

\begin{abstract}
\emph{Citizens' assemblies} are an increasingly influential form of deliberative democracy, where randomly selected people discuss policy questions. The legitimacy of these assemblies hinges on their representation of the broader population, but participant dropout often leads to an unbalanced composition. In practice, dropouts are replaced by preselected \emph{alternates}, but existing methods do not address how to choose these alternates. To address this gap, we introduce an optimization framework for alternate selection. Our algorithmic approach, which leverages learning-theoretic machinery, estimates dropout probabilities using historical data and selects alternates to minimize expected misrepresentation. Our theoretical bounds provide guarantees on sample complexity (with implications for computational efficiency) and on loss due to dropout probability mis-estimation. Empirical evaluation using real-world data demonstrates that, compared to the status quo, our method significantly improves representation while requiring fewer alternates.  
\end{abstract}

\begin{document}

\maketitle

\section{Introduction} \label{sec:introduction}
Motivated by mounting threats to democratic governance, new methods for facilitating public participation in democracy are taking root. One of the most prominent examples is the \textit{citizens' assembly}, in which members of the public are chosen randomly to serve on a panel. These panelists convene for several days to learn about an issue, deliberate, and then weigh in on what policy measures should be taken. Across continents, citizens' assemblies and other similar processes, falling under the broader umbrella of \textit{deliberative minipublics}, are entering politics at all scales (e.g., see the 2021 Global Climate Assembly \cite{GlobalAssembly2021}, France's recent national assemblies \cite{giraudet2022co,france2}, Ireland's constitutional convention \cite{ireland2019}, Germany's national assembly on nutrition \cite{Bundestag2023Nutrition}, and the plethora of local and regional assemblies occurring worldwide \cite{OECD2020Innovative}).

Because citizens' assemblies are costly per participant\footnote{Participants are usually compensated and, if the event is in person, their travel expenses are covered.} and deliberation is difficult to facilitate at large scale, only a small fraction of the population can partake\emdash typical panel sizes range from tens to hundreds. Consequently, \textit{representation} is paramount: it is important that the panelists who partake can credibly represent the views of the broader populace. In practice, assembly organizers typically approach this problem by ensuring that the panelists satisfy proportionally representative \textit{quotas}: upper and lower limits on how many panelists must be from various groups. Usually these quotas are set to achieve proportional representation of the population on many dimensions at once, including  gender, age, race/ethnicity, education level, and some measure of political opinion. For example, on a panel of 100 people serving the state of Wisconsin, the quotas for \textit{political opinion} might require that the panel contain 40-44 democrats, 40-44 republicans, and 12-20 undecided voters, mirroring the state's party affiliation rates of $42\%, 42\%, 16\%$. The quotas on \textit{education level} might additionally require that 30-36 panelists have a college degree and 64-70 do not. 

Over the past several years, computer science research has produced algorithms for randomly sampling a panel that is guaranteed to satisfy practitioner-chosen quotas of this form \cite{flanigan2021fair,flanigan2021transparent,flanigan2024manipulation,baharav2024fair}. The panel selection process (both as it is studied in past work \textit{and} as it typically works in practice) consists of two steps: First, thousands of invitations are sent out to the population, inviting them to partake in the process. Those who respond affirmatively form the \textit{pool} of willing participants. Then, the final \textit{panel} is randomly selected from the pool, and it must satisfy the quotas. Existing algorithms address the second step of this process, aiming to randomize within the quotas in a way that is maximally fair, transparent, and strategyproof.

This past work goes to great lengths\emdash sometimes at the expense of other ideals \cite{flanigan2023mini}\emdash to ensure that the panel satisfies representative quotas. However, past work does not address that this representation can be undone by what happens downstream: often, months pass between the selection of the panel and its convention, and in this time, some of those who were initially selected for the panel \textit{drop out}. Many times, people drop out just before the panel convenes, or simply do not show up on the first day. These dropouts cause quota violations, which are especially problematic because they do not affect all groups equally: \Cref{fig:dropout-rates} shows dropout rates among different groups across 25 citizens' assemblies. These trends show that certain groups --- trending more financially and socially marginalized --- tend to drop out at substantially higher rates. We see that young people, indigenous people, racial minorities, people with less permanent housing, and people with the least education drop out at disproportionately high rates\emdash trends that  largely persist across datasets collected from assemblies run in two separate countries by two separate organizations. This is not at all surprising: dropout is likely due to random shocks (e.g., illness, childcare falling through, last-minute work commitments, transportation issues), and those with less resources may be less able to overcome such shocks in order to attend the panel.

\begin{figure}[t]
    \centering
\includegraphics[width=0.85\linewidth]{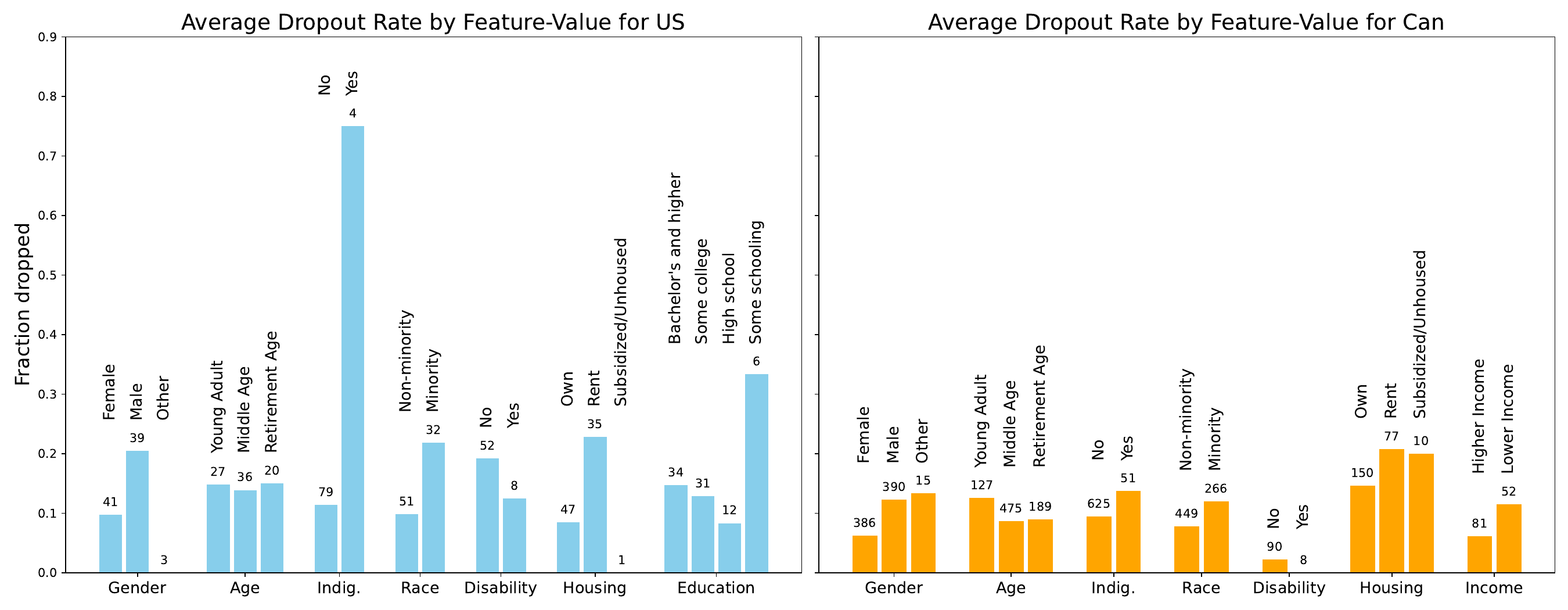}
    \vspace{-0.5em}
    \caption{Average dropout rates among different groups over 3 US assemblies (blue) and 22 Canadian assemblies (orange). Not all datasets contained all attributes; numbers above bars show how many people across datasets had that attribute. Data cleaning methods are in \Cref{app:datacleaning}}
    \label{fig:dropout-rates}
    \vspace*{-1.1em}
\end{figure}
\vspace{0.3em}

\noindent \textbf{State-of-the-art: heuristics for selecting \textit{alternates}.} To hedge against loss of representation due to dropout, practitioners often select \textit{alternates} --- extra participants to have on standby to replace panelists who drop out. The key challenge is that the alternates must be selected \textit{before} seeing who drops out; this is because by the time dropouts occur (on or near the first day of the panel's convention), it is too late to recruit replacements for a multi-day process. Subject to this constraint, alternates are used in different ways across organizations: some organizations select an entire duplicate panel using the same method as was used to select the original, and then compensate these alternates for attending the panel's first day. Others directly inflate the size of the panel, imposing quotas that require extra panelists in groups they expect to drop out based on past experience. 

Such methods are heuristic and have two main potential inefficiencies, both of which we aim to address in this paper. First, these methods either ignore differential dropout rates altogether, or they rely on practitioners to guess at dropout rates. Given the wealth of historical data on which kinds of people have dropped out from \textit{past} panels, there is a natural opportunity to instead learn statistical predictions of different groups' likelihoods of dropping out. Second, these methods do not consider the problem's combinatorial structure (i.e., which \textit{combinations} of attributes alternates should have), operating instead on the level of quotas, which are imposed on single attributes in isolation. Beyond their potential to lead to suboptimal representation, these inefficiencies can be costly: each additional alternate comes at a cost as they need to be compensated. These opportunities to improve the efficiency of alternate selection motivate our research question.

\begin{quote}     \textbf{\textit{Research question:}} \textit{Given panelists' dropout probabilities\emdash which are estimated from data and are thus subject to prediction errors\emdash can we design a practical and performant algorithm that selects an optimal set of alternates of a given size?}
\end{quote}

\vspace{0.3em}
\noindent \textbf{Problem setup.}
More formally, the inputs to our problem are the quotas we want to satisfy; the originally-selected panel $K$ of size $k$, which respects these quotas; a pool $N$ of size $n$ from which can choose alternates; a budget $a \in \mathbb{N}^+$ limiting the number of alternates we can choose; and panelists' \textit{dropout probabilities} $\rho_i | i \in K$, describing each panelist's probability of dropping out. As will be technically consequential, we assume that panelists drop out \textit{independently}, reflecting the reality that panelists are randomly chosen from a large population and unlikely to know each other, and thus the event that one drops out should not affect the dropout event of another. These dropout probabilities induce a \textit{dropout set distribution} $\mathcal{D}$, describing the probability that each \textit{panel subset} drops out. Given all these inputs, the pipeline of selecting and deploying alternates works as depicted in \Cref{fig:pipeline}. Though our solution must engage with all steps 1-3 in the figure, our main goal is to design a procedure for step 1. 

\begin{figure}[h!]
    \centering
    
\includegraphics[width=0.9\textwidth]{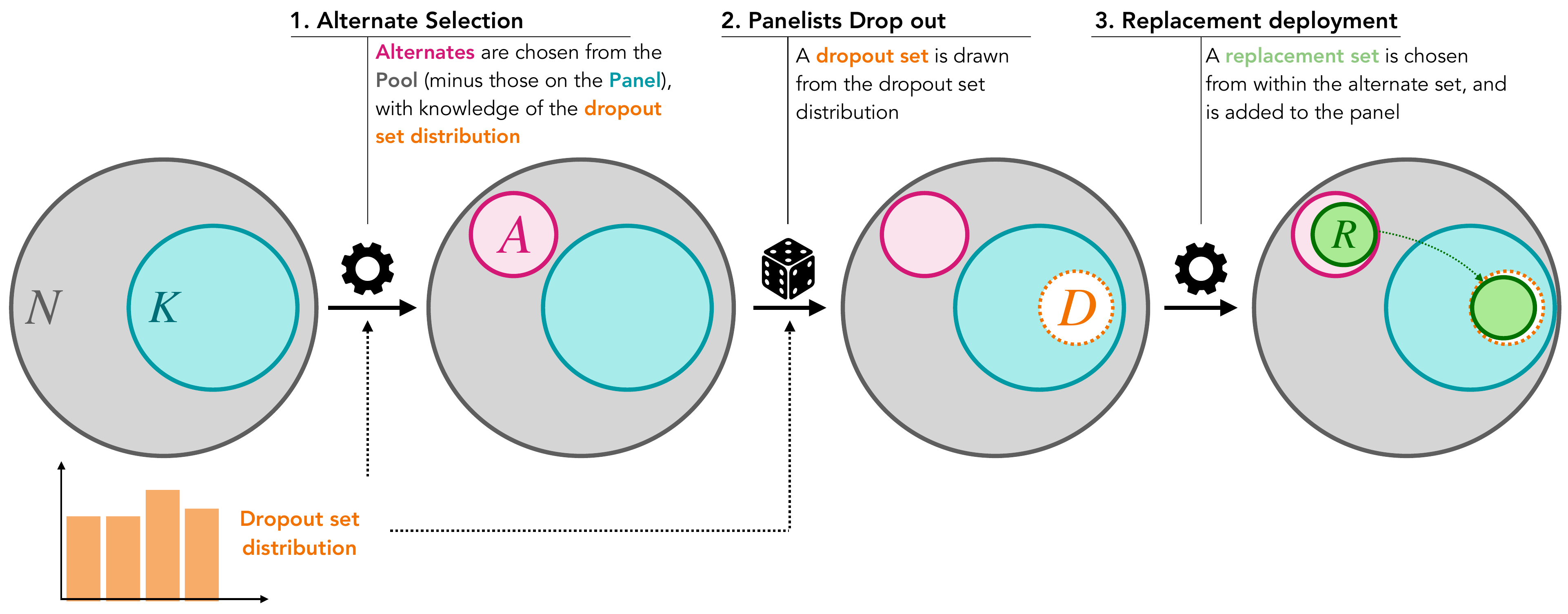}
\vspace{-1em}
    \caption{The alternate selection pipeline. \textbf{1.} With knowledge of all inputs, an \textit{alternate selection algorithm} chooses an alternate set $A \subseteq N : |A| = a$. Importantly, $A$ must be chosen \textit{before seeing the random draw of which panelists drop out} (the \textit{dropout set} $D$), reflecting that alternates need to be placed on retainer ahead of time so that, e.g., they can be invited to the first day of the panel's convention in anticipation of some panelists not showing up. \textbf{2.} A dropout set $D \subseteq K$ is drawn from the dropout set distribution. \textbf{3.} The best possible \textit{replacement set} $R \subseteq A$ is chosen replace $D$. The final panel is then $(K \setminus D) \cup R$. To ensure we respect the initial panel size constraint (for budgetary and other logistical reasons), we require that $|R| \leq |D|$, so that adding the replacement set does not exceed the original panel size. }
    \label{fig:pipeline}
\end{figure}
Our goal is to compute the alternate set that can best restore the quotas after dropouts\emdash in expectation over the randomness of which panelists drop out. The optimal alternate set $A^*$ is then expressed (semi-formally) as follows, where the \textit{loss} measures how well the quotas are satisfied:
\begin{equation} \label{eq:opt-informal}
    A^* = \text{argmin}_{A \subseteq N} \sum_{D \subseteq K} \Pr_{D \sim \mathcal{D}}[D \text{ drops out}] \cdot \left(\min_{R \subseteq A : |R| \leq |D|}  \ \text{\textit{loss}}(K \setminus D \cup R)\right).
\end{equation}
We next discuss the difficulty of computing $A^*$. For this, it is useful to have in mind that $n$ is typically on the order of hundreds or thousands, $k$ is typically on the order of tens or hundreds, and $a$ may range widely, from $<10$ up to $k$ or more.

\vspace{0.3em}
\noindent \textbf{Key challenge.} Solving \Cref{eq:opt-informal} presents a combinatorial minefield, consisting of three separate, nested tasks. First and most generally, we must compute the best alternate set (size $a$) within the pool (size $n$), which by brute force would require examining ${n \choose a}$ alternate sets. Secondly, we must evaluate at least one alternate set $A$, which requires computing its expected performance over all possible dropout sets, i.e., all subsets of the panel (size $k$). By brute force, this amounts to computing a sum with $2^{k}$ terms. Finally, to compute \textit{each} term of this sum, corresponding to a specific dropout set $D$, we must determine how well $A$ can replace $D$. This means finding the best subset of $A$, which requires looking at all subsets of $A$ of size at most $D$, which is at least ${a \choose |D|}$. In fact this third problem\emdash the sub-sub-routine among these three combinatorial problems\emdash is \textit{itself} NP-hard in asymptotic parameter $a$, as follows from an existing theorem (\Cref{app:reduction}).

\vspace{0.3em}
\noindent \textbf{Approach: \textit{Integer Linear Programming} (ILP) and \textit{Empirical Risk Minimization} (ERM).} In light of our problem's NP-hardness, a natural choice would be to pursue a polynomial-time approximation algorithm. For practical reasons, we do not go quite so far: of the three tasks above, the two subset selection tasks (the first and third tasks) can be solved quickly and exactly via ILP solvers. We therefore opt to present an algorithm that is not polynomial time but in exchange does not compromise on practical performance in these aspects. It is the \textit{second} task\emdash computing the expected performance of a given alternate set $A$ over a distribution with support size $2^k$\emdash which does not lend itself to ILP solving. This is where we will pursue polynomial dependency on the problem's parameters, and in exchange forgo exact optimality. 

For this second task, our approach makes use of the observation that, because agents drop out independently, we have \textit{sample access} to the dropout set distribution even though its support is too large to write down: we can sample a dropout set $D \sim \mathcal{D}$ by flipping a $\dprob_i$-biased coin for every panelist; those which come up 1s compose $D$. Taking inspiration from ERM, we use this sample access to construct an \textit{empirical estimate} of $\mathcal{D}$, and then we solve our problem over this empirical distribution $\hat{\mathcal{D}}$ instead. 
We then use an ILP to perform both set selection tasks at once, choosing the alternate set $\hat{A}^*$ that optimizes the expected performance \textit{over the empirical version of the distribution}\emdash just \Cref{eq:opt-informal} where $\mathcal{D}$ is replaced with $\hat{\mathcal{D}}$.
Ideally, at sufficient sample size, $\hat{\mathcal{D}}$ should be ``similar enough'' $\mathcal{D}$ that $\hat{A}^*$ should exhibit loss close to that of the optimal alternate set $A^*$ with high probability. For this ERM-based approach to be a sufficient improvement on the brute-force method, it must be that we can ensure that the loss of $\hat{A}^*$ closely approximates that of $A^*$ with only polynomially-many samples. 

\vspace{0.3em}
\noindent \textbf{Contributions.}
We answer our research question affirmatively via the following series of results.

In \textbf{\Cref{sec:algo}}, we define our ERM algorithms and analyze their sample complexity. We first formally connect our problem to the PAC learning model, which will allow us to leverage known learning theory bounds. We draw this connection for two standard loss functions: the \textit{binary loss}, capturing whether $A$ can perfectly restore the quotas; and the \textit{linear loss}, capturing \textit{how far} $A$ is from restoring the quotas. Then we show that for the ERM algorithm using either loss definition, only $O(a \log n)$ samples are sufficient to ensure that $loss(\hat{A}^*) - loss(A^*) \leq \varepsilon$ with high probability. This bound is the result of tight bounds on the dimensionality of our hypothesis classes. This very low sample complexity is important, because we must solve an ILP whose size depends on it. 

In \textbf{\Cref{sec:robustness}}, we analyze the extent to which our two ERM algorithms (using binary and linear loss, respectively) are robust to inaccuracies in the dropout probabilities. While our two loss functions were indistinguishable on the basis of sample complexity, here a distinction emerges: the binary loss can be arbitrarily non-robust to small errors, while the linear loss is extremely robust, its loss growing linearly with the prediction error.

In \textbf{\Cref{sec:empirics}}, we evaluate our entire algorithmic pipeline on data from real citizens' assemblies. We first estimate dropout probabilities from historical data, and then we compare the performance of our binary- and linear-loss ERM algorithms against heuristic benchmark algorithms emulating those used in practice. We find that while our estimates are likely imperfect in some instances (due largely to our data being highly suboptimal for this task), our ERM algorithms have lower loss, more consistent performance, and greater robustness to mis-estimates than our heuristic benchmarks. 

In \textbf{Section 6}, we first discuss how our algorithms extend to the more general case where \textit{alternates} drop out, and to various other ways of hedging against dropouts beyond selecting an alternate set. Finally, we close by situating our algorithmic solutions\emdash and the problem of alternate selection in general\emdash within the broader goal of \textit{randomly} selecting citizens' assembly participants.

\vspace{0.3em}

\noindent {\bf For appendices and omitted proofs, see the full version \cite{thispaper-ssrn}.}

\vspace{0.3em}
\noindent \textbf{Related work.} Broadly, the CS community is increasingly interested in deliberative democracy, with recent papers studying AI-aided deliberation, opinion change, and representation in these contexts \cite{fish2023generative,tessler2024ai,fishkin2019deliberative, gelauff2022opinion}. More specifically to our setting, there has been substantial recent computer science research on the task of randomly selecting the participants of citizens' assemblies, known as \textit{sortition} \cite{benade2019no,flanigan2020neutralizing,flanigan2021fair,flanigan2021transparent,EKMP+22,ebadianboosting,flanigan2024manipulation,baharav2024fair}. However, this existing work focuses only on selecting the \textit{original panel}, leaving untouched the issue of subsequent panel dropouts. Methodologically, our use of the PAC learning framework to (probably, approximately) solve an NP-hard problem is analogous to the approach used by \citet{peters2022robust} to compute a (probably, approximately) optimal allocation and pricing of rented rooms for prospective tenants.
Finally, the key challenge of our problem, estimating a sum over $2^k$ possible dropout sets, is reminiscent of the core challenge in computing the partition function for graphs or other combinatorial objects. While our analysis does not utilize this technical connection, typical methods for estimating partition functions also often involve sampling techniques, though these techniques differ from ours~\cite{ma2013estimating}.

\section{Model} \label{sec:model}
An \textit{instance} of the alternate selection problem  $\inst$ consists of a set of \textit{quotas}, a \textit{panel} that satisfies these quotas, a \textit{pool}, a vector of \textit{dropout probabilities}, and an \textit{alternates budget}. 

The \textbf{\textit{panel}} $K$ is a set of the $k$ agents who were originally chosen as assembly members (we often call these agents \textit{panelists}). The \textbf{\textit{pool}} $N$ is a set of $n$ agents that is disjoint from $K$. The pool corresponds to the agents that were willing to participate in the assembly but were \textit{not} chosen, and thus are now available to be chosen as alternates. We will refer to individual agents as $i$.

To define the quotas, we first define \textbf{\textit{features}} and \textbf{\textit{feature-values}}. A feature is an attribute category, such as ``gender'' or ``age''. Each feature $f$ takes on a predefined set of feature-values $V_f$, an exhaustive and mutually exclusive set of categorical values of that feature. For example, the feature $f = $ ``age'' might take on the values $V_{age} = \{<\text{50}, \geq\text{50 years old}\}$. Formally, a feature is a function $f : (N \cup K) \to V_f$ that maps each agent to their value for feature $f$. In practice, the feature-values of all agents in $K$ and $N$ are known, and we assume this is the case. We will often use $f,v$ to refer a feature, value pair (e.g., $age$, $< 50$). We express the set of all feature, value pairs as $FV := \{f,v | f \in F, v \in V_f\}$.

On each $f,v \in FV$, there are upper and lower \textbf{\textit{quotas}} describing the maximum and minimum number of agents with value $v$ for feature $f$ the panel should contain. These upper and lower quotas are denoted as $u_{f,v} \in \mathbb{N}$ and $l_{f,v} \in \mathbb{N}$, respectively, where $u_{f,v} \geq l_{f,v}$ and $u_{f,v}> 0$ (otherwise, the group defined by $f,v$ is effectively dropped from the instance). We summarize these quotas as $\boldsymbol{l}:=(l_{f,v} | f,v \in FV)$ and $\boldsymbol{u} := (u_{f,v} | f,v \in FV)$. We say that a set of agents $S$ \textit{satisfies the quotas} iff
\[\sum_{i \in S} \mathbb{I}(f(i) = v) \in [l_{f,v},u_{f,v}] \qquad \text{for all }f,v \in FV.\]
We assume the original panel $K$ satisfies the quotas.

Finally, our \textbf{\textit{alternates budget}} $a \in \mathbb{N}$ is the maximum number of alternates we can choose. An instance is formally expressed by the tuple $(N,K,\boldsymbol{l},\boldsymbol{u},a)$. In instance $\inst$, we will let $\mathcal{A}(\inst) := \{A \subseteq N : |A| = a\}$ be the space of all possible alternate sets (we will drop the $\inst$ when clear from context).

\vspace{0.3em}
\noindent \textbf{Dropouts.} Each panelist $i \in K$ has some \textbf{\textit{dropout probability}} $\dprobi \in [0,1]$, describing the probability that they drop out after the original panel $K$ is selected. As discussed in \Cref{sec:introduction}, panelists do not know each others' identities prior to the panel (and thus their dropout events should not be related), so we assume that each agent $i$ drops out independently; that is, if $X_i \in \{0,1\}$ is the indicator of the event that $i$ drops out, then the $X_i$'s are independent and $Pr[X_i = 1] = \dprobi$. We summarize these dropout probabilities in the vector $\dprobs:=(\dprobi | i \in K)$. Let $D := \sum_{i \in K} X_i$ be the \textbf{\textit{dropout set}}, the random variable describing the subset of the panel that drops out. Note that $D \in 2^K$ (where $2^K$ is the power set of $K$), i.e., any subset of the panel can drop out. Let the \textbf{\textit{dropout set distribution}} $\ddist_{\dprobs} : 2^K \to [0,1]$ be the distribution over dropout sets induced by the dropout probabilities $\dprobs$. Then, by our assumption of independent dropouts, 
\[\ddist_{\dprobs}(D) = \prod_{i \in D} \rho_i \cdot \prod_{j \in K \setminus D}(1-\dprob_j).\]
While $\ddist_{\dprobs}$ is the \textit{true} dropout distribution, under various circumstances we end up working with only \textit{estimates} of this distribution. We will thus use $\mathcal{D}$ to refer to a generic dropout set distribution.

\vspace{0.3em}
\noindent \textbf{Selecting and evaluating an alternate set.} An alternate set is chosen by an \textbf{\textit{alternate selection algorithm}}: a mapping that takes as input an instance $\inst$ and a dropout set distribution $\mathcal{D}$, and outputs an alternate set $A \in \mathcal{A}(\inst)$. Designing this alternate selection algorithm is the main challenge, and the task we will focus on in this paper. 

Intuitively, we evaluate our chosen alternate set according to its ability to ``replace'' a randomly drawn dropout set $D$. Formalizing this intuition requires defining precisely how $A$ is used to ``replace'' $D$, which works as follows:
%
With our alternate set $A$ selected, we then observe the realization of $D \sim \mathcal{D}_{\dprobs}$, at which point we are left with only the agents in $K \setminus D$. We must then restore the quotas to the greatest degree possible using a \textbf{\textit{replacement set}} $R$: any subset of $A$ is of size at most $|D|$ (that is, $R \subseteq A : |R| \leq |D|$). The best replacement set in $A$ can simply be computed by a simple integer linear program (see \Cref{app:replacementILP} for formulation), which solves following problem:
\begin{equation}
    \min_{R \subseteq A : |R| \leq |D|} \quad dev(K \setminus D \cup R, \boldsymbol{l}, \boldsymbol{u}). \label{eq:replacements}
\end{equation}
Here, we are minimizing a \textbf{\textit{deviation function}} $dev$, which conceptually measures how well $R$ restores the quotas $\boldsymbol{l}, \boldsymbol{u}$ after $D$ drops out from $K$. Formally, a deviation function is a mapping that takes in a set of agents and quotas and outputs a real number describing how far that set of agents is from satisfying those quotas.

We will consider two specific deviation functions. First, the \textit{binary deviation} $(dev^{0/1})$  simply checks whether the quotas are satisfied, outputting 0 if yes and 1 if no. The \textit{linear deviation} $(dev^{\ell_1})$ is more continuous, measuring how far away each quota is from being satisfied by normalized $\ell_1$ distance.\footnote{Normalizing by any value in $[l_{f,v}, u_{f,v}]$ is equally principled, as any choice therein is equivalent when the quotas are tight. Our choice of $u_{f,v}$ is for technical convenience in \Cref{sec:robustness} and avoidance of divide-by-0 issues.} For a set of agents $S$ and quotas $\boldsymbol{l},\boldsymbol{u}$, we define these deviation functions as follows.
    \begin{align*}
        dev^{0/1}(S,\boldsymbol{l},\boldsymbol{u}):=&\begin{cases}
        0 \quad & \text{if } \sum_{i \in S} \mathbb{I}(f(i) = v) \in [l_{f,v},u_{f,v}] \qquad \forall \, f \in F, v \in V_f\\
        1\quad & \text{else } 
    \end{cases} \tag{\text{binary dev.}}\\    dev^{\ell_1}(S,\boldsymbol{l},\boldsymbol{u}):=& \sum_{f \in F}\sum_{v \in V_f}\frac{\max\left\{0,\,l_{f,v} - \sum_{i \in S} \mathbb{I}(f(i) = v),\,-u_{f,v} + \sum_{i \in S} \mathbb{I}(f(i) = v)\right\}}{u_{f,v}} \tag{\text{linear dev.}}
    \end{align*}
Note that the linear deviation is strictly more expressive than the binary deviation: the functions correspond exactly when $dev^{0/1} = dev^{\ell_1} = 0$, but otherwise the range of $dev^{0/1}$ is simply 1 while the range of $dev^{\ell_1}$ encompasses a spread of rational numbers. Throughout the paper, we will often drop the $\boldsymbol{l},\boldsymbol{u}$ from the $dev$ inputs when they are clear from context.


Finally, we evaluate an alternate selection algorithm that outputs $A$ according to the \textit{expected} deviation $dev$, over the randomness of drawing $D$ from the true dropout set distribution $\ddist_{\dprobs}$: 
\begin{equation} \label{eq:loss}
    L^{dev}(A;\ddist_{\dprobs},\inst) = \E_{D \sim \ddist_{\dprobs}} \left[\min_{R \subseteq A : |R| \leq |D|}dev(K \setminus D \cup R,\boldsymbol{l},\boldsymbol{u})\right] 
\end{equation}
Abusing notation slightly, we will write the loss with respect to the binary deviation $dev^{0/1}$ as $L^{0/1}$ and linear deviation $dev^{\ell_1}$ as $L^{\ell_1}$. We will sometimes compute the loss over a generic dropout set distribution $\mathcal{D}$, which is defined by \Cref{eq:loss} where all instances of $\ddist_{\dprobs}$ are replaced with $\ddist$. 
We will often drop the $\inst$ out of our loss function when it is clear from context.

\vspace{0.3em}
\noindent \textbf{An ILP for computing the optimal alternate set.} Finally, we express the optimal alternate set as the solution of an ILP. This ILP\emdash expressed here informally and written fully in \Cref{app:ILPformulation}\emdash is the program we solve when computing an optimal alternate set over a given dropout distribution $\mathcal{D}$. Note that by the program definition below, $\textsc{Opt}^{dev}(\mathcal{D}_{\dprobs},\inst)$ is the true optimal alternate set in instance $\inst$ with dropout probabilities $\dprobs$. In \textsc{Opt}$^{dev}$, the objective function is exactly the expected loss (\Cref{eq:loss}), but for illustrative purposes, we have expanded the expectation into its sum form, where $support(\ddist):=\{D \in 2^K : \ddist(D) > 0\}$ be the collection of unique dropout sets with nonzero probability in $\ddist$ (noting that any $D$ outside this support cannot contribute to the loss).

\vspace{0.3em}
\noindent \textsc{Opt}$^{dev}(\mathcal{D},\inst)$: 
\begin{align*}
    &\underset{A, \, R_{A,D} | D \in support(\mathcal{D})}{\arg\min}  \quad  \sum_{D \in support(\mathcal{D})} Pr[D \sim \mathcal{D}] \cdot dev(K \setminus D \cup R_{A,D}, \boldsymbol{l},\boldsymbol{u})\\
    &\text{s.t.} \quad |A| \leq a\\
    &\qquad \  A \subseteq N\\
    &\qquad \  dev(K \setminus D \cup R_{A,D},\boldsymbol{l},\boldsymbol{u}) \leq dev(K \setminus D \cup R,\boldsymbol{l},\boldsymbol{u}) \qquad \forall \ R \subseteq A : |R| \leq |D|, \ D \in support(\mathcal{D})
\end{align*}
Observe that the number of terms in the objective and the number of constraints grows in the size of the $support(\ddist)$. In our problem, this will be the dominant source of combinatorial blowup, as such, when $support(\ddist)$ is small, the ILP is also ``small'' and can be solved quickly in practice with ILP solvers. However, when this support is large (it would be $2^k$ when $\ddist = \ddist_{\dprobs}$), this ILP becomes intractable. This is precisely the advantage of our ERM-based approach, to be presented next: it allows us to solve \textsc{Opt} over an \textit{empirical version} of $\ddist_{\dprobs}$, whose support will be polynomial in size. In this way, we handle the most problematic source of combinatorial blowup -- the $2^k$ possible dropout sets -- while outsourcing the ``practically easy'' combinatorial aspects of the problem to \textsc{Opt}.

\section{\textsc{ERM-Alts}, an ERM-based alternate selection algorithm} \label{sec:algo}
While we cannot explicitly write down or optimize over $\ddist_{\dprobs}$, we can \textit{sample} it: to draw a $D \sim \ddist_{\dprobs}$, we must simply flip a $\dprob_i$-weighted coin for each agent, and those whose coins turn up 1 are those in $D$. Using this observation, we now present our algorithm, \algoname, which finds the alternate set that minimizes the \textit{empirical risk}: the expected loss on an empirical approximation of $\ddist_{\dprobs}$.

\begin{algorithm}
\caption{$\algoname^{dev}(s,\dprobs,\inst)$}
\DontPrintSemicolon
     $\hat{\ddist}_{\dprobs} \leftarrow$ Draw $s$ samples (dropout sets) from $\ddist_{\dprobs}$ to produce \textit{empirical} approximation of $\ddist_{\dprobs}$, called $\hat{\ddist}_{\dprobs}$.\;
     $\hat{A}^* \leftarrow$ Solve \textsc{Opt}$^{dev}\;(\hat{\ddist}_{\dprobs},\inst)$.\;
    \Return $\hat{A}^*$
\end{algorithm}

Our theoretical analysis of $\algoname^{dev}$ will seek to bound its sample complexity: \textit{how large must $s$ be to ensure that $\hat{\ddist}_{\dprobs} \approx \ddist_{\dprobs}$ such that the loss of $\algoname^{dev}(s,\dprobs,\inst)$ closely approximates the loss of \textsc{Opt}$^{dev}(\ddist_{\dprobs},\inst)$?} To ensure that the ILP \textsc{Opt}$^{dev}(\hat{\ddist}_{\dprobs},\inst)$ (solved in step 2 of the algorithm) is small enough to solve quickly, we ideally want $s$\emdash and therefore $support(\hat{\ddist_{\dprobs}})$\emdash to be polynomial in the parameters of the instance. We prove such sample complexity bounds by applying known results from PAC learning, so we now cast our problem within the PAC learning framework.

\subsection{Casting our problem within the PAC learning framework}
From PAC learning, we will primarily use the concept of \textit{agnostic PAC learnability}, which is defined as follows (see Definition 3.3, \cite{shalev2014understanding}).
\begin{definition}[\textbf{Agnostic PAC Learnability}] \label{def:agnostic} A hypothesis class of functions $\mathcal{H}$ with domain $\mathcal{X}$ and range $\mathcal{Y}$ is agnostic PAC learnable if there exists a function $s_{\mathcal{H}} : (0, 1)^2 \to \mathbb{N}$ and a learning algorithm
with the following property: for every $\delta,\varepsilon \in (0,1)$ and every distribution $\mathcal{D}$
over $\mathcal{X} \times \mathcal{Y}$, when running the learning algorithm on $s \geq s_{\mathcal{H}}(\varepsilon,\delta)$ i.i.d. examples generated by a fixed distribution $\mathcal{D}$, the algorithm returns a hypothesis $\hat{h} \in \mathcal{H}$ such that\[
\Pr[L(\hat{h}) - \min_{h \in \mathcal{H}} L(h)\leq \varepsilon]\geq 1-\delta,
\] 
where the probability is computed over the random draw of the choice of training set given to the learning algorithm (assumed to be drawn from $\mathcal{D}$) and the loss $L$ measures the extent to which $h$ fails to output the \textit{true label} $y \in \mathcal{Y}$ on input $x \in \mathcal{X}$, in expectation over $x,y \sim \mathcal{D}$.
\end{definition}
Mapping this onto our problem, we want to find an alternate set $\hat{A}$ (hypothesis $\hat{h}$) within $\mathcal{A}$ (hypothesis class $\mathcal{H}$) such that with high probability (probability at least $1-\delta$), $\hat{A}$ has near-optimal expected loss (has loss within $\varepsilon$ of that of the loss-minimizing hypothesis) over dropout sets (samples) drawn from $\ddist_{\dprobs}$ (a fixed distribution $\mathcal{D}$). Our learning algorithm is empirical risk minimization (Algorithm 1), corresponding to a classical algorithm for PAC learning.

Fix an instance $\inst$. As illustrated above, our hypothesis class in $\inst$ is just the collection of all alternate sets $\mathcal{A}(\inst)$. To distinguish actual alternate sets (sets of agents) and hypotheses (functions), we let $h_A$ be the hypothesis corresponding to $A$. For all $a \in \mathcal{A}(\inst)$, each $h_A$ takes as input a given dropout set (and technically the instance, but we leave this implicit) and outputs a real number describing how well $A$ can replace that dropout set, as measured by $dev$. Because $dev^{0/1}$ and $dev^{\ell_1}$ have different ranges, we must define two hypothesis classes per instance: a \textit{binary} hypothesis class $\mathcal{H}^{0/1}(\inst) = \{h_A^{0/1} | A \in \mathcal{A}(\inst)\}$, and a \textit{linear} hypothesis class $\mathcal{H}^{\ell_1}(\inst) = \{h_A^{\ell_1} | A \in \mathcal{A}(\inst)\}$. The output of $h_A^{dev}$ is exactly the minimum deviation $dev$ over all replacement sets, as defined in \Cref{eq:replacements}:
\begin{equation} \label{eq:hypots}
    h_A^{0/1}(D)= \min_{R \subseteq A : |R| \leq |D|} dev^{0/1}(K \setminus D \cup R) \qquad \text{and} \qquad h_A^{\ell_1}(D) = \min_{R \subseteq A : |R| \leq |D|} dev^{\ell_1}(K \setminus D \cup R).
\end{equation}
Across these two hypothesis classes, the \textit{labeling} function is the same: the ``true'' label for any $D$ is 0, reflecting that $A$ can ideally maintain quota satisfaction after $D$ drops out. Formally, we define a labeling function as $\tau(D;\inst) = 0$, which we shorten to $\tau(D)$. Of course, there may exist no $A \in \mathcal{A}(\inst)$ such that $h_A(D) = 0$ for all $D$; this consideration is encompassed by the definition of agnostic PAC learning, where we aim to compete with the best available hypothesis.

Finally, we show that by minimizing the $L^{0/1}$ and $L^{\ell_1}$ loss as in \Cref{sec:model}, we are respectively minimizing the standard $0$-$1$ loss and $\ell_1$ loss from PAC learning. Fixing $A \in \mathcal{A}$, we apply the definition of the loss (\Cref{eq:loss}) and our definition of hypotheses (\Cref{eq:hypots}):
\begin{align*}
    L^{0/1}(A;\ddist_{\dprobs}) &= \E_{D \sim \ddist_{\dprobs}}\left[\mathbb{I}(h_A^{0/1}(D) \neq 0\right] = \E_{D \sim \ddist_{\dprobs}}\left[\mathbb{I}(h_A^{0/1}(D) \neq \tau(D)\right]\text{, the 0/1 PAC learning loss;}\\
    L^{\ell_1}(A;\ddist_{\dprobs}) &= \E_{D \sim \ddist_{\dprobs}}\left[\left|h_A^{\ell_1}(D) - 0\right|\right] = \E_{D \sim \ddist_{\dprobs}}\left[\left|h_A^{\ell_1}(D) - \tau(D)\right|\right]\text{, the $\ell_1$ PAC learning loss.}
\end{align*}


This completes our reduction to agnostic PAC learning, which we will use to apply known sample complexity bounds (i.e., bounds on $s_{\mathcal{H}(\epsilon,\delta)}$ in \Cref{def:agnostic}) from the PAC learning literature. These known bounds, stated in \Cref{app:restatement}, depend on the ``dimension'', or richness, of $\mathcal{H}$. To apply them, we must therefore characterize the respective dimensions of $\mathcal{H}^{0/1}(\inst)$ and $\mathcal{H}^{\ell_1}(\inst)$. Because $\mathcal{H}^{0/1}(\inst)$ consists of hypotheses with range $\{0,1\}$, its relevant dimensionality measure is the \textit{VC-dimension} ($\VC$); hypotheses in $\mathcal{H}^{\ell_1}(\inst)$ have real-valued range, so this class's dimension is measured by a generalization of the VC-dimension called the \textit{Pseudodimension} ($\Pdim$).
We define these notions of dimension below in the same notation as in \Cref{def:agnostic}; then, in the next section, we tightly bound our hypothesis classes' respective dimensions.

\begin{definition}[\textbf{VC-Dimension}]
    The \textit{VC-Dimension} of $\mathcal{H}$ (whose hypotheses have range $\{0,1\}$) is the size of the largest set of points in $\mathcal{X}$ that can be \textit{shattered} by $\mathcal{H}$. A set of points $X$ can be shattered by $\mathcal{H}$ iff for all labeling functions $b : X \to \{0,1\}^{|X|}$, there exists an $h_b \in \mathcal{H}$ such that $h_b(x) = b(x)$ for all $x \in X$.
\end{definition}

\begin{definition}[\textbf{Pseudodimension}]
    The \textit{Pseudodimension} of $\mathcal{H}$ (whose hypotheses have range $[0,1]$) is the size of the largest set of points in $\mathcal{X}$ that can be \textit{pseudo-shattered} by $\mathcal{H}$. A set of points $X$ can be pseudo-shattered by $\mathcal{H}$ iff there exists a \textit{witness vector} $r \in [0,1]^{|X|}$ (indexed as $r_x$) such that for all labeling functions $b : X \to \{0,1\}^{|X|}$, there exists an $h_b \in \mathcal{H}$ such that $\mathbb{I}(h_b(x) - r_x > 0) = b(x)$ for all $x \in X$.
\end{definition}


\subsection{Bounds on VC dimension of $\mathcal{H}^{0/1}$ and Pseudimension of $\mathcal{H}^{\ell_1}$}
Our hypothesis classes are instance-dependent (since the set of available alternate sets depends on who is in the instance's pool), but we want to give dimension bounds that apply across instances. Thus, for each type of hypothesis class we will examine its \textit{worst-case} dimensionality over all instances within a given collection. We define a collection of instances as follows: let $\mathfrak{I}(n, a)$ denote all instances $\mathcal{I}$ with $|N| = n$ and an alternate budget of $a$. Similarly let $\mathfrak{I}(n,a,|F|)$ denote all instances $\mathcal{I}$ with $|N|=n$, an alternate budget of $a$, and $|F|$ features.

\begin{theorem}[\textbf{Main $\boldsymbol{\VC, \, \Pdim}$ Bounds}]\label{thm:mainbounds} Fix any $n,a \in \mathbb{N}_{\geq 1}$ such that $n \geq a$. Then,
    \[\max_{\inst \in \mathfrak{I}(n,a)}\VC(\mathcal{H}^{0/1}(\inst)), \, \max_{\inst \in \mathfrak{I}(n,a)} \Pdim(\mathcal{H}^{\ell_1}(\inst) ) \in \Theta(a \log n).\]
\end{theorem}
\begin{proof}[Proof: Upper bounds]
    Fix an arbitrary instance $\inst$. Both upper bounds are proven by the same argument, which relies on the simple fact that both our hypothesis classes $\mathcal{H}^{0/1}(\inst)$ and $\mathcal{H}^{\ell_1}(\inst)$ are finite. Their sizes are upper-bounded by $|\mathcal{A}|$, which is of size at most ${n \choose a}$, as there are at most this many $a$-subsets of $N$. 
    We now apply the well-known fact that $\VC(\mathcal{H}) \leq \log(|\mathcal{H}|)$, and likewise, $\Pdim(\mathcal{H}) \leq \log(|\mathcal{H}|)$. It follows that
    \[\VC(\mathcal{H}^{0/1}(\inst)),\Pdim(\mathcal{H}^{\ell_1}(\inst)) \leq \log {n \choose a} \leq  \log\left(\left(\frac{en}{a}\right)^a\right) = a(\log n - \log a + 1) \in O(a \log n). \qedhere\]
\end{proof}
\vspace{-0.5em}
The lower bounds on VC dimension and pseudodimension both arise from the same instance. At a high level, the instance uses $|F| = k = \lfloor \log{n \choose a} \rfloor - 2a$ binary features, a panel of size $|F|$ consisting of all $|F|$-length basis vectors, and a shattering set of size $|F|$ where each dropout set within it contains a single panelist. From there, the work of the proof, is engineering the pool to make any subset of this collection of dropout sets replaceable. The full proof is located in \Cref{app:mainbounds}.

The construction giving this lower bound uses many features: $|F| = k = \lfloor \log{n \choose a}\rfloor - 2a$. This prompts the question: can we prove sample complexity bounds that improve as $|F|$ gets smaller? We present such bounds below; their proofs are located in \Cref{app:FdepVCPdim}. We remark that the condition $n \geq a 2^{|F|/a}$ is required only for the lower bound.
\begin{theorem} [$\boldsymbol{|F|}$\textbf{-dependent} $\boldsymbol{\VC, \, \Pdim}$ \textbf{Bounds}]  \label{thm:FdepVCPdim}Fix any $n,a,|F| \in \mathbb{N}_{\geq 1}$ such that $n \geq a 2^{|F|/a}$. Taking $|F|$ to be a varying parameter and $\max_{f \in F} |V_f|$ to be a constant, we have that
 \[\max_{\inst \in \mathfrak{I}(n,a,|F|)}\VC(\mathcal{H}^{0/1}(\inst)), 
 \max_{\inst \in \mathfrak{I}(n,a,|F|)}\Pdim(\mathcal{H}^{\ell_1}(\inst)) \in \left[|F|, O\left(a|F|\log\max_{f\in F}|V_f|\right)\right] \in \left[|F|, O(a\,|F|)\right].\]
\end{theorem}

\subsection{Sample complexity bounds for \algoname}

We now plug our upper bounds from \Cref{thm:mainbounds} into known dimension-dependent sample complexity bounds (\Cref{thm:vc-existing-bounds} and \Cref{lem:pdimexistingbound} respectively, stated in \Cref{app:restatement}), to give sample size-dependent formal guarantees on the loss of \textsc{ERM-Alts}. The key takeaway of these bounds is that we need few samples (only logarithmic in $n$), which is important because our ILP \textsc{Opt} scales commensurately. One could achieve analogous $|F|$-dependent results by plugging in \Cref{thm:FdepVCPdim}. 
\begin{corollary}[\textbf{Bound for }$\algoname^{0/1}$]\label{cor:binary-sample-complexity}
   Fix $\inst,\dprobs$ and constants $\varepsilon, \delta > 0$. There exists 
    \[s \in O\left((a\log n+\log(1/\delta)) \cdot 1/\varepsilon^2\right) \quad \text{such that}\]
    \[\Pr\left[L^{0/1}(\textsc{ERM-Alts}^{0/1}(s,\dprobs,\inst);\ddist_{\dprobs}) - L^{0/1}(\textsc{Opt}^{0/1}(\ddist_{\dprobs},\inst);\ddist_{\dprobs})\leq \varepsilon\right] \geq 1-\delta,\]
    where the probability above is over the random sampling of $\hat{\ddist}_{\dprobs}$ in \textsc{ERM-Alts}.
\end{corollary}
\begin{corollary}[\textbf{Bound for }$\algoname^{\ell_1}$]\label{cor:sample-complexity-linear}  Fix $\inst,\dprobs$ and constants $\varepsilon, \delta > 0$. There exists 
    \[s \in O\left((a\log n+\log(1/\varepsilon) + \log(1/\delta)) \cdot 1/\varepsilon^2\right) \quad \text{such that}\]
  \[\Pr\left[L^{\ell_1}(\textsc{ERM-Alts}^{\ell_1}(s,\dprobs,\inst);\ddist_{\dprobs}) - L^{\ell_1}(\textsc{Opt}^{\ell_1}(\ddist_{\dprobs},\inst);\ddist_{\dprobs})\leq \varepsilon\right] \geq 1-\delta.\]
\end{corollary}

\section{Robustness to mis-estimated dropout probabilities} \label{sec:robustness}

So far, we have assumed that panelists' dropout probabilities $\dprobs$ are known. However, in practice they must be estimated from historical data based on known features. This prompts the question: How robust is our ERM algorithm when dropout probability estimates are subject to prediction errors? 

To formalize this, let $\tilde{\dprob}_i$ be our estimate of the agent $i$'s dropout probability, and correspondingly, let $\tilde{\dprobs} = (\tilde{\dprob}_i |i \in K)$. Then, let our \textit{prediction error} $\err$ be defined as the largest difference between any dropout probability estimate and its true value, i.e., $\err:=\|\dprobs-\tilde{\dprobs}\|_\infty.$ Ideally, the loss of an algorithm given $\tilde{\dprobs}$ should approach its loss given $\dprobs$ as $\err \to 0$. In fact, we find that whether this is true depends on our choice of deviation function: $\algoname^{0/1}$ suffers arbitrary loss under vanishingly small prediction errors (\Cref{thm:lb-robust-binary}), while $\algoname^{\ell_1}$ is highly robust (\Cref{thm:ub-robust-linear,thm:lb-robust-linear}).

For all proofs in this section, we weaken one assumption: we permit the replacement set $R$ to be larger than $D$.\footnote{\Cref{thm:lb-robust-binary} holds without this relaxation, but we have not been able to determine whether \Cref{thm:ub-robust-linear} does.} Formally, we use the following modified version of the loss (\Cref{eq:loss}), where the optimization occurs over all $R \subseteq A$ instead of just $R \subseteq A$ such that $|R| \leq |D|$:
\[\mathcal{L}^{dev}(A;\ddist_{\dprobs},\inst) := \E_{D \sim \ddist_{\dprobs}} \left[\min_{R \subseteq A }dev(K \setminus D \cup R,\boldsymbol{l},\boldsymbol{u})\right].\]
Even with this relaxation, we will still be able to prove separation between the two deviation functions' robustness. Note that this relaxation does not affect our results from \Cref{sec:algo}.

In our proofs we will use the following shorthand, so $\ermtruealts{dev}$ and $\ermfakealts{dev}$ are the alternate sets given by our algorithm given the \textit{true} and \textit{estimated} dropout probabilities, respectively, and $\opttruealts{dev}$ and $\optfakealts{dev}$ are the analogous optimal alternate sets.
\begin{align*}
\ermtruealts{dev} &:= \algoname^{dev}(s,\dprobs,\inst) \qquad \qquad &\ermfakealts{dev} &:= \algoname^{dev}(s,\tilde{\dprobs},\mathcal{I})\\
\opttruealts{dev} &:= \opttruealts{dev}(s,\ddist_{\dprobs},\inst) \qquad \qquad &\optfakealts{dev} &:= \textsc{OPT}^{dev}(s,\ddist_{\tilde{\dprobs}},\mathcal{I}),
\end{align*} 
Henceforth dropping $\inst, \boldsymbol{l},\boldsymbol{u}$ from the notation, our formal goal is to bound the added loss $\mathcal{L}$ of our algorithms due to mis-estimation,
\[ \mathcal{L}^{dev}(\ermfakealts{dev},\ddist_{\dprobs}) - \mathcal{L}^{dev}(\ermtruealts{dev},\ddist_{\dprobs}).\] 
We will generally bound this quantity by bounding the analogous difference for $\optfakealts{dev}$ and $\opttruealts{dev}$ (which will be \textit{deterministic}), and then use that, for sufficient $s$, $\ermtruealts{dev} = \opttruealts{dev}$ and $\ermfakealts{dev} = \optfakealts{dev}$ (which will be \textit{with high probability}). This last translation step will require our lower bound instances to be constructed with some care.

We first show that $\algoname^{0/1}$ is \textit{arbitrarily} non-robust. Recalling that the maximum possible binary loss is 1, this means that for arbitrarily small $\gamma$, there exists an instance where estimation error $\gamma$ increases the binary loss by a quantity arbitrarily close to 1 with high probability. 

\begin{theorem}[\textbf{Binary loss lower bound}]\label{thm:lb-robust-binary}  
Fix any constants $\alpha \in (0,1), \delta \in (0, 1/2], \gamma \in (0,1)$. There exists $\mathcal{I}$, $\boldsymbol{\rho}$ and $\tilde{\boldsymbol{\rho}}$ with $\|\boldsymbol{\rho}-\tilde{\boldsymbol{\rho}}\|_\infty \leq \gamma$, and $s(\alpha,\delta,\gamma)$ such that for all $s \geq s(\alpha,\delta,\gamma)$, 
\[\Pr\left[\mathcal{L}^{0/1}(\ermfakealts{0/1}; \ddist_{\dprobs}) - \mathcal{L}^{0/1}(\ermtruealts{0/1}; \ddist_{\dprobs})  \geq 1- \alpha\right] \geq 1-2\delta. \]
\end{theorem}
The proof, located in \Cref{app:lb-robust-binary}, consists of an instance with a single binary feature. One group has dropout probability $0$ and the other has dropout probability $\gamma$. By misestimating these groups' respective dropout probabilities as $\gamma, 0$ instead of $0, \gamma$, our alternates are exclusively replacements for the first group, when they should be replacements for the second.

In contrast to \textsc{ERM-Alts$^{0/1}$}, we find that \algoname$^1$ \textit{is} highly robust, incurring additional loss proportional to $\err |FV|$ (\Cref{thm:ub-robust-linear}). We then show that this bound is essentially tight (\Cref{thm:lb-robust-linear}). 
\begin{theorem}[\textbf{Linear loss upper bound}]\label{thm:ub-robust-linear} 
     Fix any constants $\varepsilon, \delta > 0, \gamma \in (0, 1]$. Fix any $\inst$, $\boldsymbol{\rho}$ and $\tilde{\boldsymbol{\rho}}$ with $\|\boldsymbol{\rho}-\tilde{\boldsymbol{\rho}}\|_\infty \leq \gamma$. Then, there exists $s(\varepsilon,\delta)$ such that for all $s \geq s(\varepsilon,\delta)$,
    \[\Pr\left[\mathcal{L}^{\ell_1}(\ermfakealts{\ell_1}; \ddist_{\dprobs}) - \mathcal{L}^{\ell_1}(\ermtruealts{\ell_1}; \ddist_{\dprobs})  \leq 2\err |FV| + \varepsilon\right] \geq 1-\delta
    \]
\end{theorem}
\begin{proof}[Proof sketch] The overall proof can be found in \Cref{app:ub-robust-linear}. It proceeds roughly as follows: Fix all required entities; as usual, $\inst$ will be dropped from our notation. The core of the proof is showing the following bound on the change in linear loss for \textit{any} fixed alternate set $A$, when evaluated with respect to $\mathcal{D}_{\dprobs}$ versus $\mathcal{D}_{\tilde{\dprobs}}$:
\begin{equation}\label{eq:linearublemma}
    \vert\mathcal{L}^{\ell_1}(A;\mathcal{D}_{\dprobs}) - \mathcal{L}^{\ell_1}(A;\mathcal{D}_{\tilde{\dprobs}})\vert  \leq \err |FV|.
\end{equation}

Once we have this bound, we can apply it to $\opttruealts{\ell_1}$ to show the following chain of inequalities, where the first inequality is by the optimality of $\optfakealts{\ell_1}$ for $\ddist_{\tilde{\dprobs}}$:
\[\mathcal{L}^{\ell_1}(\optfakealts{\ell_1};\ddist_{\tilde{\dprobs}}) - \mathcal{L}^{\ell_1}(\opttruealts{\ell_1};\ddist_{\dprobs})\leq \mathcal{L}^{\ell_1}(\opttruealts{\ell_1};\ddist_{\tilde{\dprobs}}) - \mathcal{L}^{\ell_1}(\opttruealts{\ell_1};\ddist_{\dprobs})\leq  \gamma|FV|.\]
This gives us an almost analogous version of our desired bound for $\opttruealts{\ell_1}$ and $\optfakealts{\ell_1}$. To relate $\mathcal{L}^{\ell_1}(\ermfakealts{\ell_1};\ddist_{\dprobs})$ to $\mathcal{L}^{\ell_1}(\optfakealts{\ell_1};\ddist_{\tilde{\dprobs}})$, we apply \Cref{eq:linearublemma} once more and derive $s(\varepsilon,\delta)$ based on \Cref{cor:sample-complexity-linear} to ensure that with probability $\geq 1-\delta$, the loss of $\ermfakealts{\ell_1}$ on $\mathcal{D}_{\tilde{\boldsymbol{\rho}}}$ with $s \geq s(\varepsilon,\delta)$ is within $\varepsilon$ of its respective corresponding optimal alternate set (for the upper bound, we do not need to ensure that $\ermtruealts{\ell_1}$ is close to $\opttruealts{\ell_1}$). This additional application of \Cref{eq:linearublemma} is the source of the 2-factor on $\err|FV|$ in the statement, and completes the proof.

It remains to show \Cref{eq:linearublemma}. The argument is quite intricate, but we will present the broad strokes of the argument here and give the fully formal version in \Cref{app:ub-robust-linear}. The high level approach is to construct $k+1$ intermediate probability vectors that transform $\dprobs$ to $\tilde{\dprobs}$ by altering one agent's dropout probability at a time. As such, let $\dprobs^i = (\tilde{\dprob}_1, \dots, \tilde{\dprob}_i, \dprob_{i+1}, \dots, \dprob_k)$ for all $i \in [k]$. Then, using a telescoping sum and the triangle inequality, we show that 
\[\vert\mathcal{L}^{\ell_1}(A;\mathcal{D}_{\dprobs}) - \mathcal{L}^{\ell_1}(A;\mathcal{D}_{\tilde{\dprobs}})\vert =\vert \mathcal{L}^{\ell_1}(A;  \mathcal{D}_{\dprobs^0}) - \mathcal{L}^{\ell_1}(A;  \mathcal{D}_{\dprobs^k})\vert\leq \sum_{i=1}^{k} \left \vert \mathcal{L}^{\ell_1}(A; \mathcal{D}_{\dprobs^{i-1}}) - \mathcal{L}^{\ell_1}(A; \mathcal{D}_{\dprobs^i}) \right \vert.\]
We first bound each individual term of this resulting sum separately. We do so via a coupling argument, where we couple the dropouts drawn per $\dprobs^i$ and $\dprobs^{i+1}$. We represent these dropouts as 
\[
\boldsymbol{Y} = (Y_j \sim \text{ Bernoulli}(\rho^{i-1}_j)|j \in [k]) \quad \text{ and } \quad \boldsymbol{Y'} = (Y'_j \sim \text{ Bernoulli}(\rho^{i}_j)|j \in [k]),
\]
where $Y_j$ (likewise $Y_j'$) indicates whether $j$ dropped out.
We couple $\boldsymbol{Y}$ and $\boldsymbol{Y'}$ by drawing each $Y_j$ and $Y_j'$ via the same independent draw of $X_j \sim Unif[0,1]$. Letting $\boldsymbol{X} = (X_j | j \in [k])$, we denote the coupled dropout vectors as
$\boldsymbol{Y}(\boldsymbol{X}) = (Y_j(\boldsymbol{X}) | j \in [k])$ and $\boldsymbol{Y'}(\boldsymbol{X})$ analogously. 
Let $D(\boldsymbol{Y}(\boldsymbol{X})) = \{j \in [k] | Y_j = 1\}$ be the dropout set specified by $\boldsymbol{Y}$ and define $D(\boldsymbol{Y'}(\boldsymbol{X}))$ analogously. Define $R(\boldsymbol{Y}(\boldsymbol{X})) := \argmin_{R \subseteq A} dev^{\ell_1}(K \setminus D(\boldsymbol{Y}(\boldsymbol{X})) \cup R)$ to be the optimal replacement set for $D(\boldsymbol{Y}(\boldsymbol{X}))$, and define $R(\boldsymbol{Y'}(\boldsymbol{X}))$ analogously. 

This coupling will serve to ensure that with substantial probability, $D(\boldsymbol{Y}(\boldsymbol{X}))$ and $D(\boldsymbol{Y'}(\boldsymbol{X}))$ are the same, and otherwise they are similar. In particular, there is some $cond(X_i)$ such that
\begin{align*}
    Pr[cond(X_i)] &= 1-\gamma &&\text{ and }\qquad cond(X_i) \ \ \, \implies  D(\boldsymbol{Y}(\boldsymbol{X})) = D(\boldsymbol{Y'}(\boldsymbol{X}))\\
    Pr[\neg cond(X_i)] &= \gamma &&\text{ and }\qquad \neg cond(X_i) \implies D(\boldsymbol{Y}(\boldsymbol{X})) \triangle D(\boldsymbol{Y'}(\boldsymbol{X})) = \{i\}.
\end{align*}
Fix an $i \in [k]$. Applying the definition of the loss along with the law of total expectation, linearity of expectation, and Jensen's inequality, we deduce that
\begin{align*}
|\mathcal{L}^{\ell_1}(A; \,\mathcal{D}_{\dprobs^{i-1}}) &- \mathcal{L}^{\ell_1}(A;  \mathcal{D}_{\dprobs^i})|\\
&\leq \mathbb{E}_{\boldsymbol{X}}\left[| dev^{\ell_1}(K \setminus D(\boldsymbol{Y}(\boldsymbol{X})) \cup R(\boldsymbol{Y}(\boldsymbol{X}))) - dev^{\ell_1}(K \setminus D(\boldsymbol{Y'}(\boldsymbol{X})) \cup R(\boldsymbol{Y'}(\boldsymbol{X})))|\right].
\intertext{From here, there are three key steps. First, we carefully choose a replacement set $R_X$ that is common across terms such that}
&\leq \mathbb{E}_{\boldsymbol{X}}\left[|dev^{\ell_1}(K \setminus D(\boldsymbol{Y}(\boldsymbol{X})) \cup R_{\boldsymbol{X}}) - dev^{\ell_1}(K \setminus D(\boldsymbol{Y'}(\boldsymbol{X})) \cup R_{\boldsymbol{X}})|\right]
\intertext{Next, we apply the first consequence of our coupling argument: that $cond(X_i)$ (which occurs with probability $1-\gamma$) implies $D(\boldsymbol{Y}(\boldsymbol{X})) = D(\boldsymbol{Y'}(\boldsymbol{X}))$, conditional on which the expectation is 0. Then,}
&\leq \gamma \, 
 \mathbb{E}_{\boldsymbol{X}}\left[|dev^{\ell_1}(K \setminus D(\boldsymbol{Y}(\boldsymbol{X})) \cup R_{\boldsymbol{X}}) - dev^{\ell_1}(K \setminus D(\boldsymbol{Y'}(\boldsymbol{X})) \cup R_{\boldsymbol{X}})| \middle\vert \  \neg cond(X_i)\right] 
\intertext{Finally, we apply the second consequence of our coupling argument: that $\neg cond(X_i)$ implies that $D(\boldsymbol{Y}(\boldsymbol{X})) \triangle D(\boldsymbol{Y'}(\boldsymbol{X})) = \{i\}$, which means the symmetric difference of the two sets on which $dev^{\ell_1}$ is evaluated above is also $\{i\}$. Over several steps of reasoning, we use this observation to bound on difference of deviations above in terms of only $i$'s feature-values. The deduction concludes with}
&\leq \gamma \sum_{f \in F} 1/{u_{f, f(i)}}.
\end{align*}

\noindent Putting it all together, $|\mathcal{L}^{\ell_1}(A;  \mathcal{D}_{\dprobs^0}) - \mathcal{L}^{\ell_1}(A;  \mathcal{D}_{\dprobs^k})|$ is upper bounded by the above expression summed up over all $i$, concluding the proof:
\[
  \sum_{i=1}^k \gamma \sum_{f \in F} \frac{1}{u_{f, f(i)}} = \gamma \sum_{f,v \in FV} \frac{\sum_{i=1}^k \mathbb{I}(f(i) = v)}{u_{f,v}} \leq \gamma \sum_{f,v \in FV} \frac{u_{f,v}}{u_{f,v}}= \gamma |FV|. \qedhere\]
\end{proof}

Finally, we show that the above bound is essentially tight. The proof proceeds similarly to that of \Cref{thm:lb-robust-binary} but the construction is more complex, requiring generic numbers of feature-values. We defer the proof to \Cref{app:lb-robust-linear}.
\begin{theorem} [\textbf{Linear loss lower bound}] \label{thm:lb-robust-linear}
Fix any constants $\gamma \in (0, 1), \delta \in (0, 1/2]$ and $|FV|$, $|F|$. There exists a $\mathcal{I}$ with $|F|$ features and $|FV|$ feature-values, $\boldsymbol{\rho}$, and  $\tilde{\boldsymbol{\rho}}$ with $ \|\dprobs - \tilde{\dprobs}\| \leq \err$, and $s(\delta,\gamma)$ such that for all $s \geq s(\delta,\gamma)$, \begin{align*}
        \Pr\left[\mathcal{L}^{\ell_1}(\ermfakealts{\ell_1}; \ddist_{\dprobs}) - \mathcal{L}^{\ell_1}(\ermtruealts{\ell_1}; \ddist_{\dprobs}) \geq \err|FV| - \err|F|\right] \geq 1-2\delta.
    \end{align*}
\end{theorem}

\section{Empirical Evaluation} \label{sec:empirics}

\textbf{Estimating dropout probabilities from historical data.} We estimate the dropout probabilities by fitting the same parametric model as in prior work on sortition, where it was used to estimate agents' probabilities of opting into the pool \cite{flanigan2020neutralizing}. This model, sometimes referred to as \textit{simple independent action}, assumes that each feature-value independently contributes to an agent's probability of dropping out. In particular, the parameters of the model are $\{\beta_0\} \cup \{\beta_{f,v} | f \in F, v \in V_f\}$, and we assume these parameters relate to the dropout probabilities as follows:
\begin{equation}\label{eq:indactmodel}
    \dprob_i = \beta_0 \prod_{f \in F} \beta_{f,f(i)}.
\end{equation}
Colloquially, $\beta_{f,v}$ is the probability of dropping out due to having value $v$ for feature $f$, and $\beta_0$ is the baseline probability of dropping out irrespective of identity. We fit these $\beta$ parameters using maximum likelihood estimation. We use this model of dropout probabilities because it was empirically validated in the sortition context \cite{flanigan2020neutralizing}, but we emphasize that users can select any estimation method, as our algorithms are agnostic to the prediction model.

\vspace{0.3em}
\noindent \textbf{Datasets.} We fit the $\beta$s in \Cref{eq:indactmodel} based on datasets whose rows are agents $i$, and columns are all $i$'s feature-values plus the binary outcome variable describing whether $i$ dropped out. Each \textit{dataset} corresponds to a single instance (real-world assembly); a data \textit{cluster} is a group of datasets with common feature-values (this commonality is necessary for coherently estimating the $\beta$s). Datasets within each cluster correspond to assemblies organized by the same practitioner group.

We train $\beta$s from seven total datasets divided into two separate data clusters: the first cluster contains 3 datasets from US assemblies, named \textit{US-1}, \textit{US-2}, \textit{US-3}, and the second contains 4 datasets from Canadian assemblies, named \textit{Can-1\,-\,Can-4}. We train on as many features present in \Cref{fig:dropout-rates} as possible. Data structure and cleaning details are in \Cref{app:datacleaning}. 

We focus on the US data cluster in the body because these datasets contain extra information relevant for evaluation: the precise quotas used in panel selection, and the alternates used in practice. 
Replication of results in the Canadian datasets are described in \Cref{app:canadian}; we see basically the same trends, but they are less pronounced because the quotas used in these Canadian datasets tend to be far less restrictive (this is a known trend, see, e.g., \cite{flanigan2020neutralizing}).

\vspace{0.3em}
\noindent \textbf{Algorithms.} We run three ERM-based algorithms, \textsc{ERM-Alts}$^{1}$, \textsc{ERM-Alts}$^{0/1}$, and \textsc{ERM-Alts}$^{1}$-\textsc{Eq} (see below), all with $s = 300$ samples (see \Cref{app:convergence} for demonstration of convergence). We compare these algorithms to the benchmarks below, formally specified blue and discussed further in \Cref{app:benchmarks}.

\textsc{Quota-Based} is a proxy for what at least one practitioner group reportedly does in practice: chooses alternates by selecting an entire duplicate panel. Interpreted literally, this method requires $a = k$; to capture this reasoning for generic $a$ we extend their method by scaling down the quotas proportionally to the size of the alternate set. This algorithm is the natural benchmark for what one could do using only existing tools and no knowledge of the dropout probabilities.

\textsc{Greedy} captures the marginal benefit of knowing the dropout probability estimates but not reasoning about the problem's combinatorial structure. \textsc{Greedy} orders the panelists in decreasing order of dropout probability. Then, taking panelists $i \in \{1 \dots a\}$ in order, it finds the remaining pool member whose attributes match $i$ on the most features and adds them to the alternate set.

\textsc{ERM-Alts$^1$-Eq} roughly captures the opposite of greedy: it captures the marginal benefit of reasoning about the combinatorial structure of the problem \textit{without} access to dropout probability estimates, so dropout probabilities are assumed to be equal. Formally, this is \textsc{ERM-alts}$^1$ run with dropout probabilities all equal to the average of the dropout probabilities in that instance, so that the expected number of dropouts is the same as in \textsc{ERM-alts}$^1$.

\textsc{Practitioner Alternates} (represented by a dot) reflects the actual alternate set that was used in that instance. These alternates were essentially chosen by \textsc{Quota-Based}, with a few context-driven adjustments. Note that there are many alternate sets satisfying a given set of quotas, so we do not expect the practitioner alternates to perfectly match the performance of \textsc{Quota-Based}.

We further contextualize our results by showing the loss under two extremes:
$A = \emptyset$, the \textit{worst case} where no alternates are used, and $A = N$, the (unattainable) \textit{best case}, where one can choose alternates from the entire pool (akin to being able to choose alternates \textit{after} seeing the dropout set). 

\vspace{0.3em}
\noindent \textbf{Loss Estimation.} We cannot precisely evaluate an alternate set $A$'s true expected loss $L^{\ell_1}(A;\ddist_{\dprobs})$ due to the exponential support of $\ddist_{\dprobs}$, which motivated our ERM approach. We thus take a sampling approach to evaluation as well: we compute the approximate loss $L^{\ell_1}(A;\hat{\ddist}_{\dprobs})$, where $\hat{\ddist}_{\dprobs}$ is the result of drawing 300 samples from $\ddist_{\dprobs}$\emdash the same sample size used in \textit{running} our ERM algorithms.

\subsection{Evaluation on \textit{realized} dropout set with \textit{estimated} \texorpdfstring{$\tilde{\dprobs}$}{test}} \label{sec:realized-eval}
First, in each instance \textit{US 1-3}, we evaluate the loss of each algorithm on the \textit{actual} dropout set that occurred in that instance, corresponding to a single draw from the distribution $\ddist_{\dprobs}$. As such, this experiment aims to test \textit{what would have happened had each algorithm been run in this assembly.}  

To emulate having trained on historical data, in \textit{US-1} we estimate $\tilde{\dprobs}_1$ on only datasets \textit{US-2} and \textit{US-3}, and analogously we get $\tilde{\dprobs}_2$ and $\tilde{\dprobs}_3$ for the other datasets. As will be the case for multiple analyses in this section, we study how the loss changes as the alternate budget $a$ is increased. 
\begin{figure}[t]
    \centering
\includegraphics[width=\linewidth]{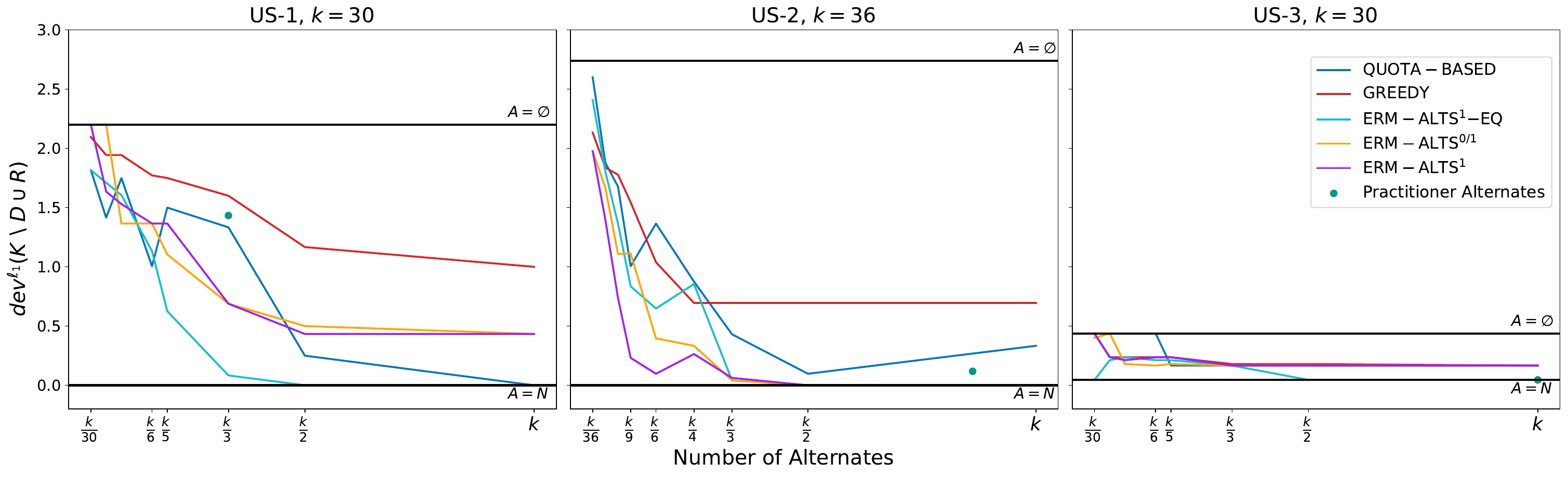}
    \caption{In each instance \textit{US-j} we are given all elements of $\inst$, plus the realization of $D$ and dropout probability estimates $\tilde{\dprobs}_j$. We run each algorithm on all inputs except $D$, and each produces an $A$. What is plotted is then $dev^{\ell_1}(K \setminus D \cup R)$, where $R$ is chosen from each $A$ to minimize $dev^{\ell_1}(K \setminus D \cup R)$.}
    \label{fig:real-dropouts}
\end{figure}
\Cref{fig:real-dropouts} shows that while no algorithm is clearly dominant, the two non-ERM algortihms, \textsc{Quota-Based} and \textsc{Greedy}, have the highest losses for most values of $a$. In \textit{US-2}, \textsc{ERM-Alts}$^{1}$ seems to perform the best; in the \textit{US-1} and \textit{US-3}, \textsc{ERM-Alts}$^{1}$-\textsc{Eq} performs the best (although in \textit{US-3}, only one person dropped out, so the maximum possible loss is very low). 

It may seem surprising that $\algoname^{\ell_1}$ does not always dominate, given that it is precisely optimizing the $L^{\ell_1}$ loss.
There are two possible explanations: $\algoname^{\ell_1}$ performs well on the dropout set \textit{distribution}, but we got unlucky with this realized draw of $D$; or, our estimated dropout probabilities are inaccurate. While one cannot tease apart these explanations conclusively, the dominance of $\algoname^{\ell_1}\textsc{-Eq}$ suggests that inaccurate dropout probability estimates is a contributing factor (since it is just $\algoname^{\ell_1}$ without the dropout probabilities). This is supported empirically: we find that our probability estimates are \textit{well-calibrated} to the realized dropouts in \textit{US-2}, but far from it in \textit{US-1} and \textit{US-3}, precisely tracking where $\algoname^{\ell_1}$ dominates (\Cref{app:calibration}). 

We draw two conclusions from these findings. First, the lower loss of ERM algorithms suggests the importance of considering the problem's combinatorial structure. Second, $\algoname^{\ell_1}$ exhibits strong performance when it has reasonably well-calibrated predictions\emdash and such predictions are possible to achieve. Importantly, the quality of our dropout probability estimates here should be taken as a weak lower bound on the achievable prediction quality in practice. The data we had access to was on-hand and not collected or curated for this purpose, so it was highly imperfect for our prediction task compared to what practitioners could have access to at steady state.\footnote{Our data originally contained 25 panels, but most were dropped due to inconsistent features. The data also did not include features that practitioners anecdotally believe to be predictive of dropping out (e.g., living far from the panel's physical location). Finally, we did not have records of recruitment process issues that might have affected participation in specific datasets (e.g., emails going to spam, or other external shocks). If panel organizers collected consistent features, recorded more dropout-related attributes, and systematically noted external shocks, predictions could likely be significantly improved.}

\subsection{Evaluation on the \textit{dropout set distribution} when \texorpdfstring{$\tilde{\dprobs} = \dprobs$}{what is this}}

We now evaluate these algorithms' $L^{\ell_1}$ loss \textit{in expectation over dropout sets, assuming the ground-truth dropout probabilities are known}. For the sake of realism, we let these ground-truth probabilities $\dprobs$ be those learned by training on \textit{US-1}, \textit{US-2}, and \textit{US-3}. 
Here, we expect \algoname$^{1}$ to outperform all other algorithms, as it is precisely optimizing our evaluation objective; what we are testing here is \textit{the extent} to which various algorithms dominate others, and \textit{how efficiently} each algorithm uses alternates, i.e., how quickly the loss approaches the unattainable optimum of $A = N$.   
\begin{figure}[h!]
    \centering
\includegraphics[width=\linewidth]{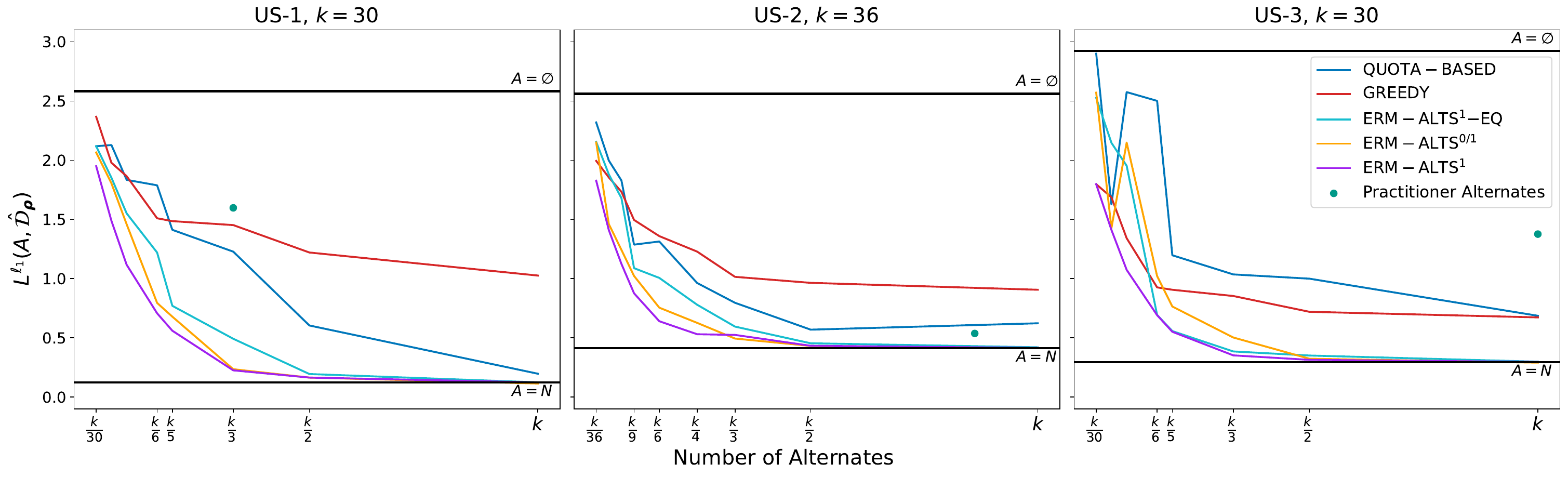}
    \caption{The dropout probabilities known to the relevant algorithms is $\dprobs$. We evaluate loss according to an empirical estimate of $\mathcal{D}_{\dprobs}$, as described under \textbf{Loss Estimation}. Versions of these plots that include standard deviations over $D \sim \mathcal{D}_{\dprobs}$ are in \Cref{app:sim-L1-stdev}.}
    \label{fig:sim-L1}
\end{figure}

In \Cref{fig:sim-L1}, we see that $\algoname^{\ell_1}$ indeed outperforms all other algorithms. It also uses alternates highly efficiently, achieving loss \textit{matching} that of the offline solution (choosing $A = N$) with an alternate budget of only $k/2$, which amounts to $7$\%, $8\%$, and $16\%$ of $n$ across \textit{US-1}, \textit{US-2}, and \textit{US-3}, respectively. Comparing the algorithms more broadly, we see that \textsc{Quota-Based} and \textsc{Greedy} are again reliably dominated by all three ERM algorithms, reinforcing the weakness of heuristics for this problem. Furthermore, the fact that $\algoname^{\ell_1}$ dominates both \textsc{Greedy} (ignores combinatorics) and \textsc{ERM-Alts}$^1$\textsc{-Eq} (ignores dropout probabilities) shows us that both our innovations\emdash learning $\dprobs$ \textit{and} considering the problem's combinatorics\emdash are individually important to achieving low loss.

\algoname$^{1}$'s dominance here is by construction, so we repeat this analysis for five other performance metrics: \textit{the $L^{\ell_1}$ loss counting only violations of lower quotas}, since too few people is generally worse than too many; \textit{the maximum deviation from any quota}, raw and normalized by the quota size; \textit{the number of groups excluded from the panel}; and \textit{the $L^{0/1}$ loss}. The results, in \Cref{app:other-criteria}, show that $\algoname^{\ell_1}$ is reliably one of the lowest-loss algorithms across all these metrics.

\subsection{Robustness: Evaluation on the \textit{dropout set distribution} when \texorpdfstring{$\tilde{\dprobs} \neq \dprobs$}{test}}
Finally, we isolate the effects of prediction errors on the performance of these algorithms. We begin with the ground-truth probabilities $\dprobs$ from the previous experiment, and then perturb each entry of $\dprobs$ such that $\tilde{\dprob}_i \sim Unif([\max\{0,\dprob_i-\gamma\}, \min\{1,\dprob_i + \gamma\}])$. 
The true dropout distribution remains $\ddist_{\dprobs}$, but the algorithms now receive $\tilde{\dprobs}$. We repeat the experiment $r = 25$ times per error level $\gamma$.
\begin{figure}[h!]
    \centering
\includegraphics[width=0.95\linewidth]{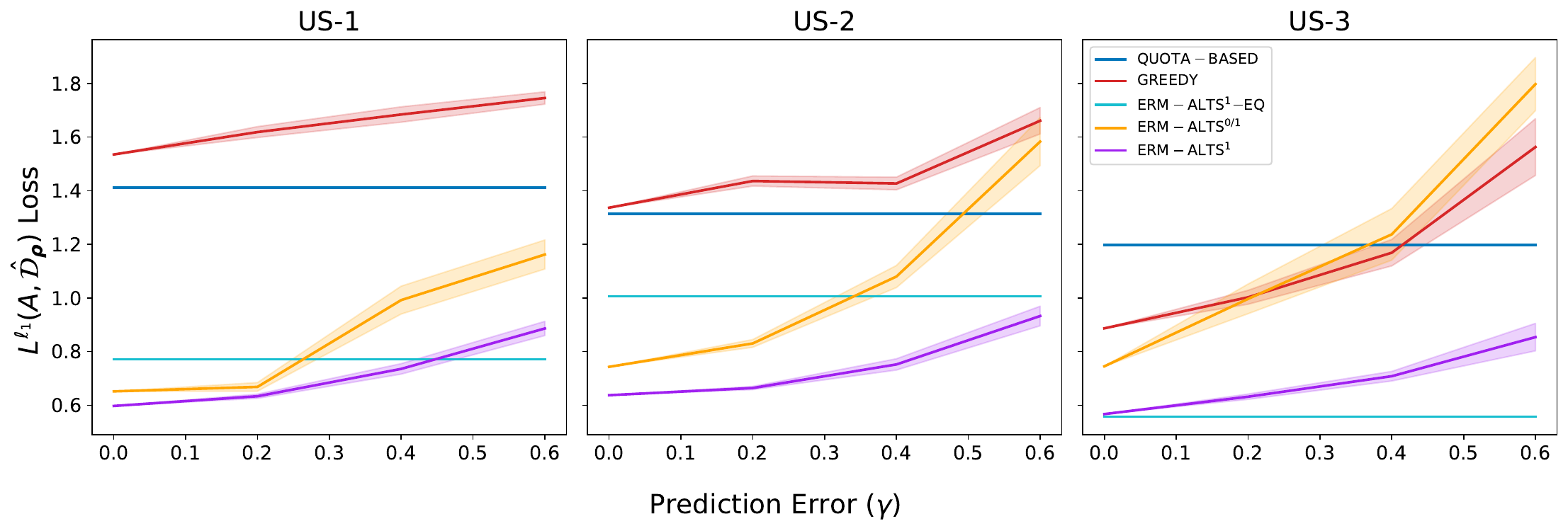}
\vspace*{-0.5em}
    \caption{$a = 6$ (between $k/6$ and $k/5$ for all instances). Per each of the $r = 25$ reps: we run each algorithm (other than the two which don't use dropout probabilities) on $\inst$ with perturbed dropout probabilities $\tilde{\dprobs}$. Then, we evaluate $A$'s loss according to an empirical estimate of $\ddist_{\dprobs}$, as described under \textbf{Loss Estimation}. The plot shows the mean and standard error of the loss over reps, with losses of \textsc{Quota-Based} and \algoname$^{\ell_1}$\textsc{-Eq} shown for comparison.} 
    \label{fig:robustness}
\end{figure}

As expected, all algorithms' loss increases in the prediction error $\gamma$. First comparing algorithms that use predictions, the separation is clear: $\algoname^{\ell_1}$ begins with lower loss \textit{and} its loss increases more slowly in $\gamma$. Comparing $\algoname^{\ell_1}$ and $\algoname^{\ell_1}$\textsc{-Eq}, we see that in \textit{US-1} and \textit{US-2}, large prediction errors are required ($0.4$, $>0.6$) before it is worth ignoring dropout probability estimates and opting for $\algoname^{\ell_1}$\textsc{-Eq}. In \textit{US-3}, $\algoname^{\ell_1}$ and $\algoname^{\ell_1}$\textsc{-Eq} have identical performance under perfect predictions, but as $\gamma$ grows to $0.6$, the gap in loss between the algorithms grows only by 0.25 (where the max possible $L^{\ell_1}$ loss in \textit{US-3} is 2.9, by \Cref{fig:sim-L1}). Moreover, $\algoname^{\ell_1}$'s empirical robustness far outperforms the theoretical bounds shown in \Cref{sec:robustness}.

\section{Discussion} \label{sec:discussion}
Our ERM algorithms are immediately practical: they all\emdash but $\algoname^{\ell_1}$ especially\emdash offer the prospect of drastically decreasing the number of alternates practitioners must select and compensate, while still maintaining or even decreasing the extent to which dropouts compromise representation. Even if a practitioner group does not have the requisite data to predict dropout rates (or simply does not want to), $\algoname^{\ell_1}$\textsc{-Eq} still outperforms both heuristic benchmarks significantly, thereby offering substantial practical improvement even for practitioners without access to predictions.

\textbf{Extensions of our ERM algorithms.} Our ERM algorithms can be generalized in many practical ways. Chief among the remaining practical concerns is that \textit{alternates} may drop out in addition to panelists. Our ERM algorithms can be directly applied to hedge optimally against such additional dropouts, even when alternates drop out according to a different distribution. Furthermore, our ERM algorithms can be used to hedge against dropouts in many different ways beyond selecting alternates: they can be used to select loss-minimizing \textit{extra panelists} while obeying additional quotas; to select loss-minimizing \textit{panels} directly; or to select loss-minimizing panels and alternates in conjunction. We describe the details of all these extensions in Appendix F.

\textbf{Future work: beyond the expected loss.} Currently, our general ERM algorithm focuses on minimizing the \textit{expected} loss. One concern with this approach is that it tells us nothing about more pessimistic properties of the loss distribution, i.e., how much loss we could incur if we were unlucky with the dropout set. An avenue to improve robustness against such scenarios would be to minimize a convex combination of the expected loss and a high percentile of the loss. 
Or, taking a fully worst-case approach, one could also study a minimax variant in which the minimizer chooses $A$ to minimize the loss against a dropout set chosen by the maximizer. We believe that the former direction could be an extension of our methods, while the latter requires separate algorithmic tools.

\textbf{Future work: retaining the randomness of sortition.} Finally, we examine how alternate selection fits within the broader pipeline of
\textit{sortition}, the process of randomly selecting the assembly's original members. Past work on sortition has gone to great lengths to design the randomness of sortition such that pool members' chances of selection for the original panel are \textit{as equal as possible} within the quotas\emdash a property which confers normative ideals like fairness, resistance to manipulation, and more \cite{flanigan2023mini}. Now, as we expand the sortition model to encompass dropouts and the subsequent deployment of alternates, we encounter a barrier to equalizing pool members' chances of selection: existing alternate selection methods and ours alike
choose alternates \textit{deterministically}. 
This determinism can compromise the ideals that past work used finely-tuned randomness to achieve: for example, alternate selection algorithms may be distinctly unfair, privileging groups with lower likelihood of dropping out. Similarly --- while it is impossible to guarantee strategyproof sortition even \textit{without} dropouts --- creating a deterministic pathway to the panel via alternates could significantly worsen the manipulability of the sortition process.

This motivates the natural next research question: \textit{How would one randomize alternate selection?} We give some intuition here. First, suppose we wanted to ``randomize'' our ERM algorithms \textit{without} compromising on robustness to dropout. To do this, we would find as many alternate sets as possible with optimal loss, and then randomize over them with the goal of maximally equalizing pool members' chances of selection, exactly as state-of-the-art sortition algorithms randomize over quota-satisfying panels \cite{flanigan2021fair}. This approach is not guaranteed to work, however; in the worst case, there may be only one such optimal alternate set, and our algorithm would remain deterministic. Intuitively, as we permit ourselves to randomize over alternate sets with increasingly suboptimal loss, we gain the flexibility to randomize over more alternate sets and can  thus further equalize pool members' chances of being selected. This exposes a fundamental trade-off between \textit{the randomness that defines sortition} and \textit{robustness of representation to dropout}; we leave the development of algorithms that strike this trade-off optimally to future work.

\begin{acks}

    We thank Healthy Democracy and MASS LBP for providing data; Manolis Zampetakis for helpful technical conversations; Kyle Redman, Chris Ellis, Linn Davis, Grace Taylor, and Justin Reedy for helpful discussions on sortition dropouts; and Tavor Baharav for the title suggestion. For funding, we thank the ETH Excellence Scholarship Opportunity Programme (CB), the Harvard Data Science Initiative Postdoctoral Fellowship (BF), and NSF IIS-2229881 (AP).

\end{acks}

\bibliographystyle{ACM-Reference-Format}
\bibliography{bibliography}

\newpage
\appendix

\section{Supplemental Materials from \Cref{sec:introduction}}
\subsection{Argument showing NP-hardness} \label{app:reduction}

\begin{lemma}
    Given a panel $K$, an alternate set $A$, a dropout set $D$, and a set of quotas $\boldsymbol{l},\boldsymbol{u}$, it is NP-hard (in asymptotic parameter $a$) to compute the optimal replacement set $R \subseteq A : |R| \leq |D|$.
\end{lemma}
\begin{proof}
    We show that simply \textit{deciding} whether there exists an $R$ such that $K \setminus D \cup A$ satisfies the quotas is NP-hard. Let it be the case that $a = k$, and let the dropout probabilities be constant so that with high probability, $|D| \in \Theta(a)$ (note that we know that $|D| \leq a$). Let the quotas $\boldsymbol{l}$, $\boldsymbol{u}$ be tight. Let $D_{f,v} = \sum_{i \in D} : f(i) = v$ be shorthand for the number of agents in $D$ with value $v$ for feature $f$. Now, let our \textit{modified} quotas $\boldsymbol{l}'$, $\boldsymbol{u}'$ also be tight, such that $l_{f,v} = u_{f,v} = D_{f,v}$ for all $f,v \in FV$.

    Note that the problem of deciding whether there exists an $R$ such that $K \setminus D \cup A$ satisfies the quotas $\boldsymbol{l}$, $\boldsymbol{u}$ is equivalent to checking whether there is a $|D|$-sized subset of $A$ satisfying the modified quotas $\boldsymbol{l}'$, $\boldsymbol{u}'$. By Theorem 1 of \cite{flanigan2021fair} (via a reduction from Set Cover), this problem is known to be NP-hard in $a$.
\end{proof}

\newpage
\section{Supplemental Materials from \Cref{sec:model}}
\subsection{ILP formulation for computing best replacement set} \label{app:replacementILP}

Given a deviation function $dev$, an alternate set $A$, and a dropout set $D$, one can compute the deviation-minimizing replacement set via the following ILP. Because a generic deviation function $dev$ directly intakes a set (a strange variable to represent in an ILP), we will write two ILPs, one for $dev^{0/1}$ and one for $dev^{\ell_1}$. we will write ILPs for  variables are $y_1,\dots,y_a$, where $y_i \in \{0,1\}$ is an indicator of whether alternate $i$ is included in the replacement set $R$. 

$dev^{0/1}$:
\begin{align*}
    \min \quad & \sum_{f,v} z_{f,v} \notag\\
    &\sum_{i \in A} y_{i} \leq |D| \notag\\
    &l_{f,v} - \left(\sum_{i \in K \setminus D} \mathbb{I}(f(i) = v) + \sum_{i \in A} y_{i}\, \mathbb{I}(f(i)=v) \right) \leq z_{f, v} \, (|K|+|A|) &&\text{for all } f \in F, v \in V_f \\
    & - u_{f,v} + \left(\sum_{i \in K \setminus D} \mathbb{I}(f(i) = v) + \sum_{i \in A} y_{i,D} \,\mathbb{I}(f(i)=v) \right)\leq z_{f, v} \, (|K|+|A|)  &&\text{for all } f \in F, v \in V_f \\
    & y_i \in \{0,1\} &&\text{for all } i \in A\\
    & z_{f,v} \in \{0,1\} &&\text{for all } f \in F, v \in V_f
\end{align*}

$dev^{\ell_1}$:
\begin{align*}
    \min \quad & \sum_{f,v} z_{f,v} \notag\\
    &\sum_{i \in A} y_{i} \leq |D| \notag\\
    &l_{f,v} - \left(\sum_{i \in K \setminus D} \mathbb{I}(f(i) = v) + \sum_{i \in A} y_{i}\, \mathbb{I}(f(i)=v) \right) \leq z_{f, v} \, u_{f,v} &&\text{for all } f \in F, v \in V_f \\
    & - u_{f,v} + \left(\sum_{i \in K \setminus D} \mathbb{I}(f(i) = v) + \sum_{i \in A} y_{i,D} \,\mathbb{I}(f(i)=v) \right)\leq z_{f, v} \, u_{f,v} &&\text{for all } f \in F, v \in V_f \\
    & y_i \in \{0,1\} &&\text{for all } i \in A\\
    & z_{f,v} \in \mathbb{R}_{\geq 0} &&\text{for all } f \in F, v \in V_f
\end{align*}

\subsection{Formulation of $\textsc{Opt}(\ddist,\inst)$} \label{app:ILPformulation}
For concreteness, we define this ILP not for generic $dev$, but for our two specific deviation functions $dev^{0/1}$ and $dev^{\ell_1}$. We express these programs simultaneously, where the $(*\,dev^{0/1}\,*)$ and $(*\,dev^{\ell_1}\,*)$ tags denote constraints appearing only in the $dev^{0/1}$ and $dev^{\ell_1}$ versions of \textsc{Opt}, respectively.

\textit{Variables in \textsc{Opt}.} The variables $x_i$ for $i\in N$ are indicator variables of whether pool member $i \in N$ is included in the alternate set. $y_{i,D}$ is then the indicator variable of whether $i$ is chosen to be on the replacement set for dropout set $D$ (note that for $y_{i,D} =1$, it must be that $i$ was in the alternate set to begin with, implying the constraint that $y_{i,D} \leq x_i$). The variables $z_{D,f,v}$ can be thought of as the linear deviation from the quotas, specifically on feature $f$ and value $v$, that is incurred when we use the selected replacement set on dropout set $D$. Finally, the variables $d_D$ capture the deviation (either linear or binary) on set $D$. The objective minimizes the expected value of the deviation function over the dropout set distribution $\mathcal{D}$.\\

\noindent \textsc{Opt}$^{dev}(\inst,\mathcal{D})$ 
\begin{align*}
    \min \ & \sum_{D \in 2^K} d_D \cdot \mathcal{D}(D) \notag\\
    \text{s.t.} \ &\sum_{i\in N} x_i = a \notag\\
    &y_{i,D} \leq x_i &&\forall \ i \in N,\ D \in 2^K \\
    &\sum_{i \in N} y_{i, D} \leq |D| &&\forall \ D \in 2^K\\
    &l_{f,v} - \left(\sum_{i \in K \setminus D} \mathbb{I}(f(i) = v) + \sum_{i \in N} y_{i,D}\, \mathbb{I}(f(i)=v) \right) \leq z_{D, f, v} \, u_{f,v} &&\forall \ D \in 2^K, f,v \in FV \\
    & - u_{f,v} + \left(\sum_{i \in K \setminus D} \mathbb{I}(f(i) = v) + \sum_{i \in N} y_{i,D} \,\mathbb{I}(f(i)=v) \right)\leq z_{D, f, v} \, u_{f,v} &&\forall \ D \in 2^K, f,v \in FV \\
    &\sum_{f,v} z_{D,f,v} \leq d_D \qquad \quad \ \qquad(*\,dev^{\ell_1}\,*) &&\forall \ D \in 2^K \\
    &\sum_{f,v} z_{D,f,v} \leq d_D \cdot (|K| + a) |FV| \qquad(*\,dev^{0/1}\,*) &&\forall \ D \in 2^K\\
    & d_D \in \mathbb{R}_{\geq 0}  \ \ \ \qquad(*\,dev^{\ell_1}\,*)&&\forall \ D \in 2^K\\
    & d_D \in \{0,1\} \qquad(*\,dev^{0/1}\,*) &&\forall \ D \in 2^K\\
    & x_i \in \{0,1\}, \ y_{i,D} \in \{0,1\}, \  z_{D,f,v} \in \mathbb{R}_{\geq 0} && \forall \  i \in N, D \in 2^K, f,v \in FV 
\end{align*}

\newpage
\section{Supplemental Materials from \Cref{sec:algo}}

\subsection{Proof of \Cref{thm:mainbounds} (Lower Bounds)} \label{app:mainbounds}

\begin{proof} 

We begin by proving that our relative values of $n$ and $a$ as constructed in our instance are compatible. We will set $k = |F| = \lfloor \log{n \choose a} \rfloor - 2a$, and we will require that $n \geq a 2^{|F|/a}$. Then, it needs to be the case that for all $a,n \in \mathbb{N}_{\geq 1}$ and $n \geq a$,
\[n \geq a2^{\left(\log{n \choose a} - 2a\right)/a}.\]
\begin{align*}
    a2^{\left(\log{n \choose a} - 2a\right)/a} = a2^{\log{n \choose a}/a - 2} = a/4 \cdot {n \choose a}^{1/a} \leq a/4 \cdot\left(\left(\frac{ne}{a}\right)^{a}\right)^{1/a} = a/4 \cdot \frac{ne}{a} = en/4 < n.
\end{align*}

With our instance parameters established as valid, we proceed by first constructing the instance, then proving the $\VC$ lower bound, then extending that result to prove a $\Pdim$ lower bound.

\vspace{0.4em}
\noindent \textbf{Instance construction.} Construct $\inst$ as follows. Let $k = |F| = \lfloor \log {n \choose a}\rfloor - 2a$. Let the features $f$ be numbered $1 \dots |F|$. Let each feature be binary-valued, so $V_f = \{0,1\}$ for all $f \in F$. Let the quotas be $l_{f,0} = 0,$ $l_{f,1} = 1$ and $u_{f,0} = u_{f,1} = k$ for all $f \in F$, so the only requirement is that the panel contains at least one person with a 1 for every feature. 

Next, we construct this instance's panel and pool. The most efficient way to define them is in terms of agents' feature \textit{vectors}, describing their values across the features; formally, agent $i$'s feature vector is defined as $w(i):=(f(i) | i \in F)$.

\vspace{0.4em}
\noindent \textit{Constructing the panel $K$.} Let the panelists be numbered $i \in [k]$, and let the $i$-th panelist's feature vector be the $i$-th standard $|F|$-length basis vector, so that
        \[
         w(1) = \underbrace{100\dots000}_{\text{length }|F|},\  w(2) = 010\dots000, \dots ,
        \ w(k-1) = 000\dots010, \ w(k) = 000\dots001.
        \]
Note that $K$ satisfies the instance's quotas.
    
\vspace{0.4em}
\noindent \textit{Constructing the pool.} We will construct only a subset of the pool and leave the rest to be constructed arbitrarily. The subset we will construct consists of $a$ groups of agents $G_1,\dots,G_a$. To define these groups, first define corresponding integers $m_1,\dots,m_a$, where $m_j \in \left\{\lfloor|F|/a\rfloor,\lceil|F|/a\rceil\right\}$. Which $m_j$ are assigned the floor versus ceiling doesn't matter so long as $\sum_{j \in [a]} m_j = |F|$. We divide up the features $1 \dots |F|$ into sequential sets of features of size $m_1,\dots,m_a$; call these sequential sets of features $g_j$, so
\begin{align*}
    g_1 = \{1,2,\dots,m_1\},\quad g_2= \{m_1 + 1,m + 2,\dots,m_1 + m_2\}, \quad \dots,\quad g_a = \{\text{$\sum_{j = 1}^{a-1}m_i$} + 1,\dots,a|F|\}.
\end{align*}
We now define our pool subgroups $G_1 \dots G_a$. Each set $G_j$ contains $2^{m_j}$ agents, and is composed of agents $i \in G_j$ who have 0s for all features other than those in $g_j$, and for the $m_j$ features in $g_j$ have values described by a unique vector $u_i \in \{0,1\}^{m_j}$. That is, each $i \in G_j$ is defined by some unique $u_i \in \{0,1\}^g$, and their overall feature vector is then $0\dots 0u_i0\dots 0$, where $u_i$ occurs over the indexes in $g_j$. Because there are exactly $2^{m_j}$ agents in group $G_j$, every unique vector $u_i \in \{0,1\}^{m_j}$ is possessed by exactly one agent in $G_j$. 

\vspace{0.4em}
\noindent \textbf{Construction of shattered set for $\VC$ lower bound.} We will now define a collection of $k$ dropout sets to be shattered, where each dropout set consists of a single panelist. 
    \[D_1 = \{1\}, \ D_2 = \{2\}, \dots, \ D_k = \{k\}.\]
 Fix any vector of labels $b \in \{0,1\}^k$, indexed as $b_i$. Our goal is to construct an alternate set $A_b$ such that $A_b$ ``realizes'' this vector of labels; that is, for all $i \in [k]$, we have that
   \[b_i = h_{A_b}^{0/1}(D_i).\]
   In words, the $1$ label at $b_i$ reflects that $A_b$ cannot fully restore the quotas after $D_i$ drops out, and the $0$ label reflects that it can. Let $c \in \{0,1\}^k$ be the ``opposite'' of $b$, so $c_i = 1 \iff b_i = 0$. We will construct $A_b$ by taking one agent from each set $G_j$, thereby taking $a$ agents overall. From the set $G_j$, add to $A_b$ the agent whose unique vector of values $u_i$ over features $g_j$ matches $c$ over those same indices. Formally, we add agent $i \in G_j$ to the alternate set for whom
  \[f(i) = c_f \qquad \text{ for all } f \in g_j.\] 
In other words, this agent's feature vector is $0\dots0\tilde{c}0\dots0$, where $\tilde{c}$ is a sub-vector of $c$, formed by the entries of $c$ at the indices in $g_j$. Now, it remains to show that
\[b_i = 0 \iff h_{A_b}^{0/1}(D_i) = 0,\]
which proves the claim. Recall that we achieve a $0$ label for dropout set $D_i$ when at least one agent in $A_b$ has a 1 for the $i$-th feature, as this is the 1-value we lose when $D_i = \{i\}$ drops out. 

To prove the forward direction, fix any $i \in [k]$ such that $b_i = 0$. By construction, $A_b$ must contain an agent with a 1 $i$-th position: this is exactly the agent we took from $G_j$ where $i \in g_j$. This agent must have a 1 for the $i$-th feature because we choose them to match the labels in $c$ (the opposite of the labels in $b$) on the features in $g_j$; their $i$th value must be 1 then, since $b_i = 0 \implies c_i = 1$.
 This means that 
    \[b_i = 0 \implies h_{A_b}^{0/1}(D_i) = 0.\]  
To prove the backward direction (by contrapositive), fix an $i \in K$ such that $b_i = 1$, and fix the $j$ such that $g_j \ni i$. The agent added to $A_b$ from group $G_j$ must have a 0 for the $i$th feature by construction of this agent to match $c$ on the indices of $g_j$. The agents added from groups $G_{j'}$ for all $j \in [a] \setminus j$ must have 0s for all features outside $g_{j'}$, which includes index $i$ by assumption. This means that 
    \[b_i = 1 \implies h_{A_b}^{0/1}(D_i) = 1.\]
We have shattered a collection of dropout sets of size $k$, which implies that $VC(\mathcal{H}^{0/1}(\inst)) \geq k = \left\lfloor \log {n \choose a}\right\rfloor - 2a$, and we obtain our tight lower bound of 
    \begin{equation} \label{eq:vc_lb_app}
        VC(\mathcal{H}^{0/1}(\inst)) \geq \left\lfloor \log {n \choose a}\right\rfloor - 2a \geq \log \left(\frac{n}{a}\right)^a - 1 - 2a = a(\log n - \log a -2) - 1 \in \Omega(a \log n).
\end{equation}

\vspace{0.4em}
\noindent \textbf{Extension to $\Pdim$ lower bound.} We will prove our Pdim lower bound by proving something stronger: that for all $\inst'$,
\begin{equation} \label{eq:vc-to-pdim-app}
    \VC(\mathcal{H}^{0/1}(\mathcal{I}')) \geq t \implies \Pdim(\mathcal{H}^{\ell_1}(\mathcal{I}')) \geq t \qquad \text{for all } t \in \mathbb{R}^+.
\end{equation}
Then, it will follow from \Cref{eq:vc_lb_app} that for our instance constructed above,
\[\Pdim(\mathcal{H}^{\ell_1}(\inst)) \in \Omega(a \log n).\]
It remains to prove the statement in \Cref{eq:vc-to-pdim-app}. Suppose we have $\mathcal{I}$ for which $\VC(\mathcal{H}^{0/1}(\mathcal{I})) \geq t$. That means we can find a collection of $t$ dropout sets $\mathbf{D} = \{D_1, \dots, D_t\}$ for which for any labeling $b \in \{0,1\}^t$ there exists $h^{0/1} \in \mathcal{H}^{0/1}(\mathcal{I})$ for which $h^{0/1}(D_i) = b_i$ for all $i \in [t]$

Now, we will show that for this same instance $\inst$ and same collection of $t$ dropout sets $\mathbf{D}$, there exists a witness vector $r = (r_1, \dots, r_t)$ such that for all $b \in \{0, 1\}^t$, there exists a hypothesis $h^{\ell_1} \in \mathcal{H}^{\ell_1}(\mathcal{I})$ for which
    \[
        \mathbb{I}(h^{\ell_1}(D_i) > r_i) = b_i \qquad \text{for all }i \in [t].
    \]
    This is precisely what is required to show that $\Pdim(\mathcal{H}^{\ell_1}(\mathcal{I})) \geq t$.

    Define $r$ such that $r_i = \nicefrac{1}{k+1}$ for all $i \in [t]$, and fix an arbitrary $b \in \{0,1\}^t$. Let the alternate set $A_b$ and its corresponding binary hypothesis $h^{0/1}_{A_b} \in \mathcal{H}^{0/1}$ be the hypothesis that realizes the labeling $b$ in the binary case, i.e., such that
    \[\mathbb{I}(h^{0/1}_{A_b}(D_i) = 1) = b_i,\]
    Note that such a hypothesis must exist by the fact that $\VC(\mathcal{H}^{0/1}(\mathcal{I})) \geq t$. 

    Examining the \textit{linear} hypothesis corresponding to the same alternate set, $h^{\ell_1}_{A_b}$, the key observation --- which concludes the proof --- is that for all dropout sets $D_i, i \in [t]$,
    \begin{equation}
        \label{eq:Pdim-conclusion}
    \mathbb{I}\left(h^{\ell_1}_{A_b}(D_i) > \nicefrac{1}{k+1}\right) \stackrel{(1)}{=} \mathbb{I}\left(h^{0/1}_{A_b}(D_i) = 1\right) \stackrel{(2)}{=} b_i.
    \end{equation}
    Equality (2) holds simply because we chose chose $A_b$ such that $h^{0/1}_{A_b}$ realizes the labeling $b$. The critical step is equality (1), and it follows from the definition of the range of hypotheses in $\mathcal{H}^{\ell_1}(\inst)$. By definition, the range of any hypothesis in this class is fundamentally the range of the $dev^{\ell_1}$, which consists of the sum of rational numbers whose denominators are $u_{f,v}$ for $f \in F, v \in V_f$. Using that all such $u_{f,v} \leq k$, it follows that the smallest possible nonzero rational number in the range of any hypothesis in $\mathcal{H}^{\ell_1}(\inst)$ is
    \[\frac{1}{\max_{f \in F, v \in V_f} u_{f,v}} \geq \frac{1}{k} > \frac{1}{k+1}.\]
    This implies that for any $h^{\ell_1} \in \mathcal{H}^{\ell_1}(\inst)$ and any dropout set $D$,
    \begin{equation} \label{eq:equiv1}
        h^{\ell_1}(D) > \nicefrac{1}{k+1} \iff h^{\ell_1}(D) > 0.
    \end{equation}
    The only required remaining observation is that for all alternate sets $A$ and corresponding hypotheses $h^{\ell_1}_A \in \mathcal{H}^{\ell_1}(\inst)$ and $h^{0/1}_A \in \mathcal{H}^{0/1}(\inst)$, plus all dropout sets $D$, we have that
     \begin{equation} \label{eq:equiv2}
         h^{\ell_1}_A(D) > 0 \iff h^{0/1}_A(D) > 0
     \end{equation}
    because both inequalities correspond exactly to the case that $A$ cannot restore the quotas when $D$ drops out.

    Putting the equivalences together in \Cref{eq:equiv1,eq:equiv2}, we get that for all $A$ and all $D$,
    \[h^{\ell_1}_A(D) > \nicefrac{1}{k+1} \iff h^{0/1}_A(D) > 0,\]
    exactly as needed to show equality (1) in \Cref{eq:Pdim-conclusion}.
\end{proof}

\subsection{Proof of \Cref{thm:FdepVCPdim}} \label{app:FdepVCPdim}

\begin{proof}[Proof: Upper Bounds]
The upper bound is proven by bounding the size of the hypothesis class in terms of $|F|$ rather than $n$ and $a$. Note that the maximum number of possible unique feature vectors that could ever occur in an instance is 
    \[m = \prod_{f \in F} |V_f| \leq \left(\max_{f\in F}|V_f|\right)^{|F|}\]
    Then, the maximum \textit{unique} possible alternate sets in any instance is defined by the number of ways to choose $a$ alternates from a set of $m$ unique elements, potentially with duplicates. It is well-known that the way this is as follows, where we are slightly abusing notation here to now treat $\mathcal{A}$ as the set of all \textit{unique} alternate sets: 
    \[|\mathcal{A}| \leq {m + a - 1 \choose a} \leq \left(\frac{e(m + a - 1)}{a}\right)^a.\]
    Taking the logarithm of this quantity and plugging in our upper bound for $m$, we get
    \begin{align*}
        |\mathcal{A}| \leq a \, \left( \log(m + a - 1) - \log a + 1\right) &\leq a \, \left( \log\left(\left(\max_{f\in F}|V_f|\right)^{|F|} + a - 1\right) - \log a + 1\right)\\
        &\leq a \, \left( \log\left(\max_{f\in F}|V_f|\right)^{|F|} +1 \right)\tag{*}\\
    &=a\left(|F|\log\max_{f\in F}|V_f|+1\right)
\end{align*}
The starred step is shown as follows. Consider the inequality written more simply as 
\[\log(x + a -1) - \log a +1 \leq \log(x) + 1 \iff \log(x + a -1) - \log a - \log(x) \leq 0.\]
Take the derivative of the LHS with respect to $x$, and get 
\[\frac{1}{a+x-1} - \frac{1}{x}.\]
This is weakly negative for all $x \geq 1, a \geq 1$, both of which are true in any non-degenerate instance of our problem. This means the original LHS is maximized at $x=1$. At $x=1$,
\begin{align*}
    \log(1 + a -1) - \log a - \log(1) = 0 \leq 0,
\end{align*}
concluding the proof.
\end{proof}

\begin{proof}[Proof: Lower Bounds]
The proof of the lower bound proceeds very similarly to the lower bound in \Cref{thm:mainbounds}. 

\textbf{Instance construction.} Construct $\inst$ as follows. Let $|F| \leq k$, and let $n \geq a 2^{|F|/a}$. Let the features $f$ be numbered $1 \dots |F|$. Again, $V_f = \{0,1\}$ for all $f \in F$ and the quotas are again $l_{f,0} = 0,$ $l_{f,1} = 1$ and $u_{f,0} = u_{f,0} = \infty$ for all $f \in F$. The panel is constructed as in the proof of \Cref{thm:mainbounds}, except now there are $|F|$ panelists $1,2,\dots,|F|$ with unique $|F|$-length basis vectors, with panelist $i$ having value 1 for feature $i$ and 0s for all other features. For the remaining $k - |F|$ panelists, let them have value 0 for all features. We construct the pool exactly as in the proof of \Cref{thm:mainbounds}. 

\textbf{Construction of shattered set for $\VC$ lower bound.} We construct the collection of dropout sets to be shattered as before, except now we construct only $|F|$ of them, from panelists $1 \dots |F|$: 
    \[D_1 = \{1\}, \ D_2 = \{2\}, \dots, \ D_{|F|} = \{|F|\}.\]
 Analogously to before, we fix any vector of labels $b \in \{0,1\}^{|F|}$, indexed as $b_i$. We construct our alternate set $A_b$ exactly as before, taking one person from each set $G_j$. The rest of the proof of the $\VC$ lower bound follows for the same reason, as does the extension to the $\Pdim$ lower bound, giving us lower bounds of 
 \[\VC(\mathcal{H}^{0/1}(\inst)), \ \Pdim(\mathcal{H}^{\ell_1}(\inst)) \geq |F|. \qedhere\] 
\end{proof}

\subsection{Restatement of Existing PAC-Learning Bounds} \label{app:restatement}
We use the following known PAC learning bounds, along with the fact that uniform convergence is sufficient for agnostic learnability (Theorem 6.7, \cite{shalev2014understanding}). 
\begin{lemma}[Theorem 6.8 of \cite{shalev2014understanding}] \label{thm:vc-existing-bounds}
Fix constants $\varepsilon, \delta > 0$, and let $\mathcal{H}$ is a hypothesis class of functions with range $\{0,1\}$ such that $\Pdim(\mathcal{H}) < \infty$. Then, there exists a constant $C \in \mathbb{R}^+$ such that $\mathcal{H}$ has the uniform convergence property with sample complexity 
    \[
       s(\delta,\varepsilon) \leq C\cdot \frac{\VC(\mathcal{H})  + \log\left( \frac{1}{\delta}\right)}{\varepsilon^2}.
    \]
\end{lemma}

\begin{lemma}[\cite{FetayaLectureNotes2016}]\label{lem:pdimexistingbound}
Fix constants $\varepsilon, \delta > 0$, and let $\mathcal{H}$ is a hypothesis class of functions with range in $\mathbb{R}$ such that $\Pdim(\mathcal{H}) < \infty$. Then, there exists a constant $C \in \mathbb{R}^+$ such that $\mathcal{H}$ has the uniform convergence property with sample complexity 
\[ s(\delta, \varepsilon) \leq C\cdot \frac{\Pdim(\mathcal{H})\cdot \log\left( \frac{1}{\varepsilon}\right) + \log\left( \frac{1}{\delta}\right)}{\varepsilon^2}
.\]
\end{lemma}

\newpage
\section{Supplemental Materials from \Cref{sec:robustness}}
\subsection{Proof of \Cref{thm:lb-robust-binary}} \label{app:lb-robust-binary}

\begin{proof}
We construct a simple instance $\mathcal{I} = (N, K, \boldsymbol{l}, \boldsymbol{u}, a)$ as follows: there is only one binary feature, so $F = {f_1}$, and $V_{f_1} = \{0,1\}$. The panel $K$ is of size $k = 2 \lceil\log_{(1-\gamma)}\alpha\rceil$ (note: must be $\geq 2$, even, and positive) and it is comprised of $k/2$ agents with feature-value 0 and $k/2$ agents with feature-value 1. The quotas are tight, so $l_{f_1,1} = u_{f_1,1} = l_{f_1,0} = u_{f_1,0} = k/2$. The pool $N$ contains at least $k$ agents with feature-value 0 and at least $k$ agents with feature-value 1, so our alternate set construction is unencumbered by any limitations of the pool. Finally, $a = k/2$. We set $\boldsymbol{\rho}$ and $\tilde{\boldsymbol{\rho}}$ as follows, noting that indeed $\| \boldsymbol{\rho} - \tilde{\boldsymbol{\rho}}\|_{\infty} = \gamma$.
\begin{align*}
    \rho_i = \begin{cases}
        \gamma &\text{if } f_1(i) = 1\\
        0 &\text{if } f_1(i) = 0
    \end{cases} \qquad \tilde{\rho}_i = \begin{cases}
        0 &\text{if } f_1(i) = 1\\
        \gamma &\text{if } f_1(i) = 0
    \end{cases} \qquad \text{ for all $i \in K$.}
\end{align*}
First, observe that $\opttruealts{0/1}$ just consists of $a=k/2$ agents $i$ with $f_1(i) = 1$, and $\optfakealts{0/1}$ consists of $a=k/2$ agents $i$ with $f_1(i) = 0$. This is because under both $\boldsymbol{\rho}$ and $\tilde{\boldsymbol{\rho}}$, exactly $a=k/2$ members of the panel have a non-zero chance of dropping out, and both of these optimal sets simply have backups for all of these agents, and no other agents. Note that $\mathcal{L}^{0/1}(\opttruealts{0/1},\ddist_{\dprobs}) = 0$. Further, 
\[\mathcal{L}^{0/1}(\optfakealts{0/1},\ddist_{\dprobs}) = 1 \cdot \Pr_{D \sim \mathcal{D}_{\boldsymbol{\rho}}}[\exists \ i \in D \colon f_1(i) = 1] = 1 - (1-\gamma)^{k/2} \geq 1-\alpha,\]
because $\optfakealts{0/1}$ contains only agents with $f(i) = 0$ and therefore must incur $dev^{0/1}$ if \textit{any} panelist with $f(i) = 1$ drops out. It follows that
\[\mathcal{L}^{0/1}(\optfakealts{0/1};\ddist_{\dprobs}) - \mathcal{L}^{0/1}(\opttruealts{0/1};\ddist_{\dprobs}) \geq 1-\alpha.\]
Finally, it remains to show that for sufficient $s$, 
\[\Pr[\ermtruealts{0/1} = \opttruealts{0/1} \land \ermfakealts{0/1} = \optfakealts{0/1}]\geq 1-2\delta,\] 
which implies the claim. This argument proceeds two steps: (1) showing a constant separation between the loss of $\opttruealts{0/1}$ and that of any other alternate set $A$, and a symmetric separation for $\optfakealts{0/1}$; and then (2) applying \Cref{cor:binary-sample-complexity} to derive $s(\alpha,\delta,\gamma)$ such that for all $s \geq s(\alpha,\delta,\gamma)$, the desired event occurs with at least $1-2\delta$ probability.

First, consider any alternate set $A \neq \opttruealts{0/1}$. $A$ contains $<k/2$ alternates with value 1, and thus it incurs $dev^{0/1}$ of $1$ when all $k/2$ panelists with value $1$ drop out. This dropout set occurs with probability $\gamma^{k/2}$. Hence, 
\begin{equation} \label{eq:Aloss}
    \mathcal{L}^{0/1}(A;\ddist_{\dprobs}) - \mathcal{L}^{0/1}(\opttruealts{0/1};\ddist_{\dprobs}) \geq \gamma^{k/2} > \gamma^k.
\end{equation}
Now, recall that $k = 2 \lceil \log_{(1-\gamma)}\alpha\rceil$ and let $s(\alpha,\delta,\gamma) \in \Theta\left(\frac{k/2 \cdot \log(k) + \log(1/\delta)}{\gamma^{k}}\right)$.
By \Cref{cor:binary-sample-complexity}, we know that for all $s \geq s(\alpha,\delta,\gamma)$,
\[\Pr\left[|\mathcal{L}^{0/1}(\ermtruealts{0/1};\ddist_{\dprobs}) - \mathcal{L}^{0/1}(\opttruealts{0/1};\ddist_{\dprobs})| \leq \gamma^{k}\right] \geq 1-\delta.\] 
Combining this fact with \Cref{eq:Aloss}, we get that
\[\Pr[\ermtruealts{0/1} = \opttruealts{0/1}] \geq 1-\delta.\]
By symmetry, under $\mathcal{D}_{\tilde{\boldsymbol{\rho}}}$ the same gap in $0/1$ loss is induced when all panelists with \textit{value 0} drop out, and thus for all $s \geq s(\alpha,\delta,\gamma)$, we have that
\[\Pr[\ermfakealts{0/1} = \optfakealts{0/1}] \geq 1-\delta.\]
By union bounding, we get that 
\[\Pr[\ermtruealts{0/1} = \opttruealts{0/1} \land \ermfakealts{0/1} = \optfakealts{0/1}]\geq 1-2\delta. \qedhere\]
\end{proof}

\subsection{Proof of \Cref{thm:ub-robust-linear}} \label{app:ub-robust-linear}

In these proofs, it will be convenient to use the following shorthand: For any $(f,v) \in FV$, define the feature-value linear deviation of a set $S$ as:\[
dev^{\ell_1}_{f,v}(S, \mathcal{I}) = \frac{\max\{0, l_{f,v} - \sum_{i \in S} \mathbb{I}(f(i) = v), -u_{f,v} + \sum_{i \in S} \mathbb{I}(f(i) = v)\}}{u_{f,v}}
\]
Then note that then we can express the overall linear deviation as the sum of the feature-value deviations: $dev^{\ell_1}(S, \mathcal{I}) = \sum_{(f, v) \in FV} dev^{\ell_1}_{f,v}(S, \mathcal{I})$. 

\begin{proof}
Because the instance will be fixed throughout the proof, we will drop $\inst$ from all our $\mathcal{L}^{\ell_1}$ and $dev^{\ell_1}$ functions, leaving it implicit.

The core of the proof is showing the following bound on the change in linear loss for \textit{any} fixed alternate set $A$, when evaluated with respect to $\mathcal{D}_{\dprobs}$ versus $\mathcal{D}_{\tilde{\dprobs}}$:
\begin{equation}\label{eq:linearublemma-app}
    \vert\mathcal{L}^{\ell_1}(A;\mathcal{D}_{\dprobs}) - \mathcal{L}^{\ell_1}(A;\mathcal{D}_{\tilde{\dprobs}})\vert  \leq \err |FV|.
\end{equation}
Once we have this bound, we can apply it to $\opttruealts{\ell_1}$ to show the following chain of inequalities, where the first inequality is by the optimality of $\optfakealts{\ell_1}$ for $\ddist_{\tilde{\dprobs}}$:
\[\mathcal{L}^{\ell_1}(\optfakealts{\ell_1};\ddist_{\tilde{\dprobs}}) - \mathcal{L}^{\ell_1}(\opttruealts{\ell_1};\ddist_{\dprobs})\leq \mathcal{L}^{\ell_1}(\opttruealts{\ell_1};\ddist_{\tilde{\dprobs}}) - \mathcal{L}^{\ell_1}(\opttruealts{\ell_1};\ddist_{\dprobs})\leq  \gamma|FV|.\]
This gives us an almost analogous version of our desired bound for $\opttruealts{\ell_1}$ and $\optfakealts{\ell_1}$. To relate $\mathcal{L}^{\ell_1}(\ermfakealts{\ell_1};\ddist_{\dprobs})$ to $\mathcal{L}^{\ell_1}(\optfakealts{\ell_1};\ddist_{\tilde{\dprobs}})$, we apply \Cref{eq:linearublemma} once more and derive $s(\varepsilon,\delta,\gamma)$ based on \Cref{cor:sample-complexity-linear} to ensure that with probability $\geq 1-\delta$, the loss of $\ermfakealts{\ell_1}$ on $\mathcal{D}_{\tilde{\boldsymbol{\rho}}}$ with $s \geq s(\varepsilon,\delta,\gamma)$ is within $\varepsilon$ of its respective corresponding optimal alternate set. As we are showing an upper bound, we do not have to ensure that $\ermtruealts{\ell_1}$ is close to $\opttruealts{\ell_1}$. This completes the proof.

The high level approach to proving \Cref{eq:linearublemma} is to construct $k+1$ intermediate probability vectors that incrementally transform $\dprobs$ to $\tilde{\dprobs}$ by altering one agent's dropout probability at a time. As such, let $\dprobs^i = (\tilde{\dprob}_1, \dots, \tilde{\dprob}_i, \dprob_{i+1}, \dots, \dprob_k)$ for all $i \in [k]$. Then our quantity of interest can be rewritten as $\vert \mathcal{L}^{\ell_1}(A;  \mathcal{D}_{\dprobs^0}) - \mathcal{L}^{\ell_1}(A;  \mathcal{D}_{\dprobs^k})\vert$. Using a telescoping sum and the triangle inequality, we have that 
\begin{align*}
    \mathcal{L}^{\ell_1}(A;  \mathcal{D}_{\dprobs^0}) - \mathcal{L}^{\ell_1}(A;  \mathcal{D}_{\dprobs^k})\vert &= \left\vert \sum_{i=1}^{k} \mathcal{L}^{\ell_1}(A; \mathcal{D}_{\dprobs^{i-1}}) - \mathcal{L}^{\ell_1}(A; \mathcal{D}_{\dprobs^i})\right\vert \\
    &\leq \sum_{i=1}^{k} \left \vert \mathcal{L}^{\ell_1}(A; \mathcal{D}_{\dprobs^{i-1}}) - \mathcal{L}^{\ell_1}(A; \mathcal{D}_{\dprobs^i}) \right \vert.
\end{align*}
We will bound each individual term of this resulting sum separately. This will be done via a coupling argument, where we couple the random dropouts under $\dprobs^i$ and $\dprobs^{i+1}$. Formally, we are coupling \[
\boldsymbol{Y} = (Y_j \sim \text{ Bernoulli}(\rho^{i-1}_j)|j \in [k]) \quad \text{ and } \quad \boldsymbol{Y'} = (Y'_j \sim \text{ Bernoulli}(\rho^{i}_j)|j \in [k]),
\]
whose entries describe whether each panelist dropped out when dropouts were sampled under the $\dprobs^{i}$ and $\dprobs^{i+1}$ vector, respectively. $\boldsymbol{Y}$ and $\boldsymbol{Y'}$ are sampled the following coupled sampling process, which maintains the marginal distributions of $\boldsymbol{Y}$ and $\boldsymbol{Y'}$ as compared to the Bernoulli processes described above, and therefore does not affect any expected value of interest. Fix a sequence of random values $\boldsymbol{X} = (X_1,\dots,X_k)$, each drawn independently from the uniform distribution on $[0,1]$. Then, $Y_j$ and $Y'_j$ depend on the $X_j$ as follows:
\begin{align*}
    Y_i(\boldsymbol{X}) = \begin{cases}
        1 \text{ if }X_i \leq \rho_i\\
        0 \text{ else}
    \end{cases} 
    \quad
     Y'_i(\boldsymbol{X}) = \begin{cases}
        1 \text{ if }X_i \leq \tilde{\rho}_i\\
        0 \text{ else}
    \end{cases}
    \quad 
    Y_j(\boldsymbol{X}), Y'_j(\boldsymbol{X}) = \begin{cases}
        1 \text{ if }X_j \leq \tilde{\rho}_j \text{ and } j < i\\
        1 \text{ if }X_j \leq \rho_j \text{ and } j > i\\
        0 \text{ else}
    \end{cases}.
\end{align*}
We write $\boldsymbol{Y}(\boldsymbol{X}) = (Y_j(\boldsymbol{X}) | j \in [k])$ and $\boldsymbol{Y'}(\boldsymbol{X}) = (Y'_j(\boldsymbol{X}) | j \in [k])$. Note that the distribution of $\boldsymbol{Y}(\boldsymbol{X})$ (resp.~$\boldsymbol{Y'}(\boldsymbol{X})$) per this sampling process is identical to the distribution of $\boldsymbol{Y}$ (resp.~$\boldsymbol{Y'}$) induced by sampling each entry of $\boldsymbol{Y}$ (resp.~$\boldsymbol{Y'}$) according to independent Bernoulli random draws: the marginals are the same, and each entry of $\boldsymbol{Y}$ (resp.~$\boldsymbol{Y'}$) remains independent.

Let $D(\boldsymbol{Y}(\boldsymbol{X})) = \{j \in [k] | Y_j = 1\}$ be the dropout set specified by $\boldsymbol{Y}$ and define $D(\boldsymbol{Y'}(\boldsymbol{X}))$ equivalently. Define $R(\boldsymbol{Y}(\boldsymbol{X})) := \argmin_{R \subseteq A} dev^{\ell_1}(K \setminus D(\boldsymbol{Y}(\boldsymbol{X})) \cup R)$ to be the optimal replacement set for $D(\boldsymbol{Y}(\boldsymbol{X}))$, and define $R(\boldsymbol{Y'}(\boldsymbol{X}))$ analogously. 

The goal of the coupling argument is to ensure that $D(\boldsymbol{Y}(\boldsymbol{X}))$ and $D(\boldsymbol{Y'}(\boldsymbol{X}))$ are similar. In fact:

\textit{Fact 1:}
    When $X_i \leq \min(\rho_i, \tilde{\rho}_i)$ or $X_i > \max(\rho_i, \tilde{\rho}_i)$,\
$D(\boldsymbol{Y}(\boldsymbol{X})) = D(\boldsymbol{Y'}(\boldsymbol{X}))$ and consequently $R(\boldsymbol{Y}(\boldsymbol{X})) = R(\boldsymbol{Y'}(\boldsymbol{X}))$.

\textit{Fact 2:} When $X_i \in (\min(\rho_i, \tilde{\rho}_i), \max(\rho_i, \tilde{\rho}_i)]$, \ $D(\boldsymbol{Y}(\boldsymbol{X})) \triangle D(\boldsymbol{Y'}(\boldsymbol{X})) = \{i\}$ ($\triangle$ is the symmetric difference).

We can express our quantity of interest in terms of these quantities and then apply the law of total expectation, linearity of expectation, and Jensen's inequality to get that
\begin{align*}
|\mathcal{L}^{\ell_1}(A; \,\mathcal{D}_{\dprobs^{i-1}}) &- \mathcal{L}^{\ell_1}(A;  \mathcal{D}_{\dprobs^i})|\\
&= \left \vert \mathbb{E}_{\boldsymbol{Y}} \left[dev^{\ell_1}(K \setminus D(\boldsymbol{Y}) \cup R(\boldsymbol{Y}))\right] - \mathbb{E}_{\boldsymbol{Y'}} \left[dev^{\ell_1}(K \setminus D(\boldsymbol{Y'}) \cup R(\boldsymbol{Y'}))\right] \right \vert\\
\qquad\qquad &=\left \vert \mathbb{E}_{\boldsymbol{X}}\left[\mathbb{E}_{\boldsymbol{Y}} \left[dev^{\ell_1}(K \setminus D(\boldsymbol{Y}) \cup R(\boldsymbol{Y}))\middle \vert \boldsymbol{X}\right] \right] - \mathbb{E}_{\boldsymbol{X}} \left[\mathbb{E}_{\boldsymbol{Y'}} \left[ dev^{\ell_1}(K \setminus D(\boldsymbol{Y'}) \cup R(\boldsymbol{Y'}))\middle \vert \boldsymbol{X}\right] \right]\right \vert\\
&= \left \vert \mathbb{E}_{\boldsymbol{X}}\left[dev^{\ell_1}(K \setminus D(\boldsymbol{Y}(\boldsymbol{X})) \cup R(\boldsymbol{Y}(\boldsymbol{X})))\right] - \mathbb{E}_{\boldsymbol{X}} \left[dev^{\ell_1}(K \setminus D(\boldsymbol{Y'}(\boldsymbol{X})) \cup R(\boldsymbol{Y'}(\boldsymbol{X})))\right]\right \vert \\
&\leq \mathbb{E}_{\boldsymbol{X}}\left[| dev^{\ell_1}(K \setminus D(\boldsymbol{Y}(\boldsymbol{X})) \cup R(\boldsymbol{Y}(\boldsymbol{X}))) - dev^{\ell_1}(K \setminus D(\boldsymbol{Y'}(\boldsymbol{X})) \cup R(\boldsymbol{Y'}(\boldsymbol{X})))|\right].
\intertext{To upper bound this expression, we select the same replacement set for both $D(\boldsymbol{Y}(\boldsymbol{X}))$ and $D(\boldsymbol{Y'}(\boldsymbol{X}))$ for a fixed realization of $\boldsymbol{X}$. Let this set $R_{\boldsymbol{X}}$ be defined as follows: if the difference inside the absolute value is nonnegative, let $R_{\boldsymbol{X}} = R(\boldsymbol{Y'}(\boldsymbol{X}))$, otherwise $R_{\boldsymbol{X}} = R(\boldsymbol{Y}(\boldsymbol{X}))$. For each possible realization of $\boldsymbol{X}$, this fixes the optimal replacement set for the smaller term and substitutes it in for the optimal replacement set for the larger term, thereby upper-bounding the absolute value of the difference.}
&\leq \mathbb{E}_{\boldsymbol{X}}\left[|dev^{\ell_1}(K \setminus D(\boldsymbol{Y}(\boldsymbol{X})) \cup R_{\boldsymbol{X}}) - dev^{\ell_1}(K \setminus D(\boldsymbol{Y'}(\boldsymbol{X})) \cup R_{\boldsymbol{X}})|\right]
\intertext{To simplify notation, we will henceforth let $S(\boldsymbol{X}) = K \setminus D(\boldsymbol{Y} (\boldsymbol{X})) \cup R_{\boldsymbol{X}}$ and $S'(\boldsymbol{X}) = K \setminus D(\boldsymbol{Y'}(\boldsymbol{X})) \cup R_{\boldsymbol{X}}$ be the panels after dropout and replacement:}
&=\mathbb{E}_{\boldsymbol{X}}\left[|dev^{\ell_1}(S'(\boldsymbol{X})) - dev^{\ell_1}(S'(\boldsymbol{X}))\right]
\intertext{By \textit{Fact 1}, when $X_i \leq \min(\rho_i, \tilde{\rho}_i)$ or $X_i > \max(\rho_i, \tilde{\rho}_i)$, the expression in the expectation evaluates to 0. Thus, it can be simplified to only consider the case when $X_i \in (\min(\rho_i, \tilde{\rho}_i), \max(\rho_i, \tilde{\rho}_i)]$, which (due to the distribution of $X_i$) occurs with probability $\max(\rho_i, \tilde{\rho}_i) - \min(\rho_i, \tilde{\rho_i}) \leq \gamma$:}
&\leq \mathbb{E}_{\boldsymbol{X}}\left[|dev^{\ell_1}(S'(\boldsymbol{X})) - dev^{\ell_1}(S'(\boldsymbol{X}))| \middle\vert \  X_i \in (\min(\rho_i, \tilde{\rho}_i), \max(\rho_i, \tilde{\rho}_i)]\right] \cdot \gamma 
\intertext{Now expanding by $f,v$ and applying the triangle inequality,}
    &\leq \mathbb{E}_{\boldsymbol{X}} \left[ \sum_{f \in F} \sum_{v \in V_f} \left \vert dev^{\ell_1}_{f,v}(S({\boldsymbol{X}})) - dev^{\ell_1}_{f,v}(S'({\boldsymbol{X}})) \right \vert\ \middle\vert \  X_i \in (\min(\rho_i, \tilde{\rho}_i), \max(\rho_i, \tilde{\rho}_i)]\right] \cdot \gamma
\intertext{By \textit{Fact 2}, when $X_i \in (\min(\rho_i, \tilde{\rho}_i), \max(\rho_i, \tilde{\rho}_i)]$, \ $D(\boldsymbol{Y}(\boldsymbol{X})) \triangle D(\boldsymbol{Y'}(\boldsymbol{X})) = \{i\}$. It follows that $S(\boldsymbol{X}) \triangle S'(\boldsymbol{X}) = \{i\}$. Then, for all feature-values $i$ does not possess (all $f,v : f(i) \neq v$), we have that $|dev^{\ell_1}_{f,v}(S({\boldsymbol{X}})) - dev^{\ell_1}_{f,v}(S'({\boldsymbol{X}}))| = 0$. Then,}
&\leq \mathbb{E}_{\boldsymbol{X}} \left[ \sum_{f \in F}  |dev^{\ell_1}_{f,f(i)}(S(\boldsymbol{X}))-dev^{\ell_1}_{f,f(i)}(S'(\boldsymbol{X}))| \middle\vert \  X_i \in (\min(\rho_i, \tilde{\rho}_i), \max(\rho_i, \tilde{\rho}_i)]\right] \cdot \gamma
\intertext{Now, to prove the final step of the deduction, assume without loss of generality that $S(\boldsymbol{X}) = S'(\boldsymbol{X}) \dot{\cup} \{i\}$. Then for any $f \in F$, there is some $z_f \in \mathbb{N}$ such that $\sum_{j \in S'(\boldsymbol{X})} \mathbb{I}(f(j) = f(i)) = z_f$ and $\sum_{j \in S(\boldsymbol{X})} \mathbb{I}(f(j) = f(i)) = z_f+1$. We can then rewrite the feature-level deviations as \[dev^{\ell_1}_{f,f(i)}(S(\boldsymbol{X})) = \frac{1}{u_{f,f(i)}}\max\left\{0, l_{f,v} - (z_f+1), -u_{f,v} + z_f+1\right\}\] and
\[dev^{\ell_1}_{f,f(i)}(S'(\boldsymbol{X})) = \frac{1}{u_{f,f(i)}} \max\left\{0, l_{f,v} - z_f, -u_{f,v} + z_f\right\}.\]
There are three relevant domains for both $z_f,z_f+1$ to consider: $< l_{f,v}, \  [l_{f,v}, u_{f,v}], \text{ and} \ > u_{f,v}$. Note that because $z_f, z_f+1, u_{f,v}$, and $l_{f,v}$ are all integers, it cannot be that $z_f < l_{f,v} \leq u_{f,v} < z_f+1$. Thus, $z_f$ and $z_f+1$ must either be in the same domain or in neighboring domains. The proof proceeds by showing that, regardless of which of the five possible cases of which domains $z_f$ and $z_f+1$ fall into, $|dev^{\ell_1}_{f,f(i)}(S(\boldsymbol{X})) - dev^{\ell_1}_{f,f(i)}(S('\boldsymbol{X}))| \leq 1/u_{f,f(i)}$. Because the arguments all use exactly the same approach, we show only the first such case, where $z_f,z_f+1 < l_{f,v}$.
\textit{Case 1:} If $z_f,z_f+1 < l_{f,v}$, then the dominating term of both $\max\left\{0, l_{f,v} - (z_f+1), -u_{f,v} + z_f+1\right\}$ and $\max\left\{0, l_{f,v} - z_f, -u_{f,v} + z_f\right\}$ are the second one; Then we get
    \[|dev^{\ell_1}_{f,f(i)}(S(\boldsymbol{X})) - dev^{\ell_1}_{f,f(i)}(S('\boldsymbol{X}))| = \left|1/u_{f,f(i)}((l_{f,v}-z_f) - l_{f,v}-z_{f}-1))\right| = \left|1/u_{f,f(i)}\right|.\]
Carrying out cases 2-5 similarly, we conclude that the final step in the overall deduction:}
&\leq \gamma \sum_{f \in F} \frac{1}{u_{f, f(i)}}
\end{align*}

Putting it all together gives
\begin{align*}
    &|\mathcal{L}^{\ell_1}(A;  \mathcal{D}_{\dprobs^0}) - \mathcal{L}^{\ell_1}(A;  \mathcal{D}_{\dprobs^k})|\\
    &\qquad \leq \sum_{i=1}^k \gamma \sum_{f \in F} \frac{1}{u_{f, f(i)}} = \gamma \sum_{f \in F} \sum_{v \in V_f} \frac{\sum_{i=1}^k \mathbb{I}(f(i) = v)}{u_{f,v}} \leq \gamma \sum_{f \in F} \sum_{v \in V_f} \frac{u_{f,v}}{u_{f,v}}= \gamma |FV|.
\end{align*}
Now that we have proven \Cref{eq:linearublemma}, we apply it to complete the proof. Take $s \geq s(\varepsilon, \delta)$ where $s(\varepsilon, \delta) \in \Theta\left((a\log n+\log(1/\varepsilon) + \log(1/\delta)) \cdot 1/\varepsilon^2\right)$. Then:
\begin{align*}
    \mathcal{L}^{\ell_1}(\ermfakealts{\ell_1};  \mathcal{D}_{\dprobs}) &\leq \mathcal{L}^{\ell_1}(\ermfakealts{\ell_1};  \mathcal{D}_{\tilde{\dprobs}}) + \gamma |FV| &&\text{[\cref{eq:linearublemma}, $\dprobs \rightarrow \tilde{\dprobs}$]}\\
    &\leq \mathcal{L}^{\ell_1}(\optfakealts{\ell_1};  \mathcal{D}_{\tilde{\dprobs}}) + \varepsilon + \gamma |FV| &&\text{[w.p. $1-\delta$ via \Cref{cor:sample-complexity-linear}]}\\
    &\leq \mathcal{L}^{\ell_1}(\opttruealts{\ell_1};  \mathcal{D}_{\tilde{\dprobs}}) + \varepsilon + \gamma |FV| &&\text{[optimality of $\optfakealts{\ell_1}$ for $\mathcal{D}_{\tilde{\dprobs}}$]}\\
    &\leq \mathcal{L}^{\ell_1}(\opttruealts{\ell_1};  \mathcal{D}_{\dprobs}) + \varepsilon + 2\gamma |FV| &&\text{[\cref{eq:linearublemma}, $\tilde{\dprobs} \rightarrow \dprobs$]}\\
    &\leq \mathcal{L}^{\ell_1}(\ermtruealts{\ell_1};  \mathcal{D}_{\dprobs}) + \varepsilon + 2\gamma |FV| &&\text{[optimality of $\opttruealts{\ell_1}$ for $\mathcal{D}_{\dprobs}$]}
\end{align*}
Rearranging gives us \[Pr[\mathcal{L}^{\ell_1}(\ermfakealts{\ell_1};  \mathcal{D}_{\dprobs}) - \mathcal{L}^{\ell_1}(\ermtruealts{\ell_1};  \mathcal{D}_{\dprobs}) \leq \varepsilon + 2\gamma |FV|] \geq 1-\delta. \qedhere\]
\end{proof}

\subsection{Proof of \Cref{thm:lb-robust-linear}} \label{app:lb-robust-linear}

\begin{proof}
We employ a similar proof strategy as in \Cref{thm:lb-robust-binary}: we construct an instance, $\mathcal{I}$, and a pair of dropout probability vectors, $\dprobs$ and $\tilde{\dprobs}$, such that the \emph{optimal} alternate set for the misestimated dropout probability vector performs substantially worse than the \emph{optimal} alternate set for the true dropout probability vector when evaluated under the true dropout probability vector. We then apply \Cref{cor:sample-complexity-linear} to select an $s(\delta,\gamma)$ such that \algoname$^{\ell_1}$ respectively given $\tilde{\dprobs}$ and $\dprobs$ outputs both the corresponding optimal sets with probability $1-2\delta$, thus allowing us to conclude the result.

We construct $\mathcal{I}$ as follows: without loss of generality, assume that all feature-values are numbers, so for all $f \in F$, $V_f = \{1, 2, \dots, |V_f|\}$. Set our panel size, $k = 2(\max_{f \in F} |V_f| - 1)$. Our panel is comprised of two ``subgroups'', each of size $\max_{f \in F} |V_f| - 1$. Their features are as follows: 
\begin{table}[h]
    \centering
    \begin{tabular}{l | c| c}
        \toprule
        & \textbf{Group 1} & \textbf{Group 2} \\
        \midrule
        Index Range & $i \in \left\{1, \dots, \max_{f \in F} |V_f| - 1\right\}$ & $i \in \left\{\max_{f \in F} |V_f|, \dots, |K|\right\}$ \\
        $\forall f \in F$, & 
        $f(i) = \begin{cases}
            i, \quad \text{if } i < |V_f| \\
            |V_f|, \quad \text{otherwise}
        \end{cases}$ & 
        $f(i) =|V_f|$ \\
        $\rho_i \ \forall i \in K$ & $\gamma$ & $0$ \\
        $\tilde{\rho}_i \ \forall i \in K$ & $0$ & $\gamma$ \\
        \bottomrule
    \end{tabular}
\end{table}

Essentially, all values except for the last value for each feature are only represented by one person on the panel. The last value for each feature, $f$,  is a ``filler'' value that everyone else has, preserving the rarity of the first $|V_f| - 1$ values. We have tight quotas, so for all $f \in F$ and all $v \in V_f$ such that $v \neq |V_f|$, $l_{f,v} = u_{f,v} = 1$. For all $f \in F$, $l_{f, |V_f|} = u_{f,|V_f|} = k - (|V_f|-1)$. Our pool is a copy of the panel, so for every $i \in K$ there is an $i' \in N$ with matching feature-values. Our alternate budget, $a$, is $k/2$. 

First, we lower bound $\mathcal{L}^{\ell_1}(\optfakealts{\ell_1}; \mathcal{D}_{\dprobs}) - \mathcal{L}^{\ell_1}(\opttruealts{\ell_1}; \mathcal{D}_{\dprobs})$. Observe that $\opttruealts{\ell_1}$ is the set of all Group 1 clones in the pool, and $\optfakealts{\ell_1}$ is the set of all Group 2 clones in the pool. This is because, similar to in the proof of \Cref{thm:lb-robust-binary}, under both $\dprobs$ and $\tilde{\dprobs}$, there are exactly $k/2$ agents with a non-zero probability of dropping out, and these sets contain clones of exactly those agents and no others. As a result, note that $\mathcal{L}^{\ell_1}(\opttruealts{\ell_1}; \mathcal{D}_{\dprobs}) = 0$ and $\mathcal{L}^{\ell_1}(\optfakealts{}{\ell_1}; \mathcal{D}_{\tilde{\dprobs}}) = 0$, so:
\begin{align*}
    \mathcal{L}^{\ell_1}(\optfakealts{\ell_1}; \mathcal{D}_{\dprobs}) - \mathcal{L}^{\ell_1}(\opttruealts{\ell_1}; \mathcal{D}_{\dprobs}) = \mathcal{L}^{\ell_1}(\optfakealts{\ell_1}; \mathcal{D}_{\dprobs})
    = \mathbb{E}_{D \sim \mathcal{D}_{\dprobs}}\left[\min_{R \subseteq \optfakealts{\ell_1}} dev^{\ell_1}(K \setminus D \cup R)\right]
\end{align*}
For every dropout set $D$ in the support of $\mathcal{D}_{\dprobs}$, denote the linear deviation minimizing replacement subset of $A$ as $R(D)$. We first lower bound our deviation by only accounting for the error on only some of the feature values.\begin{align*}
    \mathbb{E}_{D \sim \mathcal{D}_{\dprobs}}\left[dev^{\ell_1}(K \setminus D \cup R(D))\right] &\geq \mathbb{E}_{D \sim \mathcal{D}_{\dprobs}}\left[\sum_{\substack{(f,v) \in FV \\ v < |V_f|}} dev^{\ell_1}_{f,v}(K \setminus D \cup R(D))\right]
    \intertext{Noting that all of these feature-values have a quota of 1 and applying linearity, we have:}
    &\geq \mathbb{E}_{D \sim \mathcal{D}_{\dprobs}} \left[\sum_{\substack{(f,v) \in FV \\ v < |V_f|}} \frac{1}{1} \cdot \left(1 - \sum_{i \in K \setminus D \cup R(D)} \mathbb{I}(f(i) = v)\right)\right] \\
    &= \sum_{\substack{(f,v) \in FV \\ v < |V_f|}} \left(1 - \mathbb{E}_{D \sim \mathcal{D}_{\dprobs}} \left[\sum_{i \in K \setminus D \cup R(D)} \mathbb{I}(f(i)=v) \right]\right)\\
    \intertext{Finally, recall that for any $(f,v) \in FV$ such that $v < |V_f|$, only one agent on the panel (agent $v$) has that value. Additionally, as $\optfakealts{\ell_1}$ is the set of all Group 2 clones, it has \emph{no} agents with any of these values, and the same must be true for any replacement set taken from it. Therefore, $\sum_{i \in K \setminus D \cup R(D)} \mathbb{I}(f(i)=v)$ is 1 if agent $v$ does not drop, and 0 otherwise. }
    &= \sum_{\substack{(f,v) \in FV \\ v < |V_f|}} \left(1 - \Pr[\text{agent } v \text{ does not drop}]\right) \\
    &=  \sum_{\substack{(f,v) \in FV \\ v < |V_f|}} (1 - (1-\gamma)) = \gamma (|FV| - |F|).
\end{align*}

Hence, we have our desired bound on the difference between $\mathcal{L}^{\ell_1}(\optfakealts{\ell_1}; \mathcal{D}_{\dprobs})$ and $\mathcal{L}^{\ell_1}(\opttruealts{\ell_1}; \mathcal{D}_{\dprobs})$. Now, we show that there exists $s(\delta, \gamma) = \Theta\left(\frac{k/2 \cdot \log k + \log (k/\gamma^k) + \log(1/\delta)}{\gamma^{2k}/k^2}\right)$, such that for all $s \geq s(\delta, \gamma)$\[\Pr[\ermtruealts{\ell_1} = \opttruealts{\ell_1} \land \ermfakealts{\ell_1} = \optfakealts{\ell_1}]\geq 1-2\delta,\] 
First we consider the selection of $\ermfakealts{\ell_1}$. Fix any $A \subseteq N$ that is not $\optfakealts{\ell_1}$ \emdash so it is missing at least one Group 2 clone. In the event that our dropout set is all panelists in Group 2, $A$ will incur some non-zero linear loss. In particular, for any feature $f$ with maximal $|V_f|$, the dropout set will have $k/2$ agents with value $|V_f|$. However, \emph{no} Group 1 member, and thus no Group 1 clone in the pool, has value $|V_f|$ for that feature. As $A$ has strictly fewer than $a = k/2$ Group 2 clones, any replacement subset taken from it will fall short of the quota for $f, |V_f|$ by at least one person, incurring a loss of at least $\frac{1}{u_{f, |V_f|}} = \frac{1}{k/2}$. Recall that $\mathcal{L}^{\ell_1}(\optfakealts{\ell_1}; \mathcal{D}_{\tilde{\dprobs}}) = 0$, so we have that:
\begin{align*}
    |\mathcal{L}^{\ell_1}(\optfakealts{\ell_1}; \mathcal{D}_{\tilde{\dprobs}}) - \mathcal{L}^{\ell_1}(A; \mathcal{D}_{\tilde{\dprobs}})|
    &\geq \min_{R \subseteq A} dev^{\ell_1}(K \setminus \{\text{Group 2}\} \cup R) \cdot \Pr_{D \sim \mathcal{D}_{\tilde{\dprobs}}}[D = \{\text{Group 2}\} ]\\
    &\geq \frac{1}{k/2} \cdot \gamma^{k/2} > \frac{\gamma^k}{k}
\end{align*}
By our choice of $s$ and \Cref{cor:sample-complexity-linear}, we know that \[
\Pr[|\mathcal{L}^{\ell_1}(\ermfakealts{\ell_1}; \mathcal{D}_{\tilde{\dprobs}}) - \mathcal{L}^{\ell_1}(\optfakealts{\ell_1}; \mathcal{D}_{\tilde{\dprobs}})| \leq \gamma^{k}/k] \geq 1- \delta
\]
 Hence, this implies that $\ermfakealts{\ell_1} = \optfakealts{\ell_1}$ with probability at least $1- \delta$. 

Next, consider the selection of $\ermtruealts{\ell_1}$. Fix any $A \subseteq N$ that is not equal to $\opttruealts{\ell_1}$ . If our dropout set is all panelists in Group 1, then there is at least one person, $i$, in the dropout set whose clone, $i'$, is not in $A$. This person has at least one feature, $f$, for which $f(i) < |V_f|$ (as is the case for all panelists in Group 1 by its definition). Moreover, $i$ and $i'$ are the only agents in the instance with that feature-value. Therefore, $A$ must incur a loss of 1 from that feature-value, as the quota is exactly 1. Recall that $\mathcal{L}^{\ell_1}(\opttruealts{\ell_1}; \mathcal{D}_{\dprobs}) = 0$, so we have that:\begin{align*}
    |\mathcal{L}^{\ell_1}(\opttruealts{\ell_1}; \mathcal{D}_{\dprobs}) - \mathcal{L}^{\ell_1}(A; \mathcal{D}_{\dprobs})| 
    &\geq \min_{R \subseteq A} dev^{\ell_1}(K \setminus \{\text{Group 1}\} \cup R) \cdot \Pr_{D \sim \mathcal{D}_{\dprobs}}[D = \{\text{Group 1}\} ]\\
    &\geq 1 \cdot \gamma^{k/2} > \frac{\gamma^k}{k}
\end{align*}
As before, by our choice of $s$ and \Cref{cor:sample-complexity-linear}, we know that \[
\Pr[|\mathcal{L}^{\ell_1}(\ermtruealts{\ell_1}; \mathcal{D}_{\dprobs}) - \mathcal{L}^{\ell_1}(\opttruealts{\ell_1}; \mathcal{D}_{\tilde{\dprobs}})| \leq \gamma^{k}/k] \geq 1- \delta
\]
 Hence, this implies that $\ermtruealts{\ell_1} = \opttruealts{\ell_1}$ with probability at least $1- \delta$. By union bounding, we conclude that indeed $\Pr[\ermtruealts{\ell_1} = \opttruealts{\ell_1} \land \ermfakealts{\ell_1} = \optfakealts{\ell_1}] \geq 1 - 2\delta$. Hence we can combine this with our lower bound on $\mathcal{L}^{\ell_1}(\optfakealts{\ell_1}; \mathcal{D}_{\dprobs}) - \mathcal{L}^{\ell_1}(\opttruealts{\ell_1}; \mathcal{D}_{\dprobs})$ to conclude that \[
\Pr[\mathcal{L}^{\ell_1}(\ermfakealts{\ell_1}; \mathcal{D}_{\dprobs}) - \mathcal{L}^{\ell_1}(\ermtruealts{\ell_1}; \mathcal{D}_{\dprobs}) \geq \gamma(|FV| - |F|)] \geq 1- 2\delta.
\]
\end{proof}

\newpage
\section{Supplemental Materials from \Cref{sec:empirics}}


\subsection{Data cleaning details} \label{app:datacleaning}

\subsubsection{Data cleaning details for datasets \textit{US-1}, \textit{US-2}, and \textit{US-3}.}
After trimming features that were not common across these datasets, these datasets contain the following features with the following values (up to renaming for consistency across datasets).

\begin{table}[h!]
    \centering
    \begin{tabular}{l|l}
        $f$ & $V_f$  \\
        \hline
        \textit{gender} & \textit{male}\\
        &\textit{female}\\
        &\textit{non-binary/other}\\
        \hline
        \textit{age} & \textit{16-24}\\
        & \textit{25-34}\\
        &\textit{35-44}\\
        & \textit{45-54}\\
        &\textit{55-64}\\
        &\textit{65-74}\\
        &\textit{75+}\\
        \hline
        \textit{housing} & \textit{own}\\
        &\textit{rent}\\
        &\textit{subsidized housing/unhoused}\\
        \hline
        \textit{educational attainment} & \textit{bachelor's or higher}\\
        & \textit{some college}\\
        & \textit{high school}\\
        & \textit{some schooling}\\
        \hline
        \textit{race/ethnicity} & \textit{native american}\\
        &\textit{white}\\
        & \textit{AAPI}\\
        & \textit{black}\\
        & \textit{latinx}\\
        & \textit{multiracial}
    \end{tabular}
    \caption{Feature-values common across \textit{US} instances.}
    \label{tab:FV-US}
\end{table}
\noindent \textbf{Handling feature-values for \Cref{fig:dropout-rates} / learning the dropout probabilities.} The features above were all directly available in the raw data. The only modification to the data we made for these purposes was to merge some of the $f = age$ categories. For \Cref{fig:dropout-rates}, this was for consistency across \textit{US} and \textit{Canada} datasets; for training, this was to ensure sufficient sample size and try to capture practically meaningful life stages. These merges combined \textit{16-24} and \textit{25-34} into ``Young Adult'', \textit{35-44}, \textit{45-54}, and \textit{55-64} into ``Middle Age'', and \textit{65-74} and \textit{75+} into ``Retirement Age''.

\vspace{0.5em}
\noindent \textbf{Handling feature-values for running/evaluating the algorithms.} The quotas we imposed exactly reflected the quotas used in practice in each instance. In particular, these quotas were imposed on the raw-feature values above, and sometimes other feature-values that appeared in only a subset of these datasets (which were omitted above because they were not common across datasets). 

\vspace{0.5em}
\noindent \textbf{Handling the ``dropped'' variable.} In this dataset, the main (but still minor) complication was that in \textit{US-2} and \textit{US-3}, there was fairly substantial ``dropout'' due to people not responding or declining immediately when notified they were selected for the panel. There were concerns that this may have been due to logistical snags in the recruitment process in these particular cases, meaning that someone ``dropping out'' at this stage might not be reflective of general ``propensity to drop out''. Furthermore, there were people recruited to replace these people immediately after the panel was selected (called ``panel reselection''), and those who were selected in that process were essentially panelists, meaning they had the opportunity to display dropout/nondropout behavior. We therefore constructed our ``dropped'' variable in the following way: we ignored people who were selected for the panel but then did not accept the initial notification (just assigning them \textit{n/a} for the binary outcome variable describing dropout), and conversely, we considered those who were selected during panel reselection to be within our dropped/did not drop dataset, even though they were not technically selected for the original panel. Note that this construction was used only for estimating the dropout probabilities, as this is the only purpose of the ``dropped'' variable.

\vspace{0.5em}
\noindent \textbf{Handling missing data.} There was no missing data.

\subsubsection{Data cleaning details for datasets \textit{Can-1} - \textit{Can-4}}
These datasets required considerably more unification across the values of each feature. After trimming features that were not common across these datasets and unifying variables, the \textit{Can} datasets contain the following common feature-values.

\begin{table}[h!]
    \centering
    \begin{tabular}{l|l}
        $f$ & $V_f$  \\
        \hline
        \textit{gender} & \textit{male}\\
        &\textit{female}\\
        &\textit{non-binary/other}\\
        \hline
        \textit{age} & \textit{1} (roughly up to 29)\\
        & \textit{2} (roughly 30 - 44)\\
        &\textit{3} (roughly 45 - 64)\\
        &\textit{4} (roughly 65 - 70)\\
        &\textit{5} (roughly 71+)\\
        \hline
        \textit{housing} & \textit{own}\\
        &\textit{rent}\\
        &\textit{subsidized housing/unhoused}\\
        \hline
        \textit{indigenous}& \textit{yes}\\
        &\textit{no}\\
        \hline
        \textit{racial minority} & \textit{yes}\\
        &\textit{no}\\
        \hline
        \textit{income} & \textit{higher income}\\
        & \textit{lower income}
    \end{tabular}
    \caption{Feature-values common across \textit{Canada} instances.}
    \label{tab:FV-Can}
\end{table}
Of these variables, three were constructed or modified in some way. First, \textit{age} was available in raw form, but the age ranges used across datasets varied substantially. We unified them to the best of our ability, trying to combine similar ranges. The age ranges given in \Cref{tab:FV-Can} are qualified with ``roughly'' because, while a majority of datasets' ranges respected those guidelines, sometimes there was some leakage between categories. Second, \textit{racial minority} is a constructed variable, which was set to \textit{yes} if people answered ``yes'' to any number of other raw features: whether they were a \textit{visible minority}, \textit{black}, \textit{indigenous}, \textit{POC}, or \textit{racialized}. Because different datasets contained different subsets of these raw indicators, we combined them into one global indicator to make features as common as possible across datasets. Finally, \textit{income} is a constructed variable, which was created by (1) unifying several differently-coded variables asking about income, and (2) inferring income level by how many times per year a person took a flight, where greater numbers of flights were assumed to indicate higher income. This variable was ultimately recoded into just two categories, to avoid over-indexing on the assumptions underlying the construction of this variable.

\vspace{0.5em}
\noindent \textbf{Handling feature-values for \Cref{fig:dropout-rates} / learning the dropout probabilities.} The only further modification to the data we made for these purposes was to again merge some of the $f = age$ categories. For \Cref{fig:dropout-rates}, this was for consistency across \textit{US} and \textit{Canada} datasets; for training, this was to ensure sufficient sample size and try to capture practically meaningful life stages. These merges defined \textit{1} as ``Young Adult'', \textit{2} and \textit{3} as ``Middle Age'', and \textit{4} and \textit{5} as ``Retirement Age''

\vspace{0.5em}
\noindent \textbf{Handling feature-values for running/evaluating the algorithms.} In these datasets, we did not have access to the raw quotas used in practice. We therefore simply imposed tight quotas on the feature-values in \Cref{tab:FV-Can}, enforcing that if the original panel contained $k_{f,v}$ agents with $f,v$, then we set the corresponding quotas such that $l_{f,v} = u_{f,v} = k_{f,v}$. We remark that although we made these quotas perfectly tight, the quota \textit{structure} is much less rich in this dataset, owing to the fact that there are far fewer feature-values. This is one likely explanation for why the trends are less pronounced for our results in this dataset. For completeness, we also include versions of the plots where we relax all quotas by 1, i.e. $l_{f,v} = \max(k_{f,v}-1, 0)$ and $u_{f,v} = \min(k_{f,v}+1, k)$. However, the trends are severely attenuated in these plots because the instances are unrealistically ``easy.''

\vspace{0.5em}
\noindent \textbf{Handling the ``dropped'' variable.} This variable was taken off-the-shelf, and required no cleaning.

\vspace{0.5em}
\noindent \textbf{Handling missing data.} There were a few limited cases where pool members were missing a feature-value. In these cases, we simply dropped that person from the data. This occurred in two datasets here, \textit{Can-2} and \textit{Can-3}. In \textit{Can-2}, we dropped 16 pool members out of 144 total. Similarly, in \textit{Can-3}, we dropped 2 pool members out of 212 total.

\subsection{Convergence plots across sample sizes}\label{app:convergence}
\Cref{fig:convergence} shows the loss of all three ERM algorithms at increasing numbers of samples $s$. Here, $s$ refers to the number of samples used to \textit{train} (i.e., the $s$ provided to \algoname$^{dev}$, determining how many samples to draw to create $\hat{\ddist}_{\dprobs}$); to avoid too much noise generated by \textit{evaluation} (rather than non-convergence), we evaluated the loss on 500 samples for all values of $s$. These tests were performed in instance \textit{US-3}. Based in these plots, we selected a sample size of $s=300$. 
\begin{figure}[h!]
    \centering    \includegraphics[width=\linewidth]{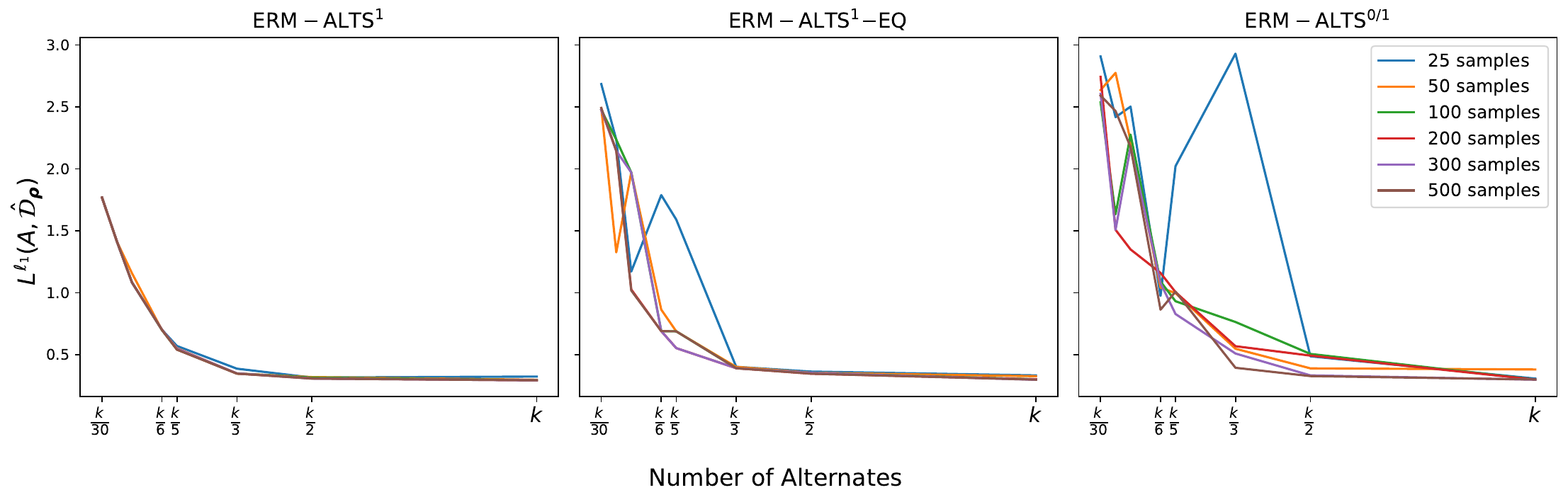}
    \caption{Plots showing losses of ERM algorithms over varying numbers of training samples (with losses always \textit{evaluated} using 500 samples).}
    \label{fig:convergence}
\end{figure}

\subsection{Benchmark algorithms} \label{app:benchmarks}

The only algorithms that remain to be defined are \textsc{Greedy} and \textsc{Quota-Based}. For each, we provide a formal definition, some intuition for what sets of people they tend to select as alternates, and what types of instances cause these algorithms to perform poorly.

First, we define \textsc{Greedy}. Recall that $w(i) = (f(i) | f \in F)$ is agent $i$'s feature vector. We define the hamming distance between two feature vectors $w,w'$ (indexed by $f$) as the total number of features on which they differ in value:
\[\textit{hamming-dist}(w,w'):= \sum_{f \in F} \mathbb{I}(w_f \neq w'_f).\]

\begin{algorithm}
\caption{$\textsc{Greedy}(\dprobs, \inst)$}
\DontPrintSemicolon
     $A \leftarrow \emptyset$\;
     $\{1,2,\dots,k\} \leftarrow \text{Number the panelists in decreasing order of } \dprob_i, \text{ such that } i < j \Rightarrow \dprob_i \geq \dprob_j$.\;
    \For{$i \in \{1,2,\dots,a\}$}{
         $i^* \leftarrow \arg\min_{j \in N \setminus A} \textit{hamming-dist}(w(i), w(j))$\;
         $A \leftarrow A \cup \{i^*\}$}
    \textbf{return} $A$
\end{algorithm}

As alluded to in \Cref{sec:empirics}, \textsc{Greedy} ignores the combinatorial structure of the problem. This is true even if we improve \textsc{Greedy} beyond realistic manual capabilities and change the distance metric for deciding which pool member is ``closest'' to a given panel member from hamming distance to $\ell_1$. While \textsc{Greedy} may typically perform well enough, it can do poorly in instances where incrementally patching up the gaps left by the probable dropout sets ends up doing worse than taking a global perspective. 

A simple example is as follows: we consider a setting with two binary features and tight panel quotas necessitating $k/2$ panelists with each feature-value in which $\dprob_1=\dprob_2 = 1$ and $\dprob_3 = \dots = \dprob_k = 0$. That is, the first two panelists deterministically drop out, and the remaining panelists deterministically stay. The two panelists have feature vectors $w(1) = 10$ and $w(2) = 01$ and the pool is made up of solely of agents with feature vectors $11$ and $00$. Now observe that under the $\ell_1$ and hamming distance metric, both types of pool members are equidistant from panelists 1 and 2. Therefore, \textsc{Greedy} may select any two pool members as alternates. It doesn't discriminate between $\{11, 11\}, \{00, 00\}, \text{ or } \{11, 00\}$ because it myopically focuses on picking a good replacement for one person at a time. However, when one considers the problem of filling in for \textit{both} of the delinquent panelists simultaneously, the last option is best.

To define \textsc{Quota-Based}, we first define how the quotas are constructed, and then define the ILP used to find a quota-satisfying panel. Fix the original quotas (designed for the size $k$ panel) $\boldsymbol{l},\boldsymbol{u}$. Now, define our \textit{scaled} quotas as
\[l_{f,v}':=  \lfloor l_{f,v} \cdot a/k \rfloor, \qquad u_{f,v}':=  \lceil u_{f,v} \cdot a/k \rceil \qquad \qquad \text{for all }f,v \in FV.\]
(where if these quotas are infeasible, we loosen all quotas by 1). Then, \textsc{Quota-Based} amounts to solving the following ILP; note that there may be many alternate sets that satisfy these quotas, and we leave this choice to be arbitrarily made by the ILP solver (in practice, a practitioner would use one of the existing panel selection algorithms \cite{baharav2024fair}, whose default functionality makes similarly arbitrary choices when it comes to the precise quota-compliant panel it outputs).

\vspace{0.5em}
\noindent \textsc{Quota-Based}$(\inst)$:
\begin{align*}
    \max_{x_i | i \in N} \quad &1\\
    \text{s.t.} \quad & \sum_{i \in N} x_i \cdot \mathbb{I}(f(i) = v) \in [l_{f,v}',u_{f,v}'] \qquad \text{for all }f,v \in FV\\
    &\sum_{i \in N} x_i = a\\
    &x_i \in \{0,1\} \qquad \text{for all }i \in N
\end{align*}

In contrast to \textsc{Greedy}, \textsc{Quota-Based} utilizes a more global outlook on the dropout problem, but completely ignores the dropout probabilities. By selecting for a scaled-down replica of the panel, it chooses alternates for each feature-value approximately proportionately to their quota. If people drop out at rates inversely correlated with their feature-value representation on the panel (as seems quite plausible with many feature-values in \Cref{fig:dropout-rates}), \textsc{Quota-Based} may completely miss the mark on which feature-values to prioritize in the alternate pool. 

For example, consider a setting with just one feature. There are $m < k$ special panelists with unique feature-values ranging from 1 to $m$, and the remaining $k-m$ panelists all have feature-value 0. The quotas require exactly one of each of the $m$ unique panelists, and anywhere from $k-m$ to $k$ of the 0 panelists. Now assume that all of the panelists with a non-zero feature-value have a high probability of dropping out, and the panelists with feature-value 0 have a dropout probability of 0. Assume that the pool has at least $a$ of each type of agent, so our selection of the alternate set is unrestricted by pool composition. If $a < k$, then by the definition of \textsc{Quota-Based}, $l'_{1,v} = 0$ for all $v \neq 0$, and $u'_{1,0} = a$. Hence, \textsc{Quota-Based} may well return a panel with just agents with feature-value 0. This alternate set is useless for this instance, and a much better alternate set would be selecting $a$ unique agents with a non-zero feature-value.

\subsection{Calibration plots} \label{app:calibration}
Here, we attempt to evaluate our empirically-learned dropout probabilities without knowing the ground truth. In each instance $j$, we compare the \textit{expected} number of dropouts with $f,v$ according to the learned probabilities $\tilde{\dprobs}_j$, versus the actual number of dropouts from $f,v$ in that instance (examining the \textit{realized} dropout set). Here, $\tilde{\dprobs}_j$ is exactly as it was defined in the description of \Cref{fig:real-dropouts} (so we train $\tilde{\dprobs}_j$ on all datasets \textit{other} than \textit{US-$j$}). We will consider our predictions well-calibrated for the intended dataset if all points on the plot hew to the line $y = x$, describing the case where the expectation of dropouts from $f,v$ perfectly matches the realized number of dropouts. We see that this is fairly true for \textit{US-2}, but far from true for \textit{US-1} and \textit{US-3}. We remark that \textit{US-3} was a slightly strange case, because there was just one dropout.
\begin{figure}[h!]
    \centering
    \includegraphics[width=\linewidth]{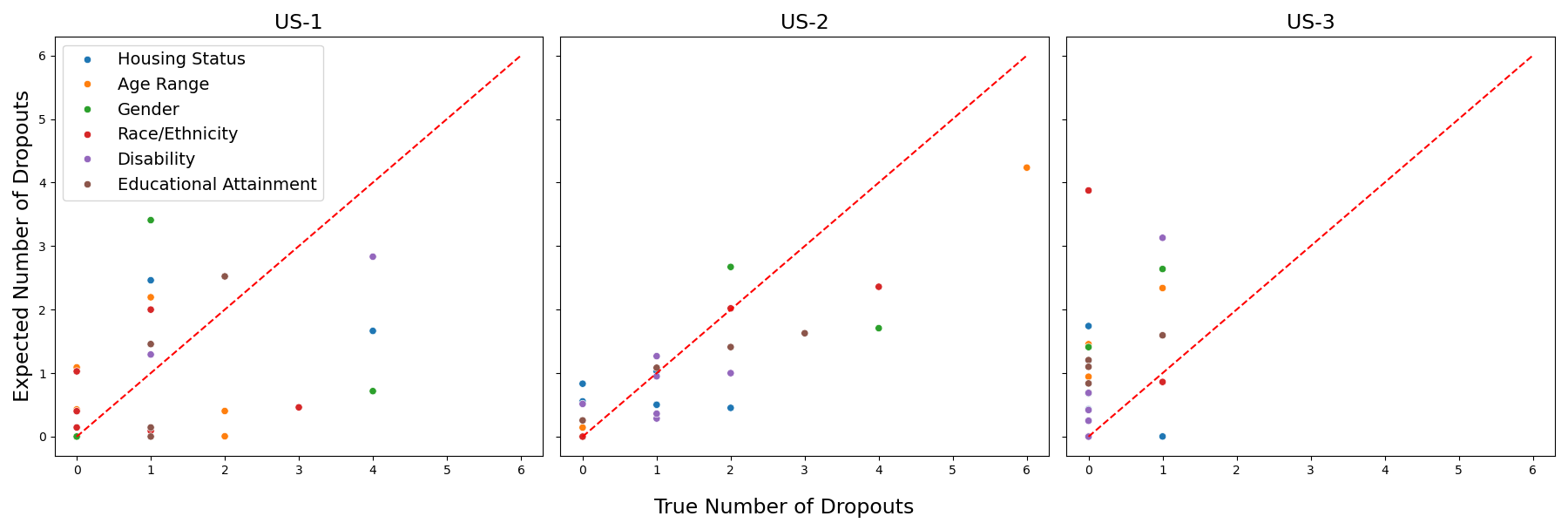}
    \caption{Points indicate number of dropouts observed in the data in each group $f,v$ (with colors corresponding to values of $f$). Red line is $y = x$.}
    \label{fig:enter-label}
\end{figure}

\subsection{Version of \Cref{fig:sim-L1} with standard deviations}\label{app:sim-L1-stdev}
In \Cref{fig:L1-sim-stdev} we show the version of \Cref{fig:sim-L1} with standard deviations, as described in the legend. All run parameters are identical to those used to create \Cref{fig:sim-L1}.

\begin{figure}[h!]
    \centering
    \includegraphics[width=0.95\linewidth]{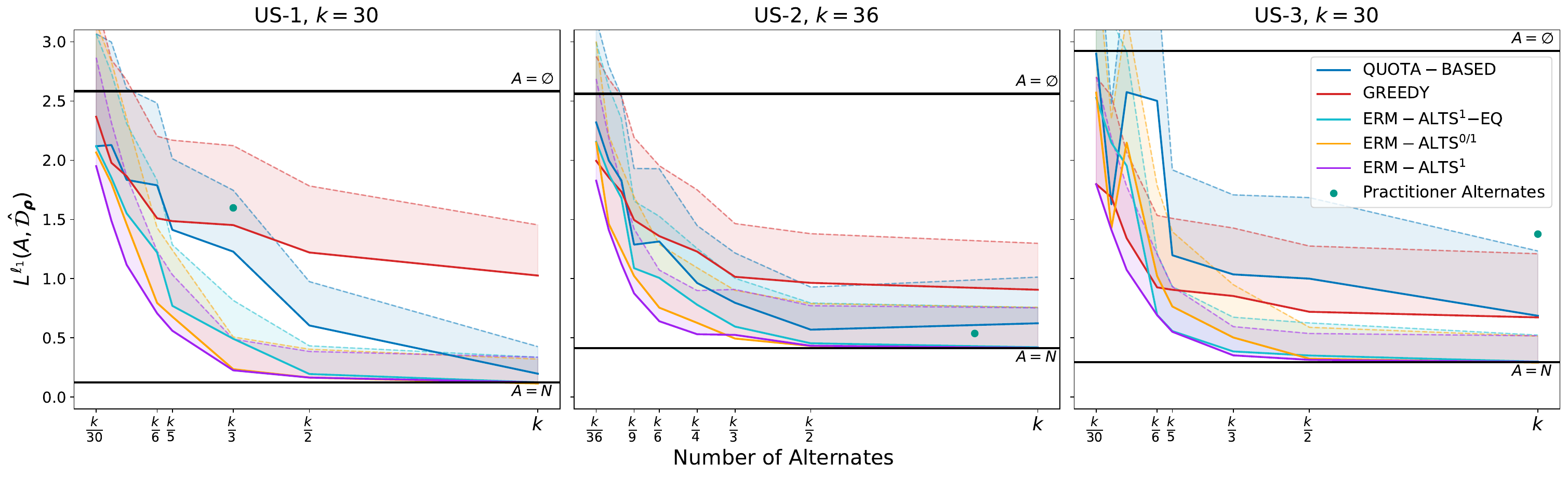}
    \caption{Analog of \Cref{fig:sim-L1} showing standard deviation of losses over all samples $D$ drawn during evaluation. Standard deviation shown only above the mean for clarity of plots, but is symmetric above and below the mean.}
    \label{fig:L1-sim-stdev}
\end{figure}

\subsection{Algorithm evaluation on other performance criteria}
\label{app:other-criteria}

In \Cref{fig:other-criteria} we replicate the analysis in \Cref{fig:sim-L1} on four other metrics of performance, which we define below. As in the body, we compute these expected values over an empirical version of $\mathcal{D}_{\dprobs}$, called $\hat{\mathcal{D}}_{\dprobs}$. For each draw of $D$, $S$ in the definitions below is equal to $K \setminus D \cup R$, where $R$ is chosen according to \Cref{eq:replacements} \textit{with the deviation function specified not as $dev^{\ell_1}$, but as the current performance benchmark}.

\vspace{1em}
\noindent \textbf{Normalized Deviation Below Quotas:}
\[\mathbb{E}_{D \sim \hat{\mathcal{D}}_{\dprobs}}\left[\sum_{f,v \in FV}\frac{\max\left\{0,\ l_{f,v} - \sum_{i \in S} \mathbb{I}(f(i) = v)\right\}}{u_{f,v}} \right].\]

\noindent \textbf{Max Normalized Quota Deviation:}
\[\mathbb{E}_{D \sim \hat{\mathcal{D}}_{\dprobs}}\left[\max_{f,v \in FV} \frac{\max\left\{0,\ l_{f,v} - \sum_{i \in S} \mathbb{I}(f(i) = v), \ -u_{f,v} + \sum_{i \in S} \mathbb{I}(f(i) = v)\right\}}{u_{f,v}} \right].\]

\noindent \textbf{Max Quota Deviation:}
\[\mathbb{E}_{D \sim \hat{\mathcal{D}}_{\dprobs}}\left[\max_{f,v \in FV} \max\left\{0,\ l_{f,v} - \sum_{i \in S} \mathbb{I}(f(i) = v), \ -u_{f,v} + \sum_{i \in S} \mathbb{I}(f(i) = v)\right\} \right]\]

\noindent \textbf{Number of Unrepresented Feature-Values:}
\[\mathbb{E}_{D \sim \hat{\mathcal{D}}_{\dprobs}}\left[\sum_{f,v \in FV} \mathbb{I}\left(\left( \sum_{i \in S}f(i) = v\right) = 0\right)\right]\]

\begin{figure}[h!]
    \centering
    \includegraphics[width=0.9\linewidth]{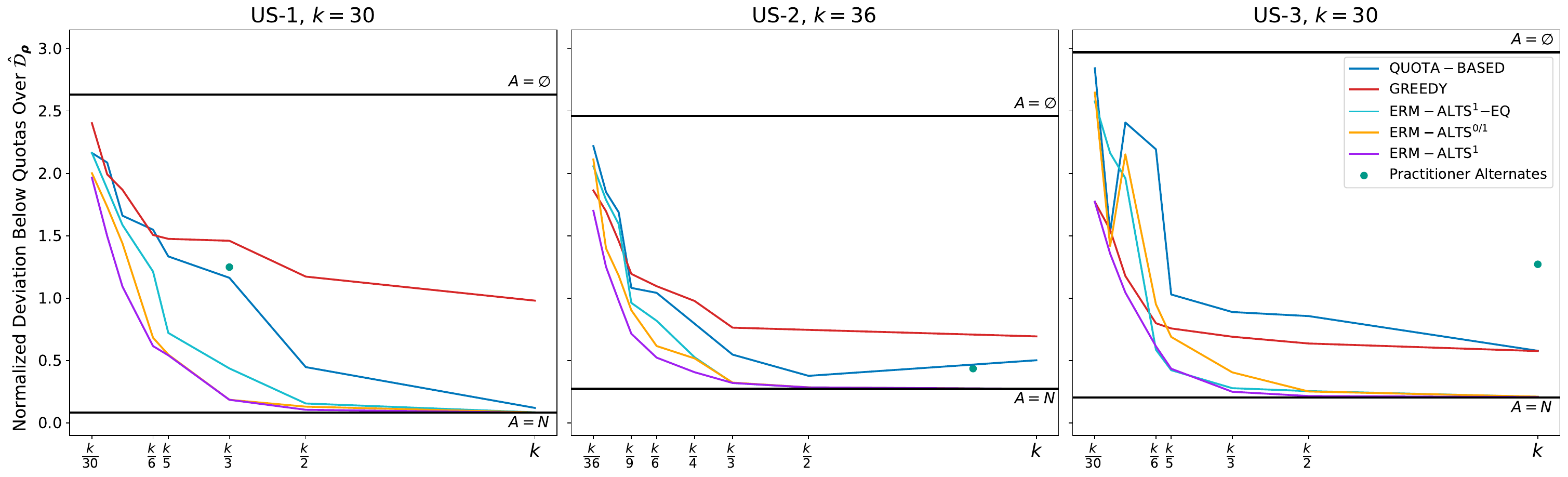}
    \includegraphics[width=0.9\linewidth]{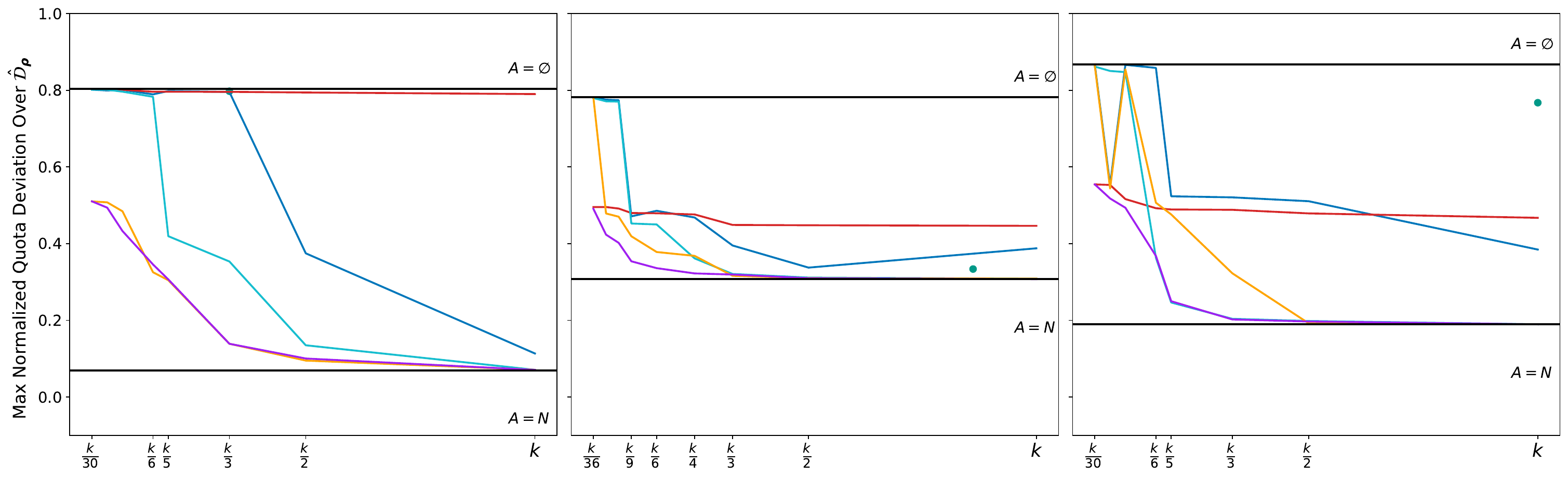}
    \includegraphics[width=0.9\linewidth]{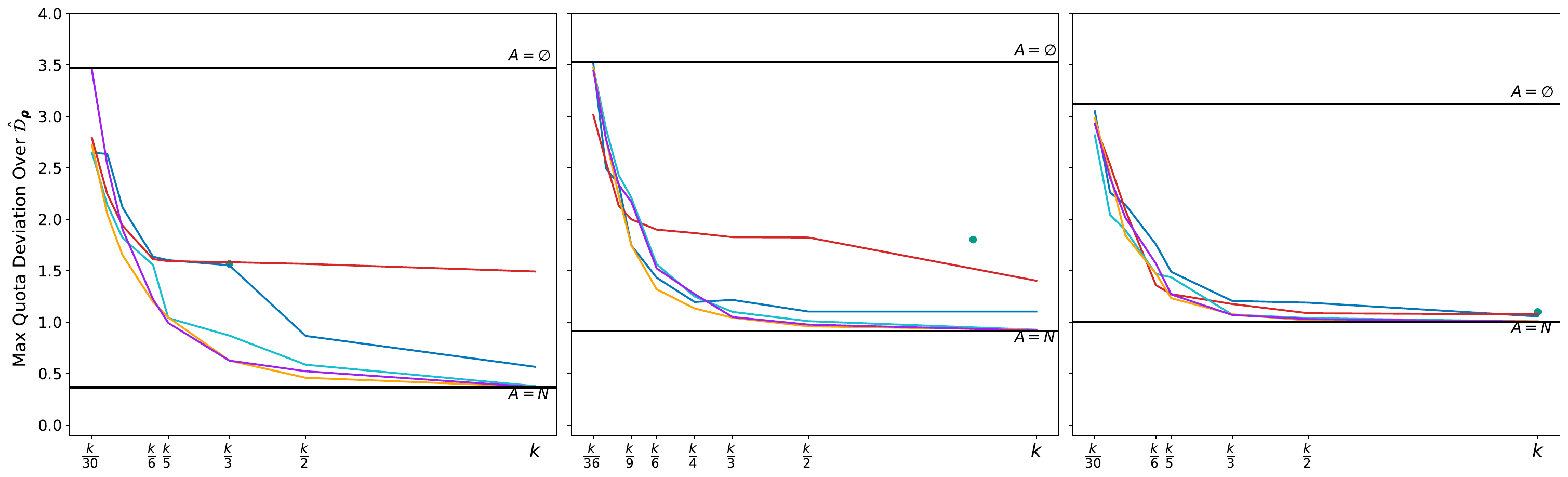}
    \includegraphics[width=0.9\linewidth]{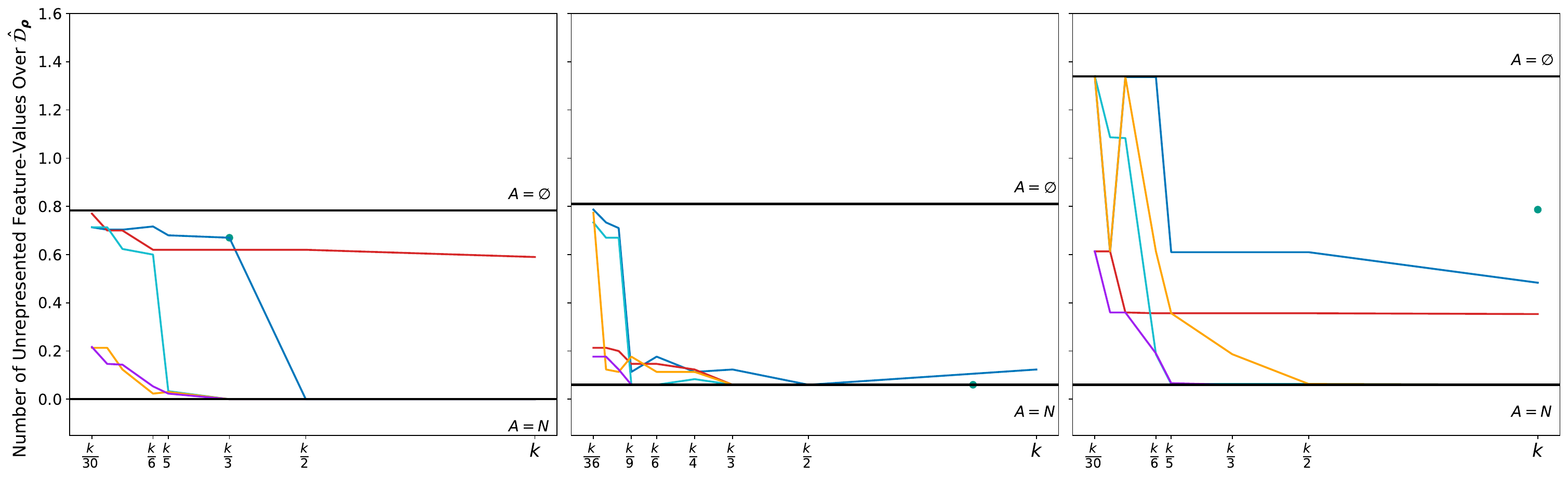}
    \includegraphics[width=0.9\linewidth]{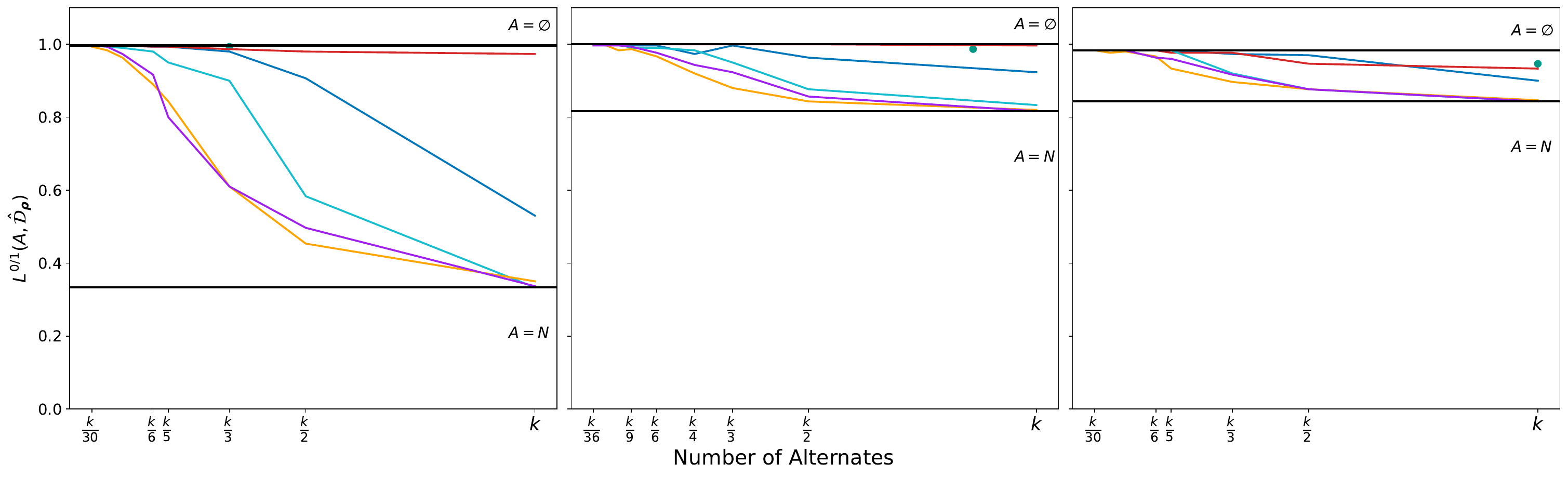}
    \caption{Analysis of expected algorithmic performance across instances \textit{US-1, US-2, US-3} on other performance metrics, as defined above. All run parameters are the same as in \Cref{fig:sim-L1}.}
    \label{fig:other-criteria}
\end{figure}

\subsection{Results on Canadian data cluster}
\label{app:canadian}

In \Cref{fig:canada} and \Cref{fig:canada2} we repeat our two main analyses from (respectively) \Cref{fig:real-dropouts} and \Cref{fig:sim-L1} in datasets \textit{Can-1} - \textit{Can-4}. We see similar relative performance of the algorithms, with the algorithms' performance on the realized dropout sets being somewhat unpredictable, and their relative performance in expectation being fairly consistent. Again, \textsc{Quota-Based} and \textsc{Greedy} tend to have the highest loss, though there are more exceptions in this dataset than the \textit{US} datasets.

\begin{figure}[h!]
    \centering
    \includegraphics[width=\linewidth]{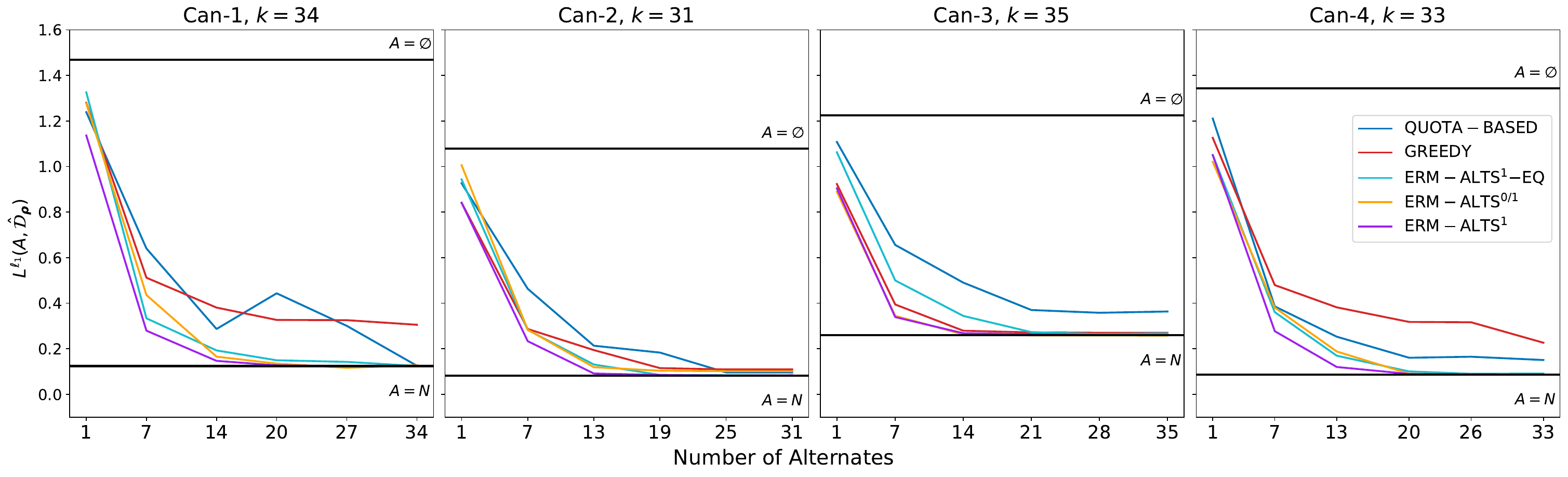}
    \caption{Analog of \Cref{fig:real-dropouts} in Datasets \textit{Can-1}-\textit{Can-4} with tight quotas.}
    \label{fig:canada}
\end{figure}

\begin{figure}[h!]
    \centering
    \includegraphics[width=\linewidth]{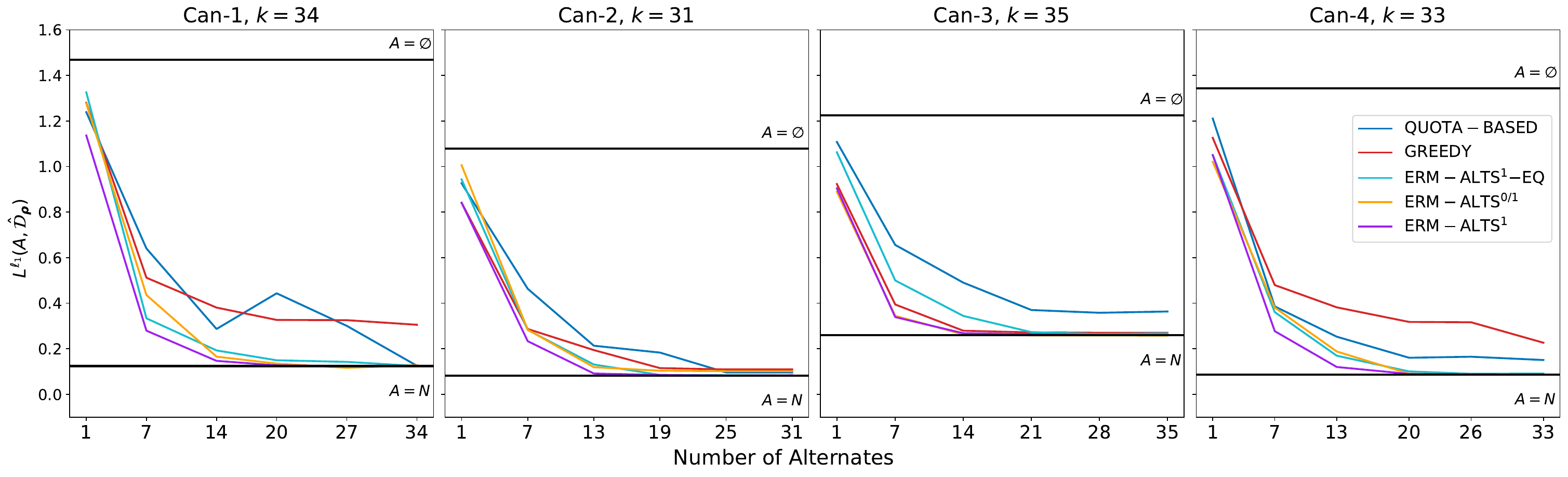}
    \caption{Analog of \Cref{fig:sim-L1} in Datasets \textit{Can-1}-\textit{Can-4} with tight quotas.}
    \label{fig:canada2}
\end{figure}

\begin{figure}[h!]
    \centering
    \includegraphics[width=\linewidth]{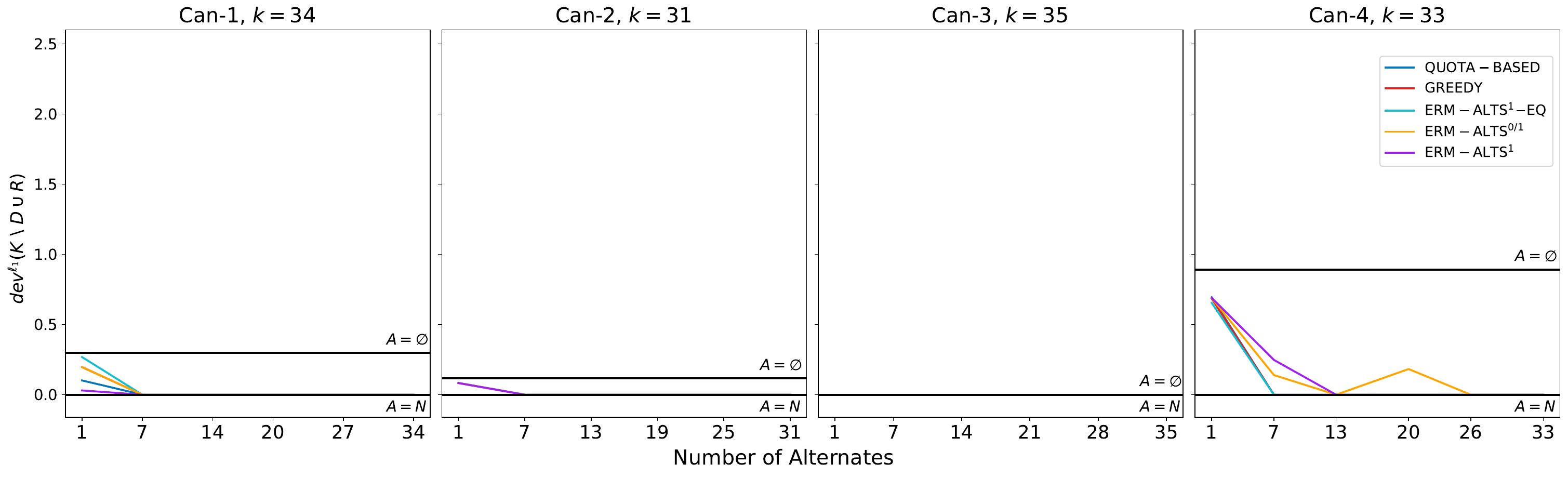}
    \caption{Analog of \Cref{fig:real-dropouts} in Datasets \textit{Can-1}-\textit{Can-4} with loose quotas.}
    \label{fig:canada3}
\end{figure}

\begin{figure}[h!]
    \centering
    \includegraphics[width=\linewidth]{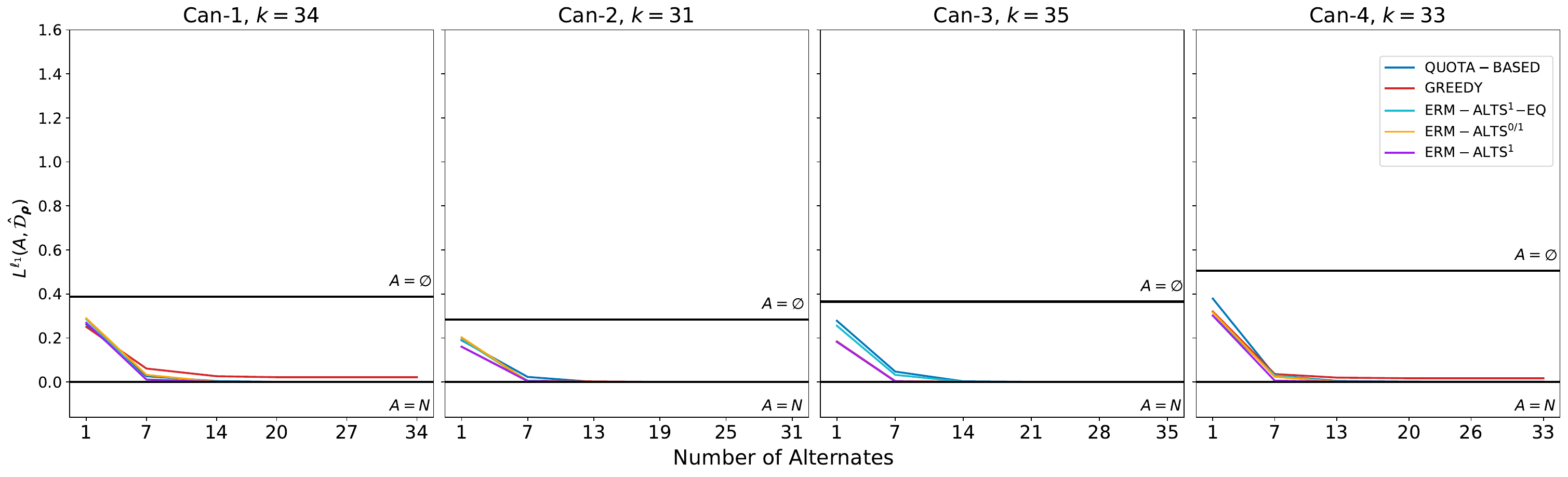}
    \caption{Analog of \Cref{fig:sim-L1} in Datasets \textit{Can-1}-\textit{Can-4} with loose quotas.}
    \label{fig:canada4}
\end{figure}

\newpage
\section{Supplemental Materials from \Cref{sec:discussion}} \label{app:extensions}
\subsection{Extension 1: Allowing \textit{alternates} to drop out} 
In the body we assumed that while panelists may drop out, the \textit{alternates} will always be available to replace them. A reasonable extension would be to assume that alternates will also drop out with some probability --- for simplicitly, we will suppose that they do so via the same dropout probabilities as do the panelists, but in principle our algorithm can handle separate distributions. In this case, one would choose the alternates via our algorithm as follows: first, sample dropouts from the panel $K$ \textit{and the remaining pool} $N$. Then, in every sample we get two dropout sets: the dropouts from the panel $D$ and the dropouts from the pool $\bar{D}$. Now, instead of evaluating $A$ on $D$, we evaluate $A \setminus \bar{D}$ on $D$. This weakly worsens the deviation for any given draw of dropouts, but it does not affect the sample complexity required for \textsc{ERM-Alts}, because the size of the hypothesis class has not changed, so running \textsc{ERM-Alts} with order $a \log n$ is still sufficient. The main change to the algorithm occurs in how we implement the ILP \textsc{Opt}, because we are now optimizing the expectation over dropouts from the panel \textit{and} from each candidate alternate set. This new version of \textsc{Opt} is specified below as \textsc{Opt-Alts-Drop}, which takes in a distribution over pool dropout sets ($\mathcal{\bar{D}}$) in addition to the instance and the distribution over panel dropout sets:\\

\noindent \textsc{Opt-Alts-Drop}$^{dev}(\inst,\mathcal{D}, \mathcal{\bar{D}})$ 
\begin{align*}
    \min \ & \sum_{(D,\bar{D}) \in 2^K \times 2^N} d_{(D,\bar{D})} \cdot \mathcal{D}(D) \cdot \mathcal{\bar{D}}(\bar{D}) \notag\\
    \text{s.t.} \ &\sum_{i\in N} x_i = a \notag\\
    &y_{i,(D,\bar{D})} \leq x_i &&\forall \ i \in N,\ (D, \bar{D}) \in 2^K \times 2^N \\
    &y_{i,(D,\bar{D})} \leq \mathbb{I}(i \in \bar{D})  \ \ \ \qquad(*\,\text{prevents using alternate dropouts}\,*) &&\forall \ i \in N,(D, \bar{D}) \in 2^K \times 2^N\\
    &\sum_{i \in N} y_{i, (D,\bar{D})} \leq |D| &&\forall \ (D, \bar{D}) \in 2^K \times 2^N\\
    &l_{f,v} - \left(\sum_{i \in K \setminus D} \mathbb{I}(f(i) = v) + \sum_{i \in N} y_{i,(D, \bar{D})}\, \mathbb{I}(f(i)=v) \right) \leq z_{(D, \bar{D}), f, v} \, u_{f,v} &&\forall \ (D, \bar{D}) \in 2^K \times 2^N, f,v \in FV \\
    & - u_{f,v} + \left(\sum_{i \in K \setminus D} \mathbb{I}(f(i) = v) + \sum_{i \in N} y_{i,(D, \bar{D})} \,\mathbb{I}(f(i)=v) \right)\leq z_{(D, \bar{D}), f, v} \, u_{f,v} &&\forall \ (D, \bar{D}) \in 2^K \times 2^N, f,v \in FV \\
    &\sum_{f,v} z_{(D, \bar{D}),f,v} \leq d_{(D, \bar{D})} \qquad \quad \ \qquad(*\,dev^{\ell_1}\,*) &&\forall \ (D, \bar{D}) \in 2^K \times 2^N \\
    &\sum_{f,v} z_{(D, \bar{D}),f,v} \leq d_{(D, \bar{D})} \cdot (|K| + a) |FV| \qquad(*\,dev^{0/1}\,*) &&\forall \ (D, \bar{D}) \in 2^K \times 2^N\\
    & d_{(D, \bar{D})} \in \mathbb{R}_{\geq 0}  \ \ \ \qquad(*\,dev^{\ell_1}\,*)&&\forall \ D \in (D, \bar{D}) \in 2^K \times 2^N\\
    & d_{(D, \bar{D})} \in \{0,1\} \qquad(*\,dev^{0/1}\,*) &&\forall \ (D, \bar{D}) \in 2^K \times 2^N\\
    & x_i \in \{0,1\}, \ y_{i,(D, \bar{D})} \in \{0,1\}, && \forall \  i \in N, (D, \bar{D}) \in 2^K \times 2^N\\
    &\  z_{(D, \bar{D}),f,v} \in \mathbb{R}_{\geq 0} && \forall \ (D, \bar{D}) \in 2^K \times 2^N, f,v \in FV 
\end{align*}

Given the practical importance of this extension, we replicate our empirical analysis from \Cref{fig:sim-L1} for this algorithm. The results, shown in \Cref{fig:alt-dropout}, show the same trends as previous results, with the three ERM algorithms consistently outperforming the two heuristic algorithms (\textsc{Quota-Based, Greedy}). 
\begin{figure}[h!]
    \centering
\includegraphics[width=\linewidth]{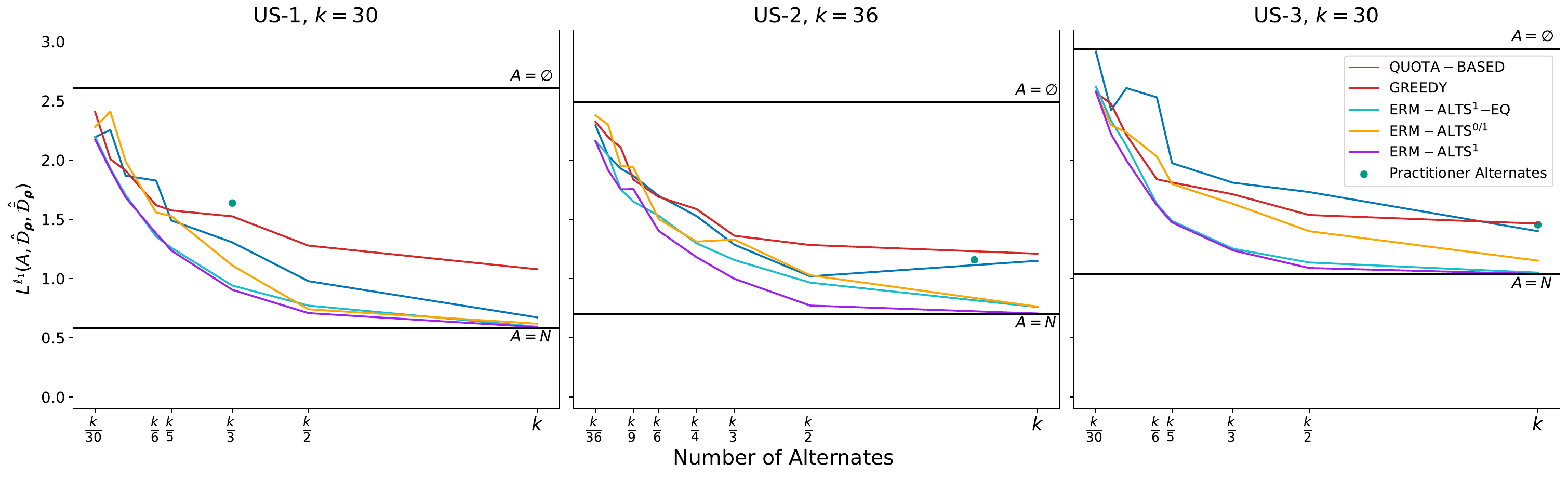}
    \caption{Replication of \Cref{fig:sim-L1} when alternates also drop out according to $\dprobs$. All run parameters are the same as in \Cref{fig:sim-L1}.}
    \label{fig:alt-dropout}
\end{figure}

\subsection{Extension 2: Directly adding extra panelists}
Suppose instead of selecting alternates from whom we can choose later, we might just want to select ``extra'' people outright who will be deterministically added to the panel. The only difference in this setup is that, because these people will be deterministically added to the panel, we may want to make sure that they do not violate any upper quotas too much. Therefore, we also take in a set of additional upper quotas, $\boldsymbol{u^*}$ to impose on the selection of extra panelists. This again does not affect the sample complexity of order $a \log n$, and amounts to a simple change in the ILP \textsc{Opt} to enforce extra quotas on the alternate set it considers. This new version of \textsc{Opt} is specified here as \textsc{Opt-Preempt}:\\

\noindent \textsc{Opt-Preempt}$^{dev}(\inst,\mathcal{D}, \boldsymbol{u}^*)$ 
\begin{align*}
    \min \ & \sum_{D \in 2^K} d_D \cdot \mathcal{D}(D) \notag\\
    \text{s.t.} \ &\sum_{i\in N} x_i \leq a \notag\\
    &\sum_{i \in N} x_i \mathbb{I}(f(i) = v) \leq u^*_{f,v} &&\forall \ D \in 2^K, f,v \in FV \\
    &l_{f,v} - \left(\sum_{i \in K \setminus D} \mathbb{I}(f(i) = v) + \sum_{i \in N} x_i\, \mathbb{I}(f(i)=v) \right) \leq z_{D, f, v} \, u_{f,v} &&\forall \ D \in 2^K, f,v \in FV \\
    & - u_{f,v} + \left(\sum_{i \in K \setminus D} \mathbb{I}(f(i) = v) + \sum_{i \in N} x_i \,\mathbb{I}(f(i)=v) \right)\leq z_{D, f, v} \, u_{f,v} &&\forall \ D \in 2^K, f,v \in FV \\
    &\sum_{f,v} z_{D,f,v} \leq d_D \qquad \quad \ \qquad(*\,dev^{\ell_1}\,*) &&\forall \ D \in 2^K \\
    &\sum_{f,v} z_{D,f,v} \leq d_D \cdot (|K| + a) |FV| \qquad(*\,dev^{0/1}\,*) &&\forall \ D \in 2^K\\
    & d_D \in \mathbb{R}_{\geq 0}  \ \ \ \qquad(*\,dev^{\ell_1}\,*)&&\forall \ D \in 2^K\\
    & d_D \in \{0,1\} \qquad(*\,dev^{0/1}\,*) &&\forall \ D \in 2^K\\
    & x_i \in \{0,1\}, \  z_{D,f,v} \in \mathbb{R}_{\geq 0} && \forall \  i \in N, D \in 2^K, f,v \in FV 
\end{align*}

\subsection{Extension 3: Selecting a maximally robust \textit{panel}}
In the body and in Extension 2, we assumed that the panel was given to us by some external selection algorithm and we could only choose alternates or extra panelists after the fact. Now, we ask: what if we could choose the panel \textit{alongside} the alternates? The next two extensions explore successively more general versions of this question.

We now suppose we want to avoid selecting alternates altogether and just try to choose a panel that is within some \textit{initial quotas} $\boldsymbol{l}'$, $\boldsymbol{u}'$, $k'$, so that it has minimal expected loss with respect to some true quotas $\boldsymbol{l}$, $\boldsymbol{u}$, $k$ after dropouts.\footnote{Note that if $\boldsymbol{l}' = \boldsymbol{l}$, $\boldsymbol{u}'=\boldsymbol{u}$, and $k'=k$, then this is just simply selecting the most robust panel in the actual bounds. If $\hat{k} > k, \hat{u}_{f,v} \geq  u_{f,v}$, and $\hat{\ell}_{f,v} = \ell_{f,v}$, then we are in the analog of the ``adding extra panelists'' scenario where now the extra people are now chosen in conjunction with the panel itself.} Just for the remainder of \Cref{app:extensions}, we will refer to the pool $N$ as the \textit{entire pool}, before the panel is selected. What we are now calling the pool would actually be $N \cup K$ in the notation of \Cref{sec:model}. To select a maximally robust panel using our ERM approach, one would first sample dropout sets directly from $N$ to get an empirical distribution of dropouts. We will refer to dropout sets from this pool as $D'$, where $D' \in 2^N$. For every dropout set $D'$ and every possible panel $K$, we are now evaluating the deviation as $dev(K \setminus D',\boldsymbol{l}, \boldsymbol{u})$. As such, we must solve a slightly modified version of the \textsc{Opt} ILP, which we specify below, that takes in a distribution over dropout sets from $2^N$. Additionally, given that our hypothesis class is now of size at most ${n \choose k}$, applying the same techniques as in \Cref{thm:mainbounds}, we deduce that it is sufficient to run \textsc{ERM-Alt} with a slightly larger sample complexity of order $k \log n$. \\

\noindent \textsc{Opt-Panel-Select}$^{dev}(\inst,\mathcal{D'})$ 
\begin{align*}
    \min \ & \sum_{D' \in 2^N} d_{D'} \cdot \mathcal{D'}(D') \notag\\
    \text{s.t.} \ &\sum_{i\in N} x_i = |K| \notag\\
    &l_{f,v} - \left(\sum_{i \in N \setminus D'} x_i\, \mathbb{I}(f(i) = v)\right) \leq z_{D', f, v} \, u_{f,v} &&\forall \ D' \in 2^N, f,v \in FV \\
    & - u_{f,v} + \left(\sum_{i \in N \setminus D'} x_i\, \mathbb{I}(f(i) = v)\right)\leq z_{D', f, v} \, u_{f,v} &&\forall \ D' \in 2^N, f,v \in FV \\
    &\sum_{f,v} z_{D',f,v} \leq d_{D'} \qquad \quad \ \qquad(*\,dev^{\ell_1}\,*) &&\forall \ D' \in 2^N \\
    &\sum_{f,v} z_{D',f,v} \leq d_{D'} \cdot |K| \cdot |FV| \qquad(*\,dev^{0/1}\,*) &&\forall \ D' \in 2^N\\
    & d_{D'} \in \mathbb{R}_{\geq 0}  \ \ \ \qquad(*\,dev^{\ell_1}\,*)&&\forall \ D' \in 2^N\\
    & d_{D'} \in \{0,1\} \qquad(*\,dev^{0/1}\,*) &&\forall \ D' \in 2^N\\
    & x_i \in \{0,1\}, \  z_{D',f,v} \in \mathbb{R}_{\geq 0} && \forall \  i \in N, D' \in 2^N, f,v \in FV 
\end{align*}

\subsection{Extension 4: Selecting a robust panel and alternates in conjunction} To select alternates and a panel in conjunction, we would again sample the entire pool to get a distribution over dropout sets $D'$. We would then solve a more complex version of the \textsc{Opt} ILP to find a panel $K$ and corresponding alternate set $A$. Here, the performance is evaluated as throughout the paper, where some dropouts are drawn from $K$ (which we have chosen) and we replace them as well as possible with the elements of $A$. We specify the new version of \textsc{Opt} below. Additionally, our hypothesis class is now of size at most ${n \choose k + a}$ (the distinction between the chosen $K$ and $A$ applies only in how we compute the deviation), meaning that it is sufficient to run \textsc{ERM-Alt} with sample complexity order $(k+a) \log n$.

\noindent \textsc{Opt-Panel-and-Alt-Select}$^{dev}(\inst,\mathcal{D'})$ 
\begin{align*}
    \min \ & \sum_{D' \in 2^N} d_{D'} \cdot \mathcal{D'}(D') \notag\\
    \text{s.t.} \ &\sum_{i \in N} w_i = |K| \ \ \ \qquad(*\,\text{panelists}\,*)\notag\\
    &\sum_{i\in N} x_i = a \notag\\
    &x_i \leq 1- w_i \ \ \ \qquad(*\,\text{alternates aren't panelists}\,*) &&\forall i \in N\\
    &y_{i,D'} \leq x_i &&\forall \ i \in N,\ D' \in 2^N \\
    &y_{i, D'} \leq \mathbb{I}(i \in D') \ \ \ \qquad(*\,\text{not using dropped replacers}\,*) &&\forall \ i \in N,\ D' \in 2^N \\
    &\sum_{i \in N} y_{i, D'} \leq \sum_{i \in D} w_i \ \ \ \qquad(*\,\text{limit replacement set size}\,*) &&\forall \ D' \in 2^N\\
    &l_{f,v} - \left(\sum_{i \in N \setminus D'} w_i\mathbb{I}(f(i) = v) + y_{i, D'}\mathbb{I}(f(i) = v)\right) \leq z_{D, f, v} \, u_{f,v} &&\forall \ D' \in 2^N, f,v \in FV \\
    & - u_{f,v} + \left(\sum_{i \in N \setminus D'} w_i\mathbb{I}(f(i) = v) + y_{i, D'}\mathbb{I}(f(i) = v)\right)\leq z_{D, f, v} \, u_{f,v} &&\forall \ D' \in 2^N, f,v \in FV \\
    &\sum_{f,v} z_{D',f,v} \leq d_{D'} \qquad \quad \ \qquad(*\,dev^{\ell_1}\,*) &&\forall \ D' \in 2^N \\
    &\sum_{f,v} z_{D',f,v} \leq d_{D'} \cdot |K| \cdot |FV| \qquad(*\,dev^{0/1}\,*) &&\forall \ D' \in 2^N\\
    & d_D \in \mathbb{R}_{\geq 0}  \ \ \ \qquad(*\,dev^{\ell_1}\,*)&&\forall \ D' \in 2^N\\
    & d_D \in \{0,1\} \qquad(*\,dev^{0/1}\,*) &&\forall \ D' \in 2^N\\
    & w_i \in \{0,1\}, x_i \in \{0,1\}, \ y_{i,D'} \in \{0,1\} && \forall \  i \in N, D' \in 2^N\\
    &z_{D',f,v} \in \mathbb{R}_{\geq 0} && \forall \ D' \in 2^N, f,v \in FV 
\end{align*}

\subsection{Extension 5: Setting robust \textit{quotas}} 

It may seem unsavory to select the panel based exclusively on the goal of being robust to dropout, as in extensions 3 and 4. This would forgo the guarantees on panel selection algorithms gained in previous work (e.g., \cite{flanigan2021fair}), and in a potentially problematic way: it might privilege groups whose members are unlikely to drop out (or even worse, who \textit{seem} unlikely to drop out based on past data). A different way of hedging against dropouts in a one-shot method (i.e., without choosing alternates or extra people after the panel is already chosen) is to try to set the quotas in a robust way to begin with. We now explore the extent to which this seemingly intuitive proposal is well-defined and algorithmically feasible.

First, observe that we cannot evaluate the robustness of a set of quotas directly, because many possible panels can be consistent with a set of quotas, and different panels may be robust to different degrees due to their differing combinatorial structure. In order to evaluate the robustness of a set of quotas, we must first fix a panel selection algorithm, which maps these quotas (along with the pool) to a particular \textit{panel distribution} $\mathcal{P}$ --- a distribution over all possible panels that satisfy the quotas. Fixing a selection algorithm, we then describe the extent to which our ``robustified quotas'' $\hat{\boldsymbol{\ell}},\hat{\boldsymbol{u}},\hat{k}$ are robust with respect to our true quotas $\boldsymbol{\ell},\boldsymbol{u},k$ as the expected deviation from the true quotas of $K \setminus D$, where $K \sim \mathcal{P}$ (satisfying $\hat{\boldsymbol{\ell}},\hat{\boldsymbol{u}},\hat{k}$) and $D \sim \ddist$. The key difference is that now, we are trying to optimize the expected deviation over the randomness of both the dropout set \textit{and the panel}. Our ERM approach still helps us here, because it allows us to evaluate our choice of quotas for a \textit{given} panel distribution; however, even with cutting edge selection algorithms, the process of computing a desirable $\mathcal{P}$\emdash as is necessary to evaluate a given choice of quotas\emdash is its own complex multilayer optimization problem \cite{flanigan2021fair}. This makes it is unclear how to directly optimize the quotas without brute-force exploration of all quota settings --- we leave this to future work.

\end{document}